\documentclass{article}

\usepackage{iclr2026_conference,times}
\iclrfinalcopy
\usepackage{times}
\usepackage{macros}
\usepackage[hang,flushmargin]{footmisc}

\newcommand{\dino}{DINO}
\newcommand{\hyp}{Minkowski Representation Hypothesis}

\title{\vspace{-5mm}Into the Rabbit Hull: From Task-Relevant Concepts in DINO to Minkowski Geometry}

\newcommand{\aut}[1]{\textbf{#1}}
\newcommand{\af}[1]{{\small #1}}

\newcommand{\afn}[1]{\textcolor{primary}{$^{#1}$}}

\author{
\hspace{-1.1mm}\aut{Thomas Fel}{\textcolor{primary}{$^\star$}\afn{a,b}} \quad
\aut{Binxu Wang}{\textcolor{primary}{$^\star$}\afn{a,b}} \\
\aut{Michael A. Lepori}\afn{d} \quad
\aut{Matthew Kowal}\afn{e} \quad 
\aut{Andrew Lee}\afn{b} \quad
\aut{Randall Balestriero}\afn{d} \quad
\aut{Sonia Joseph}\afn{g} \\
\aut{Ekdeep S. Lubana}\afn{b,h}
\aut{Talia Konkle}\afn{a,c} \quad
\aut{Demba Ba}\afn{a,b} \quad
\aut{Martin Wattenberg}{\textcolor{primary}{$^\dagger$}\afn{b,f}} \vspace{4pt}\\
\afn{a}\af{Kempner Institute, Harvard University} \quad
\afn{b}\af{Harvard University} \\
\afn{c}\af{Dept. of Psychology, Harvard University} \quad
\afn{d}\af{Brown University} \\
\afn{e}\af{FAR.AI} \quad
\afn{f}\af{Google DeepMind} \quad
\afn{g}\af{Meta} \quad
\afn{h}\af{Goodfire}\\
\\
{
\small 
\href{https://kempnerinstitute.github.io/dinovision}{\raisebox{-2.8pt}{\includegraphics[height=10pt]{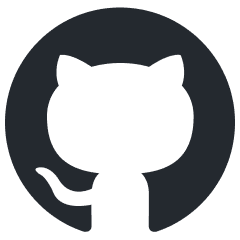}}
\texttt{kempnerinstitute.github.io/dinovision}}
}
\vspace{-4mm}
}

\begin{document}

\renewcommand{\thefootnote}{}
\footnotetext{
\textcolor{primary}{$^\star$} Equal contribution. \\
\textcolor{primary}{$^\dagger$} Work done while at Harvard University. \\
Correspondence to \texttt{\{tfel,binxu\_wang\}@fas.harvard.edu}.
}
\renewcommand{\thefootnote}{\arabic{footnote}}

\maketitle

\begin{figure}[ht]
\vspace{-1mm}
\centering
\makebox[\textwidth][c]{%
    \includegraphics[width=1.05\textwidth]{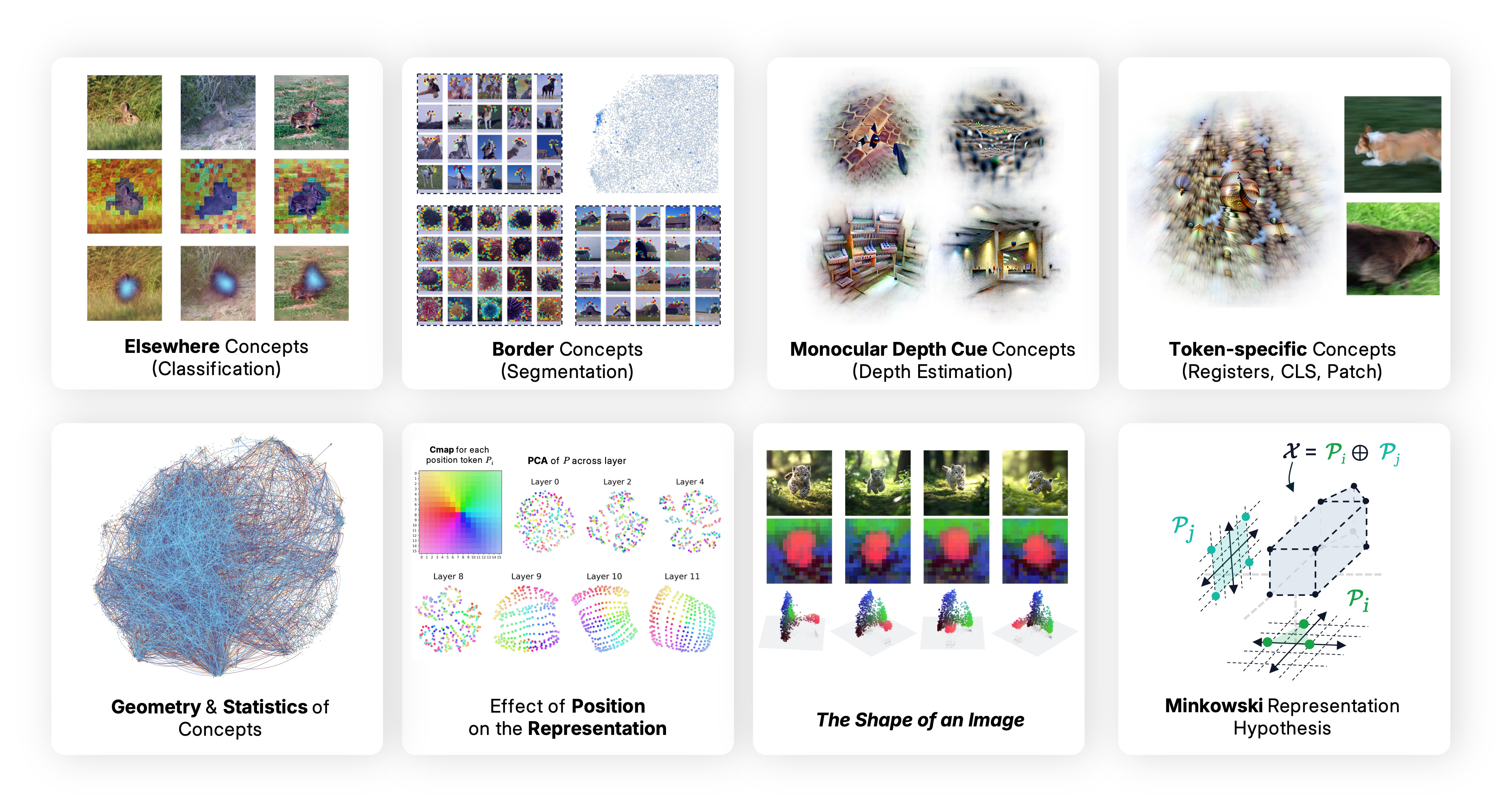}
}
\vspace{-1mm}
\caption{
\textbf{Overview of our study.}
\emph{Part I — Downstream usage.} Different tasks recruit distinctive families of concepts: classification relies on ``Elsewhere'' detectors, segmentation on boundary concepts, depth estimation on three families of monocular cues, while token-specific concepts (e.g., registers) capture global scene factors such as illumination or motion blur.
\emph{Part II — Geometry and statistics of concepts.} Even though atoms are distributed as in the sparse-coding view, we also find anisotropy aligned with task subspaces, antipodal pairs forming signed axes, and partly dense structure: positional information compresses into 2D, yet locally connected neighborhoods persist even after removing position.
Together, these signs suggest that representations are organised beyond linear sparsity alone.
\emph{Part III — Towards Minkowski Geometry.} Synthesizing these observations, we explore a refined view: token as sum of convex mixture. This view is grounded in Cognitive theory of Gärdenfors’ conceptual spaces as well as in the model’s own mechanism: each attention head  produces convex combinations of value vectors, and their outputs add across heads; tokens can thus be understood as convex mixtures of a few archetypal landmarks (e.g., a rabbit among animals, brown among colors, fluffy among textures). This points to activations being organized as Minkowski sums of convex polytopes, with concepts arising as convex regions rather than linear directions. We finish by examining empirical signals of this geometry and its consequences for interpretability.
}
\vspace{-1mm}
\label{fig:big_picture}
\end{figure}

\clearpage

\begin{abstract}
DINOv2 is routinely deployed to recognize objects, scenes, and actions; yet the nature of \textit{what} it perceives remains unknown. 
As a working baseline, we adopt the Linear Representation Hypothesis (LRH) -- which posits that internal features can be expressed through a sparse combination of nearly independent directions. We operationalize it using overcomplete sparse autoencoders, producing a large-scale set of visual concepts: a 32,000-unit dictionary\footnote{Available at \texttt{\href{https://kempnerinstitute.github.io/dinovision}{kempnerinstitute.github.io/dinovision}}} that serves as the interpretability backbone of our study, which unfolds in three parts.

~~In the first part, we analyze how different downstream tasks recruit concepts from our learned dictionary, revealing functional specialization: classification exploits ``Elsewhere'' concepts that fire everywhere except on target objects, implementing learned negations; segmentation relies on boundary detectors forming coherent subspaces; depth estimation draws on three distinct monocular depth cue matching visual neuroscience principles.

~~Following these functional results, we analyze the geometry and statistics of the concepts learned by the SAE. First, we found that representations are partly dense rather than strictly sparse. The dictionary, initialized at random, evolves toward greater coherence and departs from an ideal in which dictionary concepts are maximally orthogonal (Grassmannian frames). 
Within an image, tokens occupy a low dimensional, locally connected set, which persists after removing position.
Taken together, these signs of partial dense representation, local connectedness, and coherent dictionary atoms suggest that representations are organized beyond linear sparsity alone.

~~Synthesizing these observations, we propose a refined view: tokens are formed by combining convex mixtures of a few archetypes (e.g., a rabbit among animals, brown among colors, fluffy among textures). This structure is grounded both in cognitive theory of Gärdenfors’ conceptual spaces and in the model’s own mechanism as multi-head attention is a sum of convex mixtures, implicitly defining regions bounded by archetypes.
In this picture, concepts are expressed through proximity to archetypes and by membership within bounded regions (rather than by unbounded linear directions).
We clarify this view by introducing the \emph{Minkowski Representation Hypothesis (MRH)} and examine its empirical signatures and implications for how we study, steer, and interpret vision-transformer representations.

\end{abstract}

\section{Introduction}

Vision Transformers (ViTs)~\cite{dosovitskiy2020image} have redefined the foundations of visual representation learning. Inspired by the success of Transformer architectures~\cite{vaswani2017attention} in natural language processing, ViTs abandon the convolutional inductive bias~\cite{lecun2015deep,serre2006learning} in favor of a more general framework: images are partitioned into fixed-size patches, each linearly projected into an embedding vector (a patch ``token''), then processed by layers of self-attention and feedforward modules.
This architecture introduces a new regime of scalability and flexibility~\cite{zhai2022scaling,dehghani2023scaling}. Compared to convolutional networks, ViTs follow more favorable scaling laws~\cite{alabdulmohsin2023getting}: their performance improves reliably with increased data and parameters, making them well-suited to large-scale training. They are also malleable learners, able to adapt to a broad spectrum of vision tasks given sufficient data and compute. When trained with contrastive~\cite{zhai2023sigmoid}, masked~\cite{he2022masked}, or self-distillation objectives~\cite{touvron2022deit,caron2021emerging}, ViTs can extract rich, semantically organized representations without the need for human supervision.

\begin{figure}[t]
    \vspace{-8mm}
    \centering
    \includegraphics[width=0.95\linewidth]{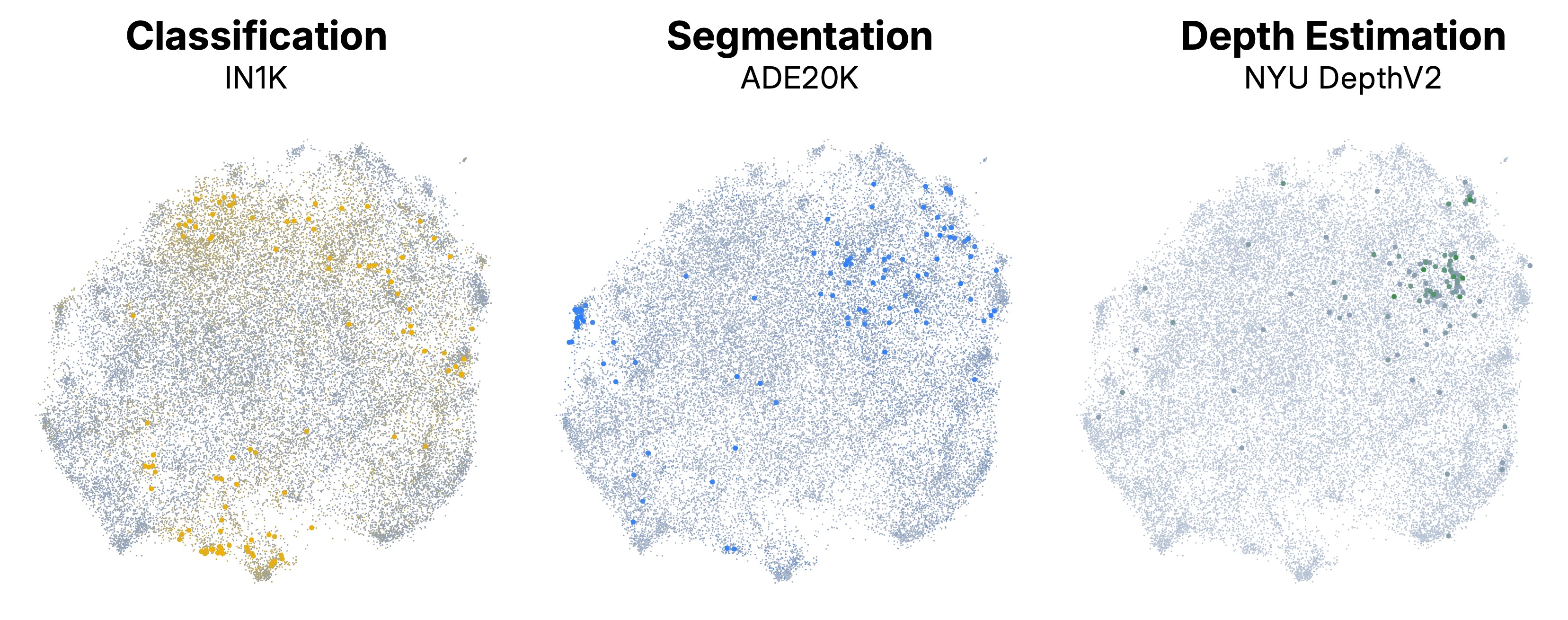}
    \vspace{-3mm}
    \caption{
    \textbf{Concept importance across tasks.} UMAP projection of the learned dictionary, with colors indicating the relative magnitude of each concept’s contribution to three downstream tasks: \textbf{(Left)} classification (ImageNet-1k), \textbf{(Middle)} segmentation, \textbf{(Right)} depth estimation.
    While classification recruits a broad set of concepts, segmentation and depth primarily activate more restricted set of concepts. 
    Although UMAP only preserve local geometry, functionally relevant groupings are visibly clustered in the projection. We show in later sections that different tasks consistently recruit distinct, low-dimensional regions of the concept space.
    }
    \label{fig:downstream_umap}
    \vspace{-4mm}
\end{figure}

\paragraph{The case of DINOv2.} Trained in a self-supervised manner on a massive corpus of unlabeled images, DINOv2 has become a showcase of emergent visual representations~\cite{oquab2023dinov2,darcet2023vision}. Its embeddings support diverse downstream tasks: fine-grained classification~\cite{chiu2024fine}, segmentation~\cite{liu2023grounding}, monocular depth estimation~\cite{mao2024stealing,cui2024surgical}, object tracking~\cite{faber2024leveraging,tumanyan2024tracker}, even robotic perception~\cite{kim2024openvla}. 
They have also proven valuable in non-discriminative settings, acting as priors for generative models~\cite{yu2024representation}, proxies for distributional similarity~\cite{stein2023exposing}, tools for revealing blindspots~\cite{bohacek2025uncovering}.
The embeddings are not just visually rich, but robust: they power video world models that predict physical dynamics from frames~\cite{baldassarre2025back}, and can be aligned with language to achieve open-vocabulary segmentation and zero-shot classification despite having never been trained on text~\cite{jose2025dinov2}.
DINOv2 also transfers robustly across domains~\cite{zhang2023stable,zhou2024dinowm}, including satellite imagery~\cite{waldmann2025panopticon} and medical scans~\cite{ayzenberg2024dinov2}, underscoring that the model has internalized a deeply structured representation of the visual world.
Yet despite these successes, the nature of DINOv2's internal representations remains unclear. What, precisely, is encoded? Which features are \emph{available}~\cite{hermann2023foundations} to downstream tasks, and how are they geometrically organized? More fundamentally, how can a label-free objective give rise to such effective few-shot readout?

\paragraph{Vision Explainability.} 
To address these questions, we draw on vision explainability, a field that has developed both an empirical toolkit and theoretical accounts for probing large vision models~\cite{doshivelez2017rigorous,gilpin2018explaining,fel2024sparks}. 
Early work emphasized attribution-based methods~\cite{simonyan2013deep,zeiler2014visualizing,bach2015pixel,springenberg2014striving,smilkov2017smoothgrad,sundararajan2017axiomatic,selvaraju2017gradcam,fong2017meaningful,fel2021sobol,novello2022making,muzellec2023gradient}, 
which visualize ``where'' a model looks or how sensitive it is to local perturbations. 
While often compelling, such techniques provide only a surface-level view: they do not reveal ``what'' internal features are being computed, 
and their usefulness is limited~\cite{hase2020evaluating,hsieh2020evaluations,nguyen2021effectiveness,fel2021cannot,kim2021hive,sixt2020explanations}. 
To move beyond attribution, the community turned to concept-based explanations~\cite{kim2018interpretability,poeta2023concept}, 
which aim to extract interpretable latent dimensions from a model’s representation space~\cite{bau2017network,ghorbani2019towards,zhang2021invertible,fel2023craft,graziani2023concept,vielhaben2023multi,kowal2024understanding,kowal2024visual,fel2023holistic}. 
A natural first attempt was to identify individual neurons as carriers of concepts. 
However, this view quickly ran into challenges: neurons are polysemantic, depend on arbitrary bases, 
and implicitly assume the number of features is bounded by the number of units~\cite{arora2018linear,elhage2022superposition,gorton2024missing}. 
These limitations motivated methods that identify concepts as distributed features, without requiring alignment to single neurons~\cite{ghorbani2017interpretation,zhang2021invertible,fel2023craft}. 
A first theoretical account soon followed in the form of phenomenology by~\citet{elhage2022superposition}: 
in a $d$-dimensional space, neural networks can encode exponentially many nearly orthogonal features 
by representing each as a sparse combination of neurons. 
This phenomenology crystallized into the \textit{Linear Representation Hypothesis (LRH)}, 
which holds that models contain many more features than neurons, arranged as sparse, quasi-orthogonal directions~\cite{park2023linear}. 
If the LRH is valid, then concept extraction amounts to an overcomplete dictionary learning problem~\cite{fel2023holistic}. 
In this framework, the activation space of a layer is factorized into a dictionary basis and sparse codes. 
Sparse autoencoders (SAEs)~\cite{makhzani2013k,cunningham2023sparse,bricken2023monosemanticity, joseph2025steeringclipsvisiontransformer, joseph2025prismaopensourcetoolkit,klindt2025superposition} 
are one practical instantiation of this idea, enforcing sparsity through a simple encoder-decoder architecture. 
When applied to pretrained models such as DINOv2, SAEs uncover a rich library of patterns, which we refer to as \textit{concepts}.

\vspace{-3mm}
\paragraph{Our contributions.}
In this work, we start our investigation of DINOv2’s representational structure by grounding our analysis in the Linear Representation Hypothesis (LRH), which views neural activations as sparse superpositions of nearly orthogonal features, and operationalizing it through sparse autoencoders. This yields a dictionary of over $32{,}000$ recurring patterns, which we release in the form of, to our knowledge, the largest interactive interpretability demo for any vision foundation model. Building on this foundation, our analysis unfolds in three parts:

\begin{enumerate}[label=\textbullet,leftmargin=1pt,itemindent=0pt]
  \item \textbf{Downstream Usage.}  
  We analyze how different tasks recruit concepts from the learned dictionary and find  specialization: (i) ``Elsewhere'' concepts supporting classification by learned negation; (ii) border concepts forming coherent subspaces central to segmentation; and (iii) three monocular depth cue families consistent with classical visual-neuroscience.

\item \textbf{Concept Geometry and Statistics.}
We found observations compatible with a linear sparse-coding view: atoms are distributed rather than neuron-aligned (low Hoyer scores), and we find antipodal pairs that form signed semantic axes, so $\cos\theta\!\approx\!-1$ can reflect opposite poles of the same feature axis rather than different features\footnote{This poses challenges for cosine-based analyses in clustering or retrieval, which may misinterpret antipodal pairs as unrelated features.}.
At the same time, the dictionary departs from a near-orthogonal Grassmannian picture: pairwise inner products show heavier tails and the singular-value spectrum decays sharply, indicating anisotropy and clustered coherence, with subspaces aligned to task recruitment. Finally, although positional information compresses toward a $\sim$2D subspace in later layers, per-image token clouds remain smooth and low-dimensional even after removing positional components. Taken together, these diagnostics suggest a partly dense, structured organization that go beyond a simple sparse-coding view of the representations.

  \item \textbf{Towards Archetypal Geometry.}  
  Closer inspection suggests a refined view of the representation: tokens could behave as the outcome of several interpolations between archetypal landmarks, for instance, locating itself as a rabbit among animals, as brown among colors, and as fluffy among textures, whose contributions add to form the final embedding. This additive structure, naturally implemented by multi-head attention, motivates an organizing principle we call the \emph{Minkowski Representation Hypothesis (MRH)}: representations assemble from overlapping convex regions around archetypes, and concepts are expressed by proximity to landmarks rather than by single linear directions. We conclude by outlining practical consequences of this theory.
\end{enumerate}
Having sketched our three-part investigation, we begin by recalling the LRH and how it can be operationalized through sparse autoencoders, which provide the dictionary of concepts underpinning our analysis.

\section{Linear Representation Hypothesis and Operationalization}

A recurring phenomenology of large models is that their representational capacity vastly exceeds the number of neurons: in a $d$-dimensional space, they can encode exponentially many features by representing each as a sparse linear superposition \cite{arora2018linear,elhage2022superposition}. Empirically, such features behave as nearly orthogonal directions~\cite{papyan2020prevalence}, active only in restricted contexts, while neurons themselves are polysemantic~\cite{nguyen2016multifaceted}. This motivates two conclusions: (i) neurons are not the appropriate locus of interpretability, and (ii) one must recover the latent basis along which the model effectively operates.
Beyond empirical characterization, such geometry can be motivated by compression: when features are arranged with minimal coherence, it maximizes the number of retrievable features while minimizing destructive interference. 
This principle extends beyond neural activations to fundamental problems in discrete geometry and optimization theory. It connects to classical sphere packing problems: Tammes's problem~\cite{mooers1994tammes}, which seeks optimal angular separation of points on a sphere (originally motivated by pollen grain morphology, where surface protrusions must be optimally spaced for aerodynamic dispersal efficiency); Thomson's problem~\cite{thomson1904structure,bowick2002crystalline}, which minimizes Coulomb electrostatic repulsion energy between charged particles constrained to a sphere; and the spherical code problem~\cite{delsarte1991spherical}, which maximizes minimum distance between codewords on the unit sphere for error correction. 
These problems share a common mathematical substrate: the geometric structure formalized as Grassmannian frames in signal processing theory~\cite{strohmer2003grassmannian}. This object underlies the optimization principle of minimizing mutual coherence (the maximum absolute inner product between distinct normalized vectors). These mathematical convergences suggest that neural networks may naturally approximate such optimal geometric configurations to (\textbf{\textit{i}}) maximize representational capacity while (\textbf{\textit{ii}}) minimizing cross-feature interference.
This geometric phenomenology crystallizes into the Linear Representation Hypothesis~\cite{elhage2022superposition,costa2025flat} that we recall here:
\begin{definition}\label{def:lrh}{\textbf{Linear Representation Hypothesis (LRH).}}
A representation $\a \in \mathbb{R}^{d}$ satisfies the linear representation hypothesis if there exists a sparsity constant $k$, an overcomplete dictionary $\D=(\bm{d}_{1},\dots,\bm{d}_{c}) \in \mathbb{R}^{c\times d}$, and a coefficient vector $\z \in \mathbb{R}^{c}$ such that $\a = \z\D $, under the following conditions:
\[
\begin{cases}
\text{\textit{(\textbf{i})} Overcompleteness:} & c \gg d, \\[2pt]
\text{\textit{(\textbf{ii})} Quasi-orthogonality:} &
\mu(\D) \equiv \max_{i\neq j}|\D_{i}\D_{j}^\tr| \le \varepsilon, \quad \|\D_{i}\|_{2}=1, \\[4pt]
\text{\textit{(\textbf{iii})} Sparsity:} & |\operatorname{supp}(\bm z)| \le k \ll c.
\end{cases}
\]
\end{definition}

The LRH is useful precisely because it can be operationalized (see \cref{app:dico}). If activations admit such a sparse overcomplete representation, then concept discovery reduces to finding the appropriate overcomplete dictionary. 
Classical approaches of Dictionary Learning include Non-negative Matrix Factorization (NMF)~\cite{gillis2020nonnegative}, Sparse Coding~\cite{elad2010sparse} and, more recently, Sparse Autoencoders (SAEs) have emerged as efficient, large-scale instantiations of LRH \cite{cunningham2023sparse,bricken2023monosemanticity}, providing practical access to the latent basis.

A persistent challenge, however, is stability: naïve SAEs produce inconsistent features across runs, undermining interpretability~\cite{paulo2025sparse,papadimitriou2025interpreting}. 
To address this, we adopt a stable SAEs~\cite{fel2025archetypal}, which constrain each dictionary atom to lie in the convex hull of real activations. This guarantees that atoms remain in-distribution and yields reproducible, geometrically faithful dictionaries.
Formally, let $(\mathcal{X},\mathcal{F},\mathbb{P})$ denote the probability space of natural images, $\mathcal{X}\subset\R^{H\times W\times 3}$. For a pretrained Vision Transformer $\f:\mathcal{X}\to\R^{t\times d}$, any image $\x\sim\mathbb{P}$ yields activations $\a=\f(\x)\in\R^{t\times d}$, i.e.\ $t$ token embeddings of dimension $d$. For a batch $\X=(\x_i)_{i=1}^n$, concatenating all tokens gives $\A\in\R^{nt\times d}$. Our objective is to factorize $\A$ into sparse codes $\Z\in\R^{nt\times c}$ and a dictionary $\D\in\R^{c\times d}$ with
\[
\min_{\Z,\D}\ \|\A-\Z\D\|_F^2
\quad \text{subject to} \quad
\Z\ge 0,\ \|\Z_i\|_0\le k,\ \D_i \in \mathrm{conv}(\A).
\]
Here, $\Z \ge 0$ denotes an elementwise non-negativity constraint, i.e.\ all entries of $\Z$ are constrained to be nonnegative.
Concretely, we used DINOv2-B with 4 registers as $\f$, with $d=768$, $t=261$. We set $c=32{,}000$ atoms and $k=8$ active codes per token, approximate $\mathrm{conv}(\A)$ by $128{,}000$ centroids extracted via $k$-means over 1.4M ImageNet-1K images (with augmentation), and parametrize $\D=\S\C$ with $\S$ row-stochastic. Codes $\Z$ are produced via a single-layer encoder with BatchTopK projection~\cite{bussmann2024batchtopk,hindupur2025projecting}. Training with Adam for 50 epochs yields $R^2>88\%$ reconstruction fidelity, consistent with prior stability results.

\paragraph{Interpretation of $\Z$ and $\D$.} Essentially, the matrix $\Z$ encodes the \emph{statistical structure} of the activation space -- capturing which concepts are active, how frequently, and to what degree. In contrast, the dictionary $\D$ encodes the \emph{geometric structure} defining the atomic directions used to span the space and organize the model’s internal feature basis. This decomposition yields a library of 32,000 concept atoms, each interpretable as a linear probe on DINOv2 activations. 
With this basis in hand, we first investigate which concepts are recruited by different downstream tasks (\cref{sec:tasks}), 
then discuss their statistical and geometric organization in detail (\cref{sec:statistics_and_geometry_of_concepts}).

\clearpage
\section{Task-Specific Utilization of Learned Concepts}
\label{sec:tasks}
With the concept dictionary in place, we now ask: \emph{which of these concepts are actually used across downstream tasks?} 
To address this, we analyze how task-specific linear probes interact with the concept dictionary.
Consider our set of activations $\A \in \R^{nt \times d}$ and the concept representation $\Z\D$, since any downstream prediction of the form $\bm{Y} = \A \W^\tr$ with $\W \in \R^{o \times d}$, $o$ being the number of outputs (or classes), becomes $\bm{Y} = (\Z\D)\W^\tr = \Z\W'$ where $\W' = \D \W^\tr$ defines the effective alignment between dictionary concepts and task outputs.
This formulation provides a direct measure of each concept's importance through  $\E(\Z \W') = \E(\Z)\W'$, which represents the expected concept activation weighted by task-output alignment~\cite{ancona2017better}. 
In linear regimes, this score constitutes the optimal importance measure under standard faithfulness criteria (see Theorem~3 in~\cite{fel2023holistic})\footnote{Faithfulness quantifies how well importance scores predict feature perturbation effects, measured via C-Deletion, C-Insertion~\cite{petsiuk2018rise}, and C-$\mu$Fidelity~\cite{yeh2019infidelity}. In linear cases, these metrics coincide and share the same optimal attribution function.}. 
This optimality result validates $\E(\Z)\W'$ as the canonical measure of concept importance for task alignment (see \cref{app:importance} for details).

\paragraph{Different tasks recruit different concepts.}
Figure~\ref{fig:downstream_umap} reveals that different tasks recruit different subsets of concepts, often with minimal overlap. Classification activates a wide and dispersed array of concepts, while in contrast, segmentation and depth estimation draw on more compact and localized regions of the concept manifold. This may suggest \emph{functional regions} in the latent space, where concepts are reused non-uniformly across tasks.
Quantitatively, we confirm this asymmetry in~\cref{fig:downstream_intratask} (Left): classification draws from a broader span of the dictionary than segmentation or depth estimation. 
\begin{figure}[h]
    \centering
    \includegraphics[width=0.48\linewidth]{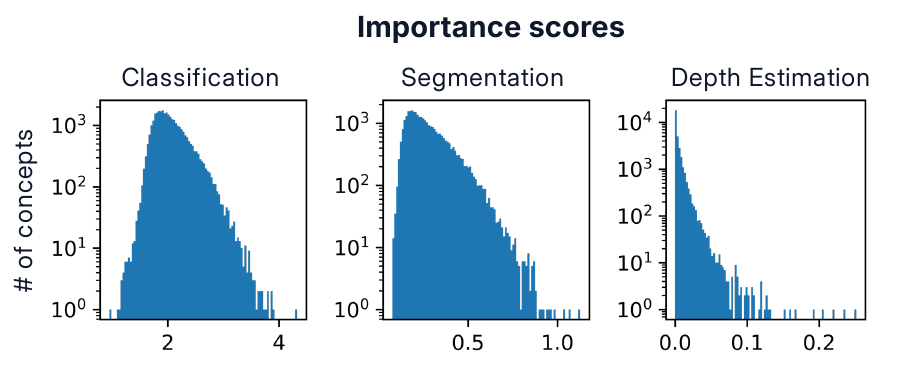}
    \includegraphics[width=0.3\linewidth]{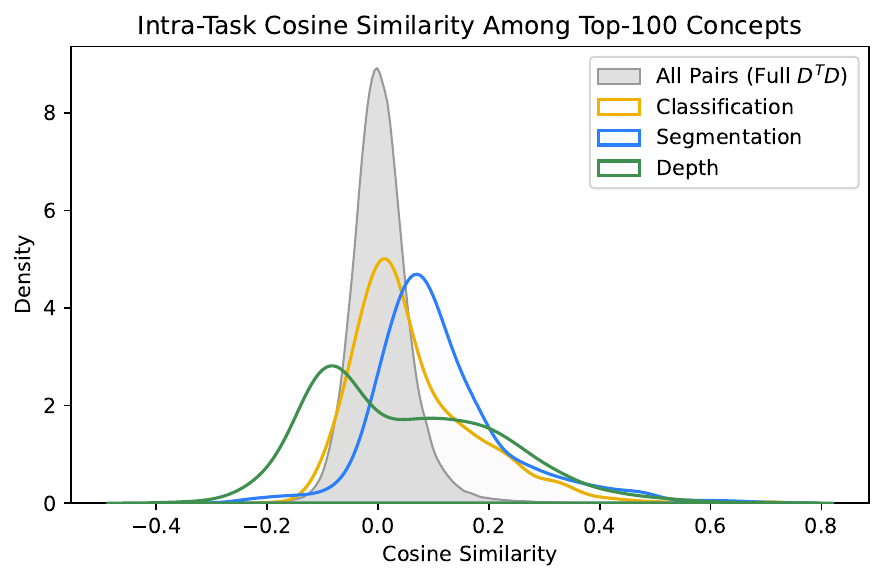}
     \includegraphics[width=0.2\linewidth]{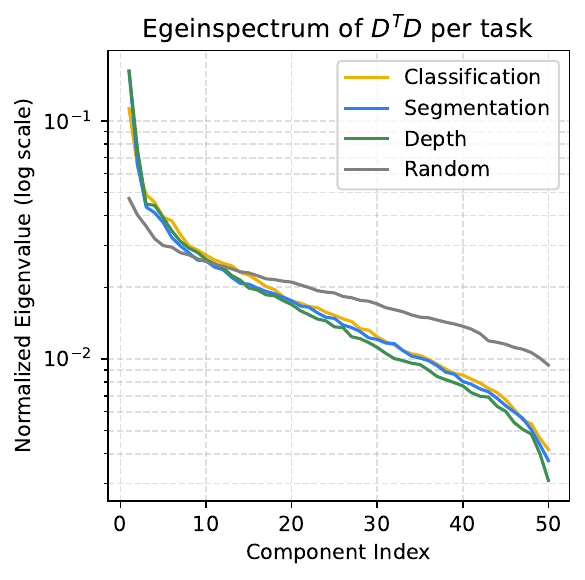}
    \caption{
    \textbf{(Left) Classification recruits more concepts than segmentation than depth.}  Classification utilizes a larger fraction of the dictionary compared to segmentation and depth, likely reflecting the higher rank of the classification head. This supports the view that task complexity and output dimensionality shape the breadth of concept recruitment.
    \textbf{(Middle) Intra-task concept similarity.} Cosine similarity histograms of the top 100 most important concepts per task, compared to random subsets of the dictionary. Intra-task concept pairs exhibit higher mutual alignment, deviating from the quasi-orthogonality expected of generic dictionary atoms. This suggests that functional concepts form more coherent subspaces. 
    \textbf{(Right) Spectral analysis of task-specific subspaces.} Singular value spectra of the top-100 task-relevant concepts reveal sharply decaying profiles for all tasks (especially segmentation and depth) indicating that each task activates a low-dimensional functional subspace. Compared to random concept subsets, task-aligned subspaces exhibit stronger concentration, supporting a ``functional subspace'' hypothesis.
    }
    \label{fig:downstream_intratask}
\end{figure}

In fact, we can show that the recruited concepts seem to bear geometric resemblance (see \cref{fig:downstream_intratask}).
We isolate the top 100 most task-aligned concepts per head and analyze their pairwise similarities. As shown in~\cref{fig:downstream_intratask} (Middle), intra-task concepts are significantly more aligned with one another than randomly selected concepts, breaking the quasi-orthogonality observed globally. Finally, in ~\cref{fig:downstream_intratask} (Right), we confirm this observation by comparing the eigenvalue spectrum of each task's sub-dictionary. All three spectra decay much faster than those of random subsets of concepts, indicating that task-specific concepts form a low-dimensional subspace.

Together, these findings suggest that perceptual tasks selectively activate low-dimensional, functionally specialized subspaces within the broader concept representation space of \dino. 
But what do these task-specific subspaces actually look like? Can we identify recurring families of concepts that characterize each task? We now turn to a qualitative examination of these patterns, starting with classification.

\begin{figure}[t]
    \vspace{-5mm}
    \centering
    \includegraphics[width=0.99\linewidth]{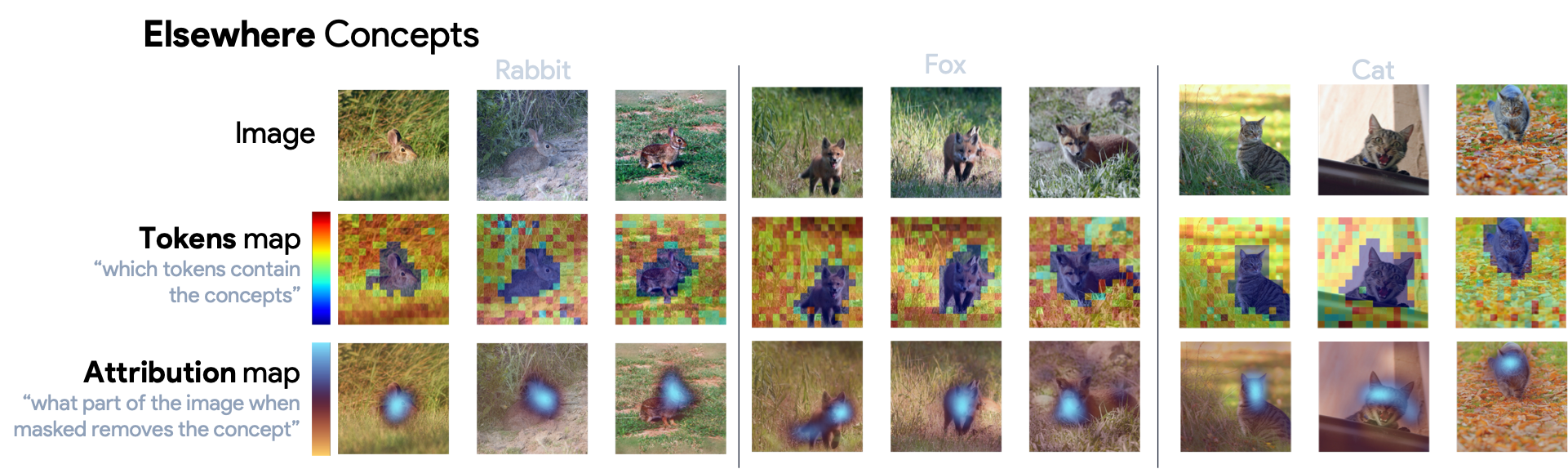}
    \caption{
    \textbf{``Elsewhere'' concepts reflect off-object activation conditioned on object presence.}
    Visualization of a recurring concepts pattern, consistently among the top-3 most important  concepts for several ImageNet classes (rows: rabbit, fox, cat), using token-level attribution (middle row) and causal masking~\cite{petsiuk2018rise} (bottom row). 
    These ``Elsewhere'' concepts consistently activate in tokens disjoint from the object, yet their presence is conditional on the object itself being present elsewhere in the image: they vanish when the object is removed. 
    Rather than capturing background texture, they express a structured logical relation: ``\emph{not the object, but the object exists}''. This suggests that \dino~implicitly learns a form of fuzzy spatial logic, distributing class evidence across both object-centric and off-object tokens. See \cref{app:elsewhere} for more details.
    }
    \label{fig:hollow}
\end{figure}

\paragraph{Classification and the \emph{Elsewhere} Concepts.}

Zooming into classification, we observe a consistent and intriguing pattern: across a wide range of ImageNet classes, the top few most important concepts typically include not only interpretable objects or object-parts, but also an ``\textit{Elsewhere}'' concept. As illustrated in~\cref{fig:hollow}, these concepts activate broadly across the tokens, but crucially \emph{not} on the object itself. Their firing is suppressed exactly where the object appears, and prominent in surrounding regions or background. Importantly, ``Elsewhere'' concepts are \emph{not} generic background detectors: their firing depends on the object’s presence, and vanish if the object is removed.
This suggests that certain concepts do not provide direct positive evidence for the presence of a class, but instead activate in regions excluding the object, while still being dependent on its presence elsewhere in the image. These concepts reflect a kind of conditional negation: ``the object exists elsewhere, but this token is not the object''. Such spatially disjoint activations may contribute to classification by implicitly outlining object boundaries, encoding contextual contrast, or supporting a distributed logic over tokens. However, their non-local nature can mislead attribution-based interpretations: concept heatmaps might falsely suggest reliance on background regions, when the actual signal lies in a structured off-object pattern. Further, this also violates the unspoken assumption of heatmaps visualization: \textit{The concept is just about the token where the concept fires}. 

This highlights the need for interpretability tools that take both concept localization and causal perturbation~\cite{shaham2024multimodal} into consideration.
Having explored classification, we now turn to the task of semantic segmentation.

\begin{figure}[t]
    \vspace{-5mm}
    \centering
    \includegraphics[width=0.99\linewidth]{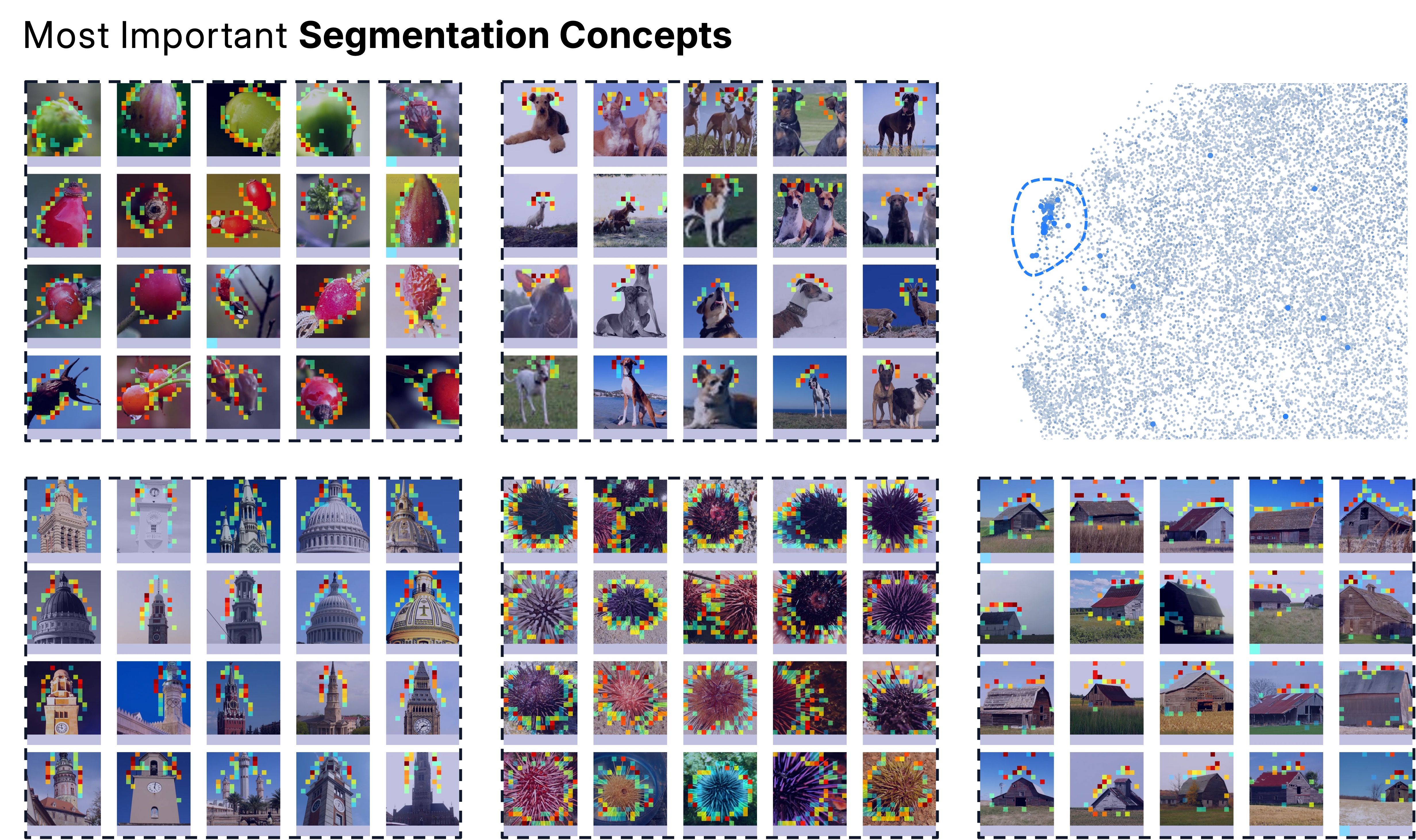}
    \caption{
    \textbf{Segmentation relies on spatially localized border concepts.}
    Examples of the most important concepts across segmentation tasks, visualized via token attribution (colored overlays). Most of these concepts activate along object boundaries, whether biological (e.g., limbs, heads) or architectural (e.g., domes, rooflines). Despite differences in content, these border concepts exhibit consistent spatial patterns and nontrivial similarity in embedding space (right), suggesting a shared functional role and a possibly low-dimensional submanifold within the concept geometry.
    }
    \label{fig:border}
\end{figure}

\paragraph{Segmentation and Border Concepts.}
For Segmentation, we observe that all the concepts among the top-50 consistently localize along object contours or spatial boundaries. As shown in~\cref{fig:border}, these ``border concepts'' activate narrowly along the periphery of objects (highlighting limbs, outlines, or silhouette transitions). Remarkably, while the precise visual features vary across classes (e.g., animal ears, tower edges, or sea urchin spines), the spatial footprints of these concepts remain strikingly consistent.
Furthermore, in the concept embedding space, these border concepts form a visibly tight cluster (\cref{fig:border}, right), suggesting that \dino~allocates a dedicated region of its representational geometry to encoding object boundaries. As quantitatively shown in \cref{fig:downstream_intratask}, their absolute cosine similarity is higher than average and their  eigenspectrum decays faster than a random subset of concepts suggesting a low-dimensional structure composed of boundary detectors.

Segmentation concepts reveal that \dino~dedicates portions of its concept space to encoding local spatial structure. We now examine depth estimation, another spatially grounded task, but one that requires global 3D understanding rather than contour localization.

\begin{figure}[t]
    \vspace{-2mm}
    \centering
    \includegraphics[width=0.99\linewidth]{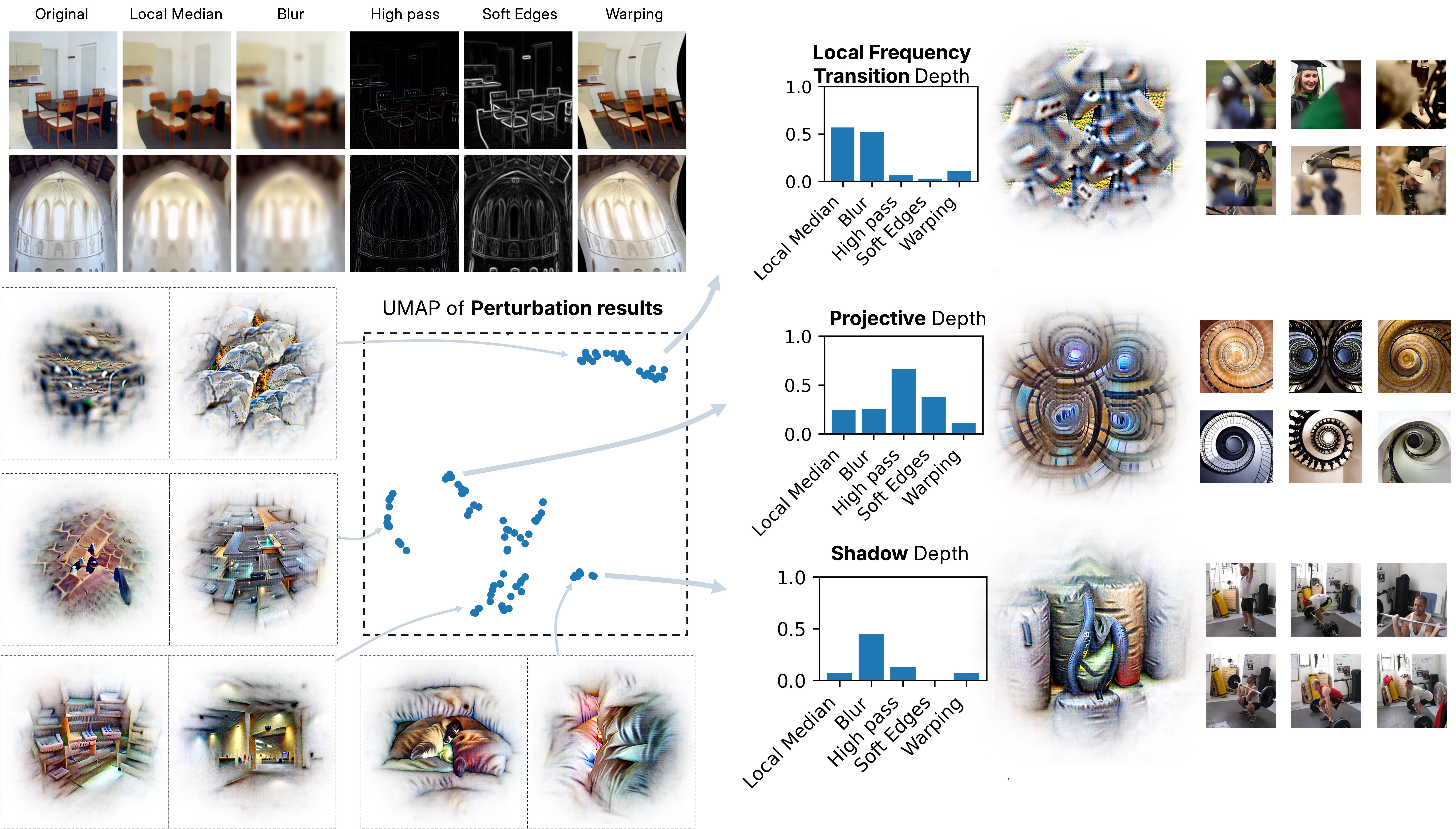}
    \caption{
    \textbf{\dino~encodes diverse monocular depth cues.}
    Visualization of key concepts used in monocular depth estimation tasks. We identify three dominant types: projective geometry cues (e.g., vanishing lines, converging structures), shadow-based cues (e.g., soft lighting gradients and cast shadows), and frequency-based cues (e.g., transitions between high- and low-texture regions). These findings suggest that \dino~learns a rich basis of 3D perception primitives from 2D data alone.
    }
    \label{fig:depth}
\end{figure}

\paragraph{Depth and Monocular Cue Concepts.}
Despite being trained without any explicit supervision for 3D understanding, \dino~exhibits a surprising aptitude for depth-related tasks~\cite{mao2024stealing,el2024probing,zhan2024general}. Prior studies have shown its features to be competitive in monocular depth estimation benchmarks, raising the question: what types of internal features support this capacity?
We conduct a targeted perturbation analysis of the most depth-relevant concepts. Specifically, we apply a set of controlled image manipulations designed to selectively remove or preserve different classes of monocular cues (e.g., local median blurring to suppress shadows, edge-preserving smoothing to retain object contours, or high-pass filtering to emphasize projective geometry of the scene).
We then measure the concept activation profiles across these perturbations and project the results onto a UMAP embedding (\cref{fig:depth}). This reveals three coherent clusters, each differentially sensitive to specific perturbation types. By examining both their activation patterns and the visual prototypes associated with each cluster, we identify three main families of depth-related concepts: (\textbf{\textit{i}}) \emph{projective geometry} cues, responsive to perspective lines and structural convergence; (\textbf{\textit{ii}}) \emph{shadow-based} cues, reliant on soft lighting gradients; and (\textbf{\textit{iii}}) \emph{local frequency transitions}, responding to abrupt changes in spatial detail or texture~\cite{schubert2021highlowFreq,Ding2023BiPartiteInvariances}, a pattern reminiscent of the \textit{bokeh} concepts identified in~\cite{fel2024understanding}.
Importantly, this taxonomy is not exhaustive: some depth-relevant concepts exhibit mixed sensitivity profiles or combine multiple cues. Full perturbation results for the top 100 depth-relevant concepts are provided in~\cref{app:fig:depth_full}.
Nevertheless, the observed clustering provides evidence that \dino~incorporates a set of interpretable, monocular depth cues. These features, emerging in a self-supervised setting, suggest that \dino's concept geometry supports not only object identity, but also depth reasoning, accessible via simple linear probes.

\section{Token-Type-Specific Concept}
\label{sec:tokenspecific}
Up to this point, we have studied \dino~concepts primarily through their semantic content and task alignment. However, this overlooks a fundamental structural aspect in Vision Transformers: the token types. In ViT architectures~\cite{dosovitskiy2020image}, not all tokens play the same role: \texttt{cls} and \texttt{reg} tokens are explicitly designed for global processing~\cite{darcet2023vision}, while spatial tokens correspond to local image patches. This raises a natural question: \emph{are some concepts specialized for specific token types? And if so, do they occupy distinct subspaces of the concept geometry?}
To address this, we study the footprint of each concept: the distribution of its activations across token positions. For every concept, we compute the entropy of its token-wise activation over 1.4 million images. Concepts with low footprint entropy are highly localized—activating consistently on specific token subsets—whereas high-entropy concepts are spatially diffuse and positionally agnostic.

\begin{figure}[h]
    \centering
    \includegraphics[width=0.75\linewidth]{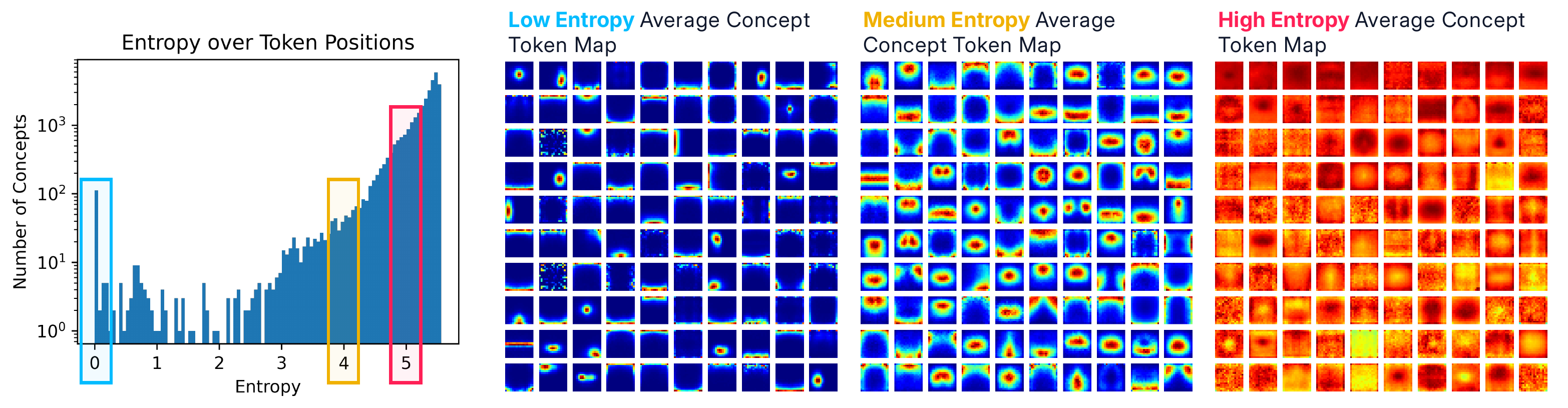}
    \includegraphics[width=0.24\linewidth]{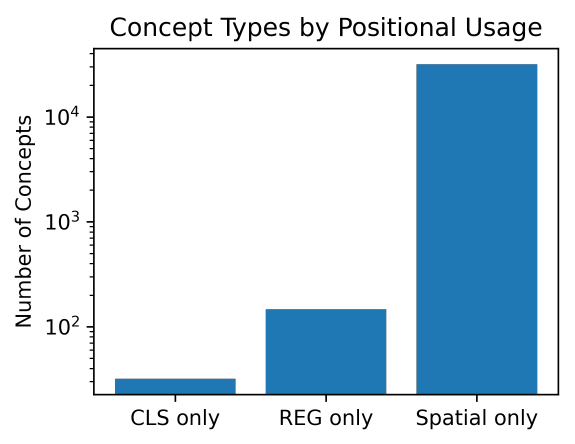}
    \caption{
    \textbf{(Left) Token-wise footprint of concepts.}
    Distribution of concept entropy across token positions. Most concepts are not token specific (high entropy), but a significant tail exhibits strong localization. The average activation maps for low-entropy concepts often reveal spatial edges (e.g., left/right, top/bottom) or special-token specificity. \textbf{(Right) Number of concepts by token-type exclusivity.} While only one concept fires exclusively on the \texttt{cls} token, hundreds are \texttt{reg}-only, indicating substantial specialization beyond positional embedding concept.
    }
    \label{fig:footprint}
\end{figure}

\paragraph{Footprint Types and Special-Token Concepts.}
\cref{fig:footprint} reveals a continuum of footprint entropy across the concept dictionary. While most concepts are positionally agnostic (high entropy), a distinct subset exhibits highly localized activations. We identify three main categories among low-entropy concepts:
(\textbf{\textit{i}}) Position-specific concepts that consistently activate on narrow spatial regions, such as ``left-only'' or ``bottom-only.'' These may reflect residual positional encoding, biases in the training data, or local geometric primitives.
(\textbf{\textit{ii}}) A unique \texttt{cls}-only concept, which fires persistently on the \texttt{cls} token across all images. This concept is closely tied to the positional embedding of the \texttt{cls} token and may act as an ``ID'' or ``passport''-like concepts in the network.
(\textbf{\textit{iii}}) A much more diverse (and surprising) group of \texttt{reg}-only concepts that activate exclusively on the registers tokens. Unlike the \texttt{cls} case, the variety and number of \texttt{reg}-specific concepts cannot be explained by position alone. This suggests the register tokens encode a set of non-spatial features that we explore in the next section.

\begin{figure}[t]
    \vspace{-8mm}
    \centering
    \includegraphics[width=0.95\linewidth]{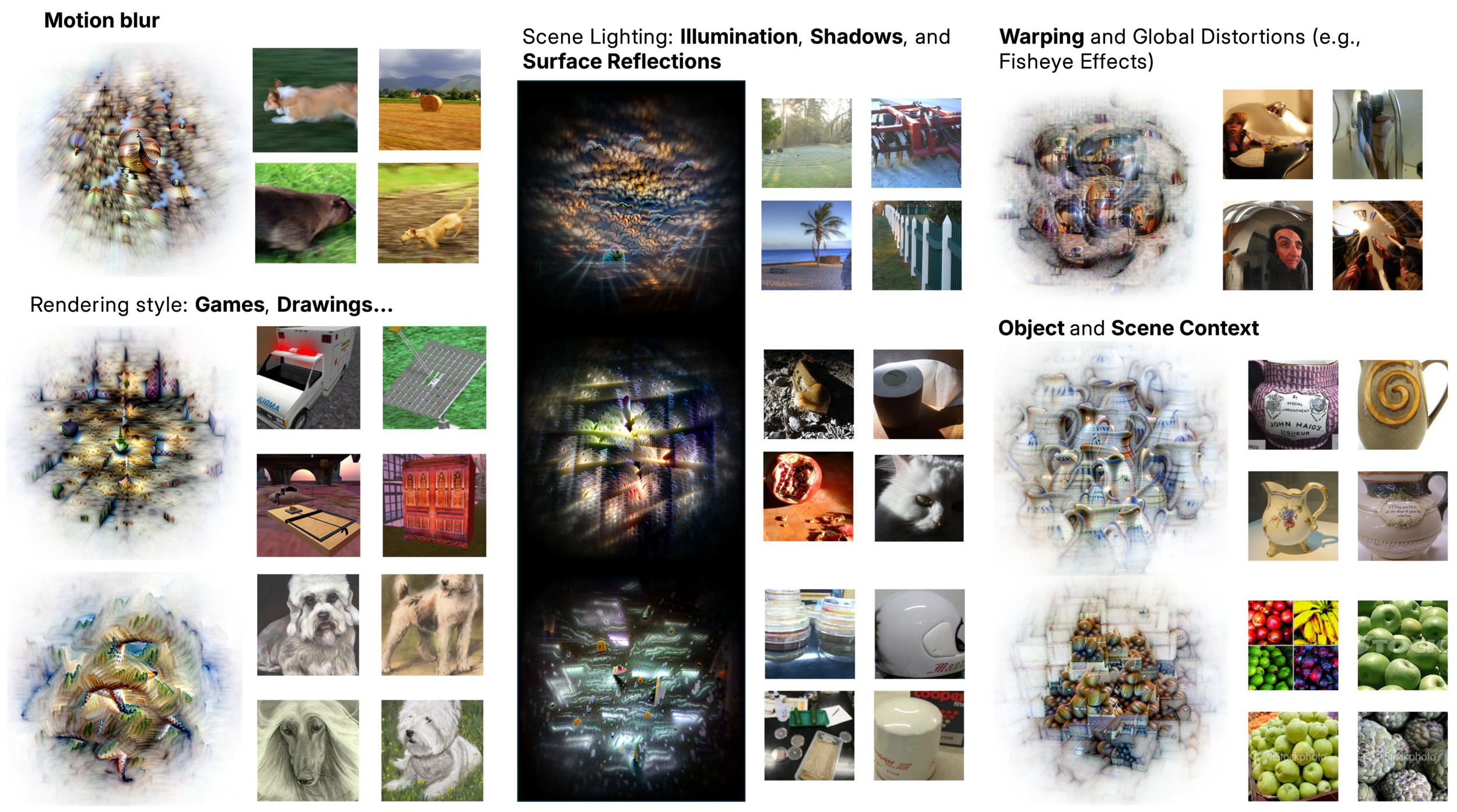}
    \caption{
    \textbf{Register-only concepts capture global, non-local scene properties.}
    Visualization of selected concepts that activate exclusively on \texttt{reg} tokens. These concepts are not object or region specific, but instead encode global properties such as motion blur, illumination, caustics reflections, lens effects, and style. Their emergence suggests that the register token acts as a potential conduit for abstract scene level information.
    }
    \label{fig:registers}
\end{figure}

\paragraph{Register-Only Concepts and Global Scene Properties.}
Upon inspecting the register-only concepts, an interesting pattern emerges: these concepts do not align with localized object parts or semantic categories. Instead, they seem to encode holistic, global attributes of the image. As shown in Figure~\ref{fig:registers}, register-only concepts respond to phenomena such as lighting style, motion blur, caustic reflections, or artistic distortion. Some even appear sensitive to camera properties (e.g., wide-angle warping or depth-of-field effects).

Interestingly, this specialization is highly asymmetric: while only a single concept activates exclusively on \texttt{cls}, hundreds specialize for \texttt{reg}. This suggests that DINOv2 not only encodes high-level features, but also distributes them across structurally distinct token pathways.

\clearpage
\section{Statistics and Geometry of Concepts}
\label{sec:statistics_and_geometry_of_concepts}
\begin{figure}[t]
    \vspace{-8mm}
    \centering
    \includegraphics[width=0.63\linewidth]{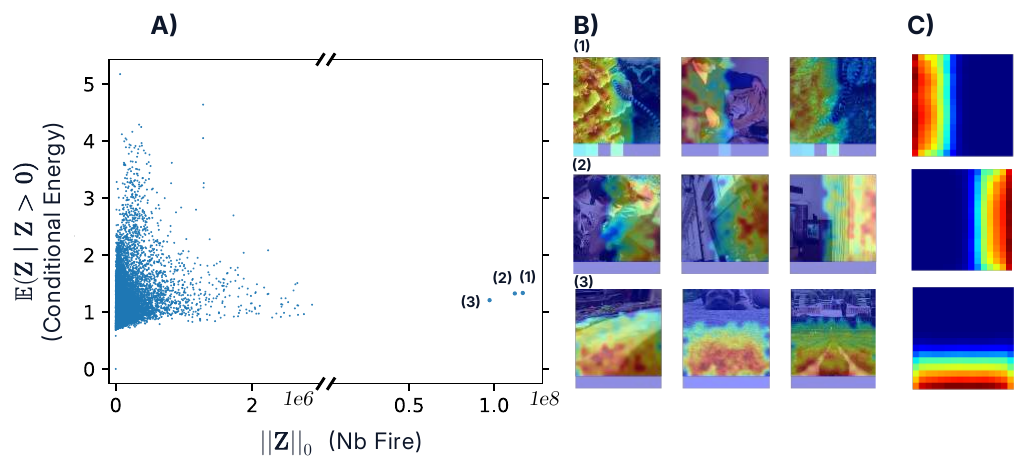}
    \includegraphics[width=0.35\linewidth]{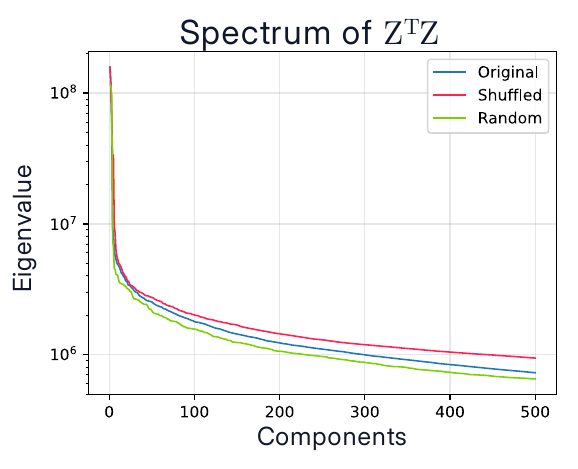}
    \caption{
    \textbf{A)} Conditional energy $\mathbb{E}(\Z_i \mid \Z_i > 0)$ versus number of firings (i.e., datapoints where $\Z_i > 0$) for each concept. 
    Most features follow a frequency–energy trade-off: rarely firing concepts tend to allocate higher energy per activation, while frequently active ones remain weaker. 
    The resulting triangular envelope indicates that concepts operate at distinct energy regimes, contributing unequally to the overall norm of the reconstructed activation: some features dominate norm-wise.
    Three outliers (labeled 1, 2, 3) deviate from this pattern by exhibiting dense~\cite{sun2025dense}, dataset-wide activation at relatively low energy.
    \textbf{B)} These outlier concepts encode persistent positional features (e.g., whether a token lies on the left, right, or bottom), suggesting the presence of dense representations within an otherwise sparse latent structure.
    \textbf{(Right)} Spectrum of the normalized concept co-activation matrix $\Z^\tr\Z$. The eigenvalue decay is smooth, with no sharp gaps or dominant modes, indicating the absence of block structure or low-rank modularity. Relative to random and shuffled baselines (see \cref{app:baseline_ztz} for details), the empirical spectrum retains broader variance across many components, reflecting high-dimensional patterns. This suggests that the concept space is richly connected and distributed, resisting simple clustering while still deviating from chance structure. 
    }
    \label{fig:stastistics}
\end{figure}

\paragraph{Concept Occurrence Statistics.} 
Having identified which concepts are recruited by downstream tasks, we next examine their broader statistical and geometric organization. 
\cref{fig:stastistics} (left) shows, for each concept, its \emph{conditional energy} $\mathbb{E}(\Z_i\mid\Z_i > 0)$, the expected activation given that it fires, plotted against its firing count, i.e., the number of datapoints where $\Z_i$ is nonzero. 
For the majority of concepts, the points are distributed within a triangular envelope: even at comparable firing frequencies, some concepts exhibit systematically higher activation energy than others. This reveals that concepts occupy distinct norm regimes, i.e. their contributions to the overall norm of the activation can differ markedly. 
In other words, certain features dominate (norm wise), whereas others remain energetically less important.%

However, superimposed on this sparse backdrop are three outlier concepts that exhibit an anomalously dense activation pattern~\cite{sun2025dense}, firing across the entire dataset. Upon inspection, these concepts are not arbitrary artifacts, but instead they encode positional information: specifically indicating whether a token lies on the left, right, or bottom of the input image grid. Their ubiquity reflects the persistent relevance of spatial structure within the model, and their emergence aligns with our decoding experiments, which reveal that positional information remains linearly decodable up to the intermediate layers (see \cref{app:position}) and form a 2-dimensional subspace at the penultimate layer (see~\cref{sec:shape_image}). 
This finding adds nuance to the assumption that the concept space is strictly sparse. Instead, we observe a hybrid representational regime where a small set of universally-active, (here spatially grounded features) coexists with a broader ensemble of highly selective, image-contingent activations, resembling sparse and dense low-rank representations~\cite{jiang2014sparse,pramanik2020deep}. 

\paragraph{Concepts Co-occurence.} The Gram matrix of concept co-activation $\bm{G} = \Z^\tr\Z \in \R^{c \times c}$, as shown in \cref{fig:stastistics} (right), has a spectrum that decays smoothly, with no gaps or dominant modes, providing little evidence for modular low-rank structure. 
As reference points, we create two baselines, one from random activations $\Z$ with same sparsity (random), and one directly shuffling the coactivation matrix $\bm G$ (shuffle). 
Relative to baselines (\cref{app:baseline_ztz}), it lies between the steep decay of the random surrogate and the inflated tail of the shuffled one. %
Below both would imply redundancy, above both extreme diffuseness; its intermediate position instead reflects a high-dimensional, distributed regime where concepts co-activate without forming clear clusters.

\begin{figure}[t]
    \vspace{-8mm}
    \centering
    \includegraphics[width=0.99\linewidth]{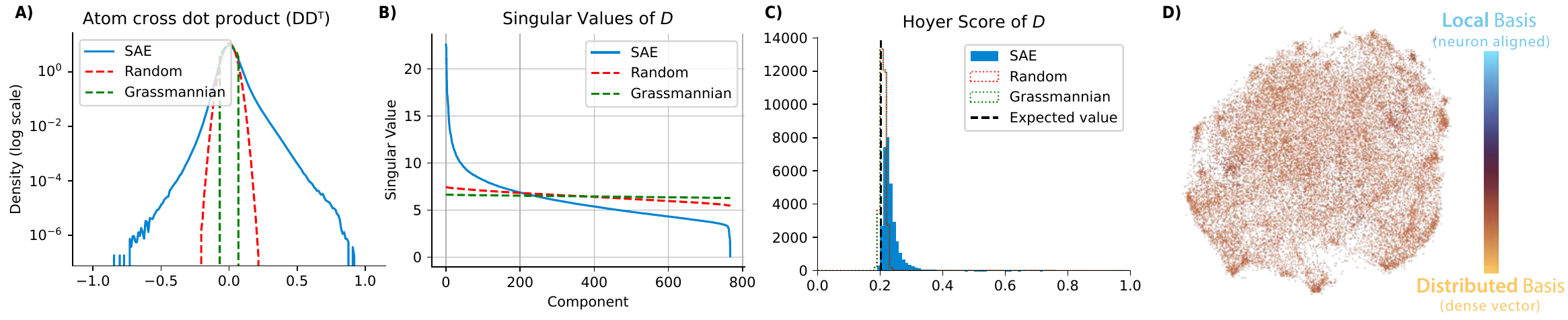}
    \caption{\textbf{Geometry of concept vectors.}
    \textbf{(A)} Distribution of pairwise inner products in the learned dictionary, compared to Random and Grassmannian baselines. While the Grassmannian frame approximates the theoretical optimum for mutual incoherence, and the Random dictionary reflects a typical null model, the SAE dictionary exhibits heavier tails and a flatter central peak—indicative of higher coherence. This suggests that the SAE geometry departs from the ideal of a uniformly incoherent basis, suggesting a clustered structure.
    \textbf{(B)} Singular value spectrum of the dictionary. SAE exhibits sharper decay than both baselines, suggesting lower effective rank and anisotropy: some directions are densely populated while others are sparse or unused.
    \textbf{(C \& D)} Hoyer sparsity scores and a UMAP embedding of dictionary atoms colored by Hoyer score. Scores remain below the neuron-aligned limit (1.0), indicating that concepts are not axis-aligned. Compared to baseline expectations (e.g., $1 - \sqrt{2/\pi} \approx 0.202$ for Gaussian vectors), the SAE atoms are slightly sparser, still supporting the notion of distributed representations.
    Together, these deviations from idealized Grassmannian dictionaries may reflect emerging structural organization—such as specialization or local orthogonality, potentially aligned with functional roles explored in \cref{sec:tasks}.
    }
    \label{fig:dictionary}
\end{figure}

\paragraph{Geometric Organization.}
Having examined the statistics of activations, we now turn to the geometry of the learned dictionary.\cref{fig:dictionary} summarizes structural properties of the concept basis $\D$, including pairwise dot products, singular value spectrum, and sparsity scores.
\textbf{Local Coherence.} Panel (\textbf{A}) reveals that relative to the isotropic random vectors (e.g. at initialization), the dictionary does not move towards the uniform incoherence of an ideal Grassmannian frame. Instead, the SAE dictionary moves to the other direction and exhibits higher coherence: while most atoms are weakly correlated, a non-negligible fraction form tight clusters with stronger alignment. This geometry departs from the assumptions of the LRH. Rather than approximating a Johnson–Lindenstrauss like embedding~\cite{larsen2017optimality},\footnote{The Johnson–Lindenstrauss (JL) lemma~\cite{johnson1984extensions} implies that \(p = \exp(\mathcal{O}(m \varepsilon^2))\) vectors can be embedded into \(\mathbb{R}^m\) such that all pairwise inner products are preserved up to a small distortion \(\varepsilon > 0\). This supports the feasibility of near-orthogonality even in regimes where \(p \gg m\).} the model appears to favor structured redundancy, potentially to support specialization or compositional reuse.
Moreover, a pattern of interest emerges upon inspecting the tail of the dot-product distribution: a small but significant number of concept pairs are nearly antipodal, i.e., $\D_i \approx -\D_j$ (see~\cref{fig:geometry_patterns}, A). This go against the assumption of a Grassmannian-like dictionary ~\cite{strohmer2003grassmannian}. Instead, it suggests that \dino~sometimes exploits paired directions to encode semantically opposed features such as ``left vs right'' or ``white vs black'' along shared lines but with opposite signs. Practically, such antipodal structure should be accounted when using cosine similarity in downstream analyses, particularly in data filtering, clustering, or retrieval contexts.
\textbf{Global Anisotropy.} The singular value spectrum in Panel (\textbf{B}) shows a sharp decay, indicating that the dictionary spans a low-dimensional subspace despite its overcomplete size. This again suggests an anisotropic allocation of representational capacity: some directions are densely populated, while others are sparsely or not at all represented. One potential explanation is that the activations $\A$ themselves have anisotropic covariance, thus to reconstruct them $\A\approx \Z\D$ with L2 error, more atoms are allocated to the more important and higher variance subspace.\footnote{A speculative hypothesis is that the affine components of the last LayerNorm deforms the activation space into an ellipsoidal geometry, locally increasing curvature along some directions. In such regions, dictionary learning may require more atoms to adequately cover the manifold with low reconstruction error. While this remains to be formalized, similar intuitions appear in manifold tiling and sparse coding contexts~\cite{peyre2009manifold}.}  %
\textbf{Distributed Encoding.} Finally, Panels (\textbf{C}) and (\textbf{D}) show that the dictionary atoms are not aligned to individual units: their Hoyer scores lie far below the one-hot bound. Compared to random baselines, SAE atoms are slightly sparser but remain broadly distributed—supporting the view that concept representations are shared across the population rather than axis-aligned~\cite{elhage2022superposition,colin2024local}.

\begin{figure}[t]
    \vspace{-8mm}
    \centering
    \includegraphics[width=0.48\linewidth]{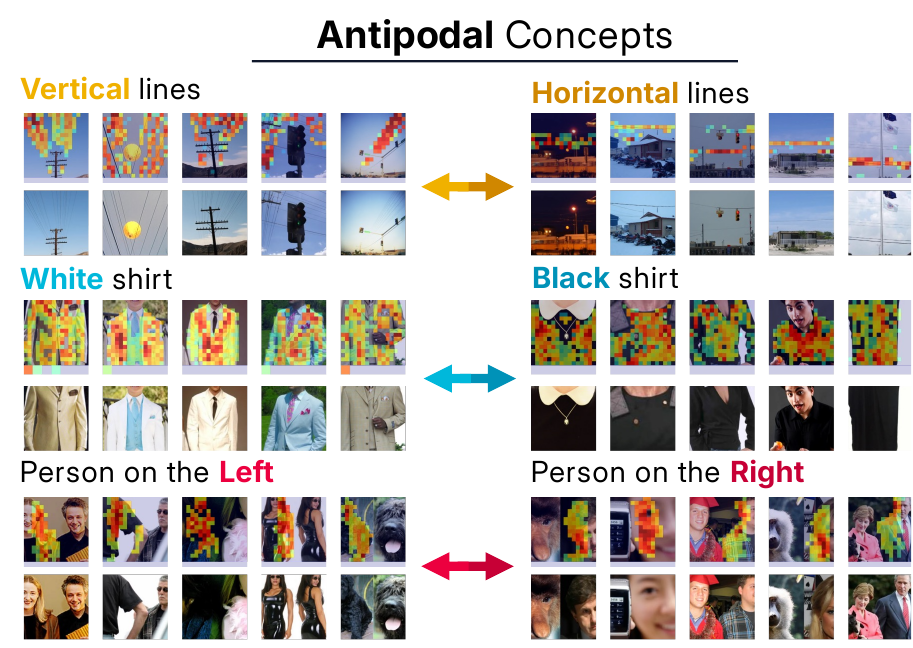}
    \hfill
    \includegraphics[width=0.50\linewidth]{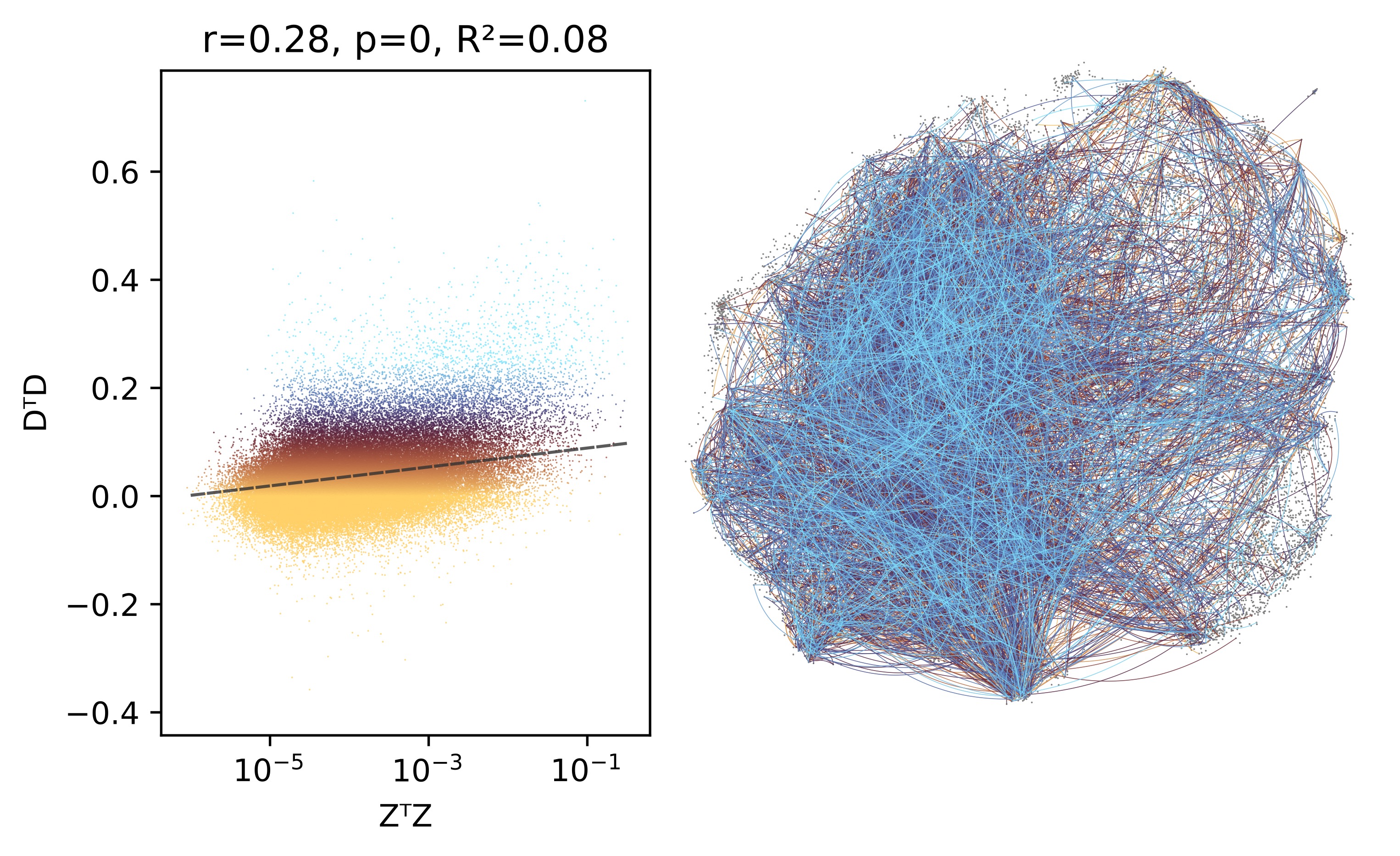}
    \caption{
    \textbf{Emergent structure in concept geometry.}
    \textbf{(A)} Antipodal concept pairs: examples where \(\D_i \approx -\D_j\). Many such pairs correspond to semantically opposed patterns—e.g., Vertical vs. Horizontal lines, or White vs. Black shirts. Despite their opposition, the vectors are nearly colinear with opposite signs, suggesting the model uses polarity to encode meaning.
    \textbf{(B) Concept geometry is only moderately shaped by co-activation statistics}. Left panel: correlation between concept co-activation (\(\Z^\tr \Z\)) and geometric similarity (\(\D \D^\tr\)) shows a weak relationship (\(r = 0.28\), \(R^2 = 0.08\)). Right panel: UMAP embedding overlaid with high co-activation links reveals widespread non-local connections and no clear modular organization, suggesting that usage statistics do not strictly determine spatial proximity in concept space.
    }
    \label{fig:geometry_patterns}
\end{figure}

\begin{figure}[t]
    \centering
    \includegraphics[width=\linewidth]{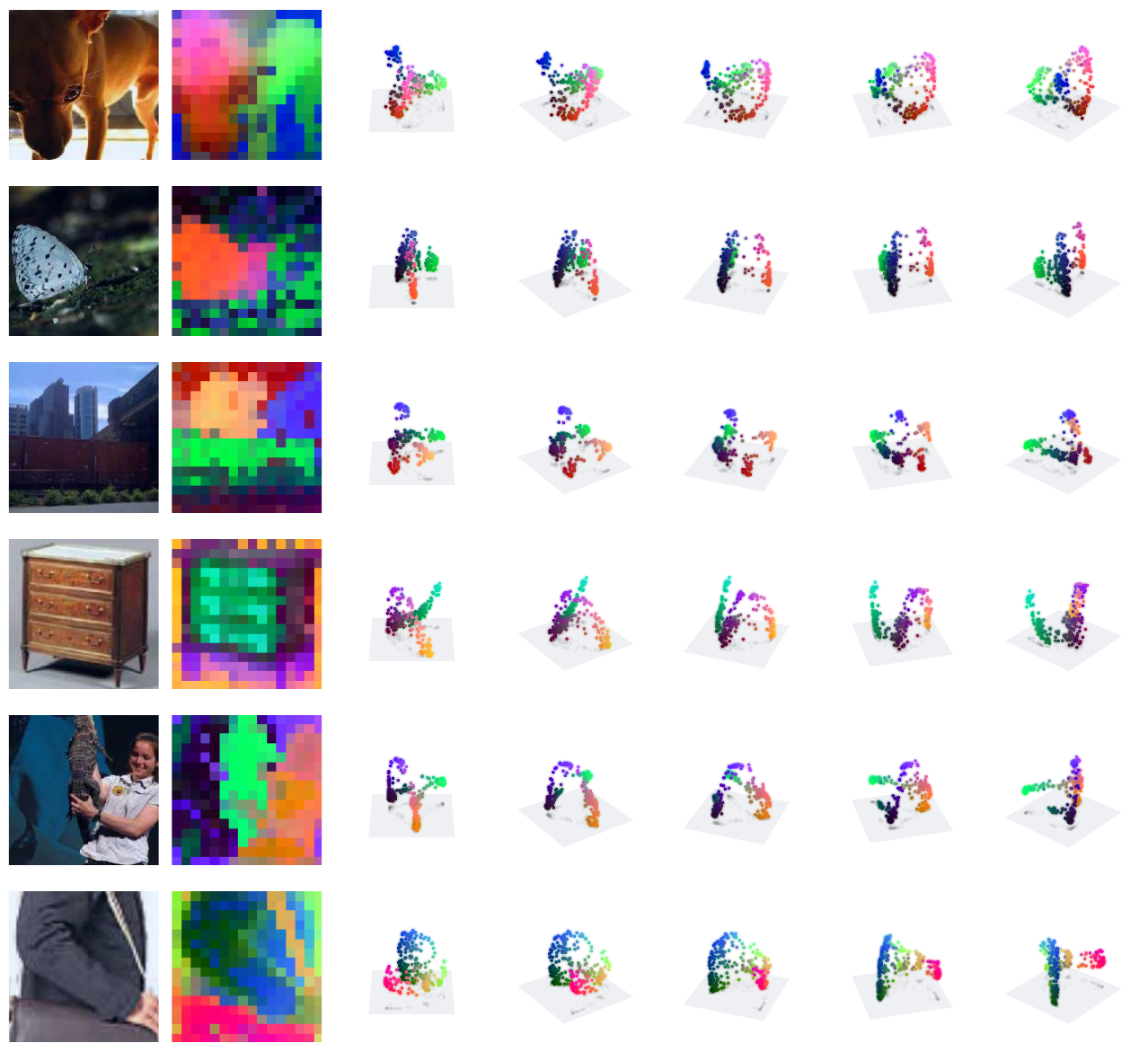}
    \caption{%
        \textbf{Local embedding geometry of patch tokens across random ImageNet samples.}
        Each panel shows a raw ImageNet image alongside its token-level geometry. Patch embeddings from \dino~are projected via PCA (on a per-image basis) and visualized by assigning the first three principal components to RGB color (normalized to $[0,1]$). Despite the lack of supervision, token embeddings appear to lie on smooth, low-dimensional manifolds, with transitions often aligned to object boundaries or perceptual contours. Importantly, PCA is an affine (linear) transformation and cannot generate curvature on its own, implying that the observed smoothness reflects genuine structure in the embeddings. 
    }
    \label{fig:local_geometry_pca1}
\end{figure}

\paragraph{Concept geometry is only weakly shaped by co-activation statistics.}
A natural hypothesis to explain this structure is that the geometry of the concept dictionary reflects patterns of co-activation, that is, concepts frequently used together are embedded nearby in representational space. In this view, co-occurrence of concepts would ``bend'' their geometry, and we should see clusters of co-occurring concepts.
To evaluate this, we compare the co-activation matrix \( \Z^\top \Z \), which captures how often pairs of concepts are jointly active, with the geometric affinity matrix \( \D \D^\top \). As shown in~\cref{fig:geometry_patterns}B, the two matrices exhibit only weak but statistically significant correlation. This suggests that while usage patterns exert some influence on geometry, they are not the dominant organizing principle. \footnote{Algebraically, It's interesting to notice that the correlation is roughly proportional to the trace of activation covariance $\mathrm{corr}(\Z^\top \Z, \D \D^\top) \propto \mathrm{tr}[\Z^\top \Z\ \D \D^\top]=\mathrm{tr}[\D^\top\Z^\top \Z\ \D ]\approx \mathrm{tr}[\A^\top \A]\propto\mathrm{cov}(\A)$, which is guaranteed to be positive. So this positive relation might be an intrinsic property of linear reconstructive methods}
Additional anecdotal evidence comes from the global structure of the concept space: in a UMAP embedding of the dictionary (\cref{fig:geometry_patterns}, right), edges representing top co-activations form a tangled, non-local graph rather than local module that would indicate close and co-activating concepts.

Together, these results characterize the learned dictionary as neither maximally incoherent nor uniformly distributed. Instead of approximating an optimal Grassmannian frame (which, again, would minimize pairwise correlations and maximize representational diversity) the SAE converges to a geometry that exhibits higher coherence, reduced effective rank. This configuration suggests that the model is not merely optimizing for feature packing or uniform coverage of the representational space. 
Rather, it allocates representation disproportionately across the dictionary: a significant fraction of atoms exhibit high pairwise alignment, and their span concentrates in lower-dimensional subspaces. 
As we have seen in Section~\ref{sec:tasks}, these directions often correspond to concept groups that are co-activated by specific downstream objectives, indicating a functional partitioning of the dictionary aligned with task demands.

Our analysis shows that while sparse coding provides a useful starting point, 
several observations sit uneasily with this view. 
We find dense positional features alongside sparse ones, coherent clusters and antipodal pairs rather than near-orthogonal atoms, and low-dimensional task-specific subspaces. 
These patterns are difficult to reconcile with a purely linear superposition model, and instead point toward additional geometric constraints. 
To probe this structure more directly, we next examine local image-level geometry.

\clearpage
\section{The Shape of an Image}
\label{sec:shape_image}

\begin{figure}[t]
    \vspace{-8mm}
    \centering
    \includegraphics[width=0.95\linewidth]{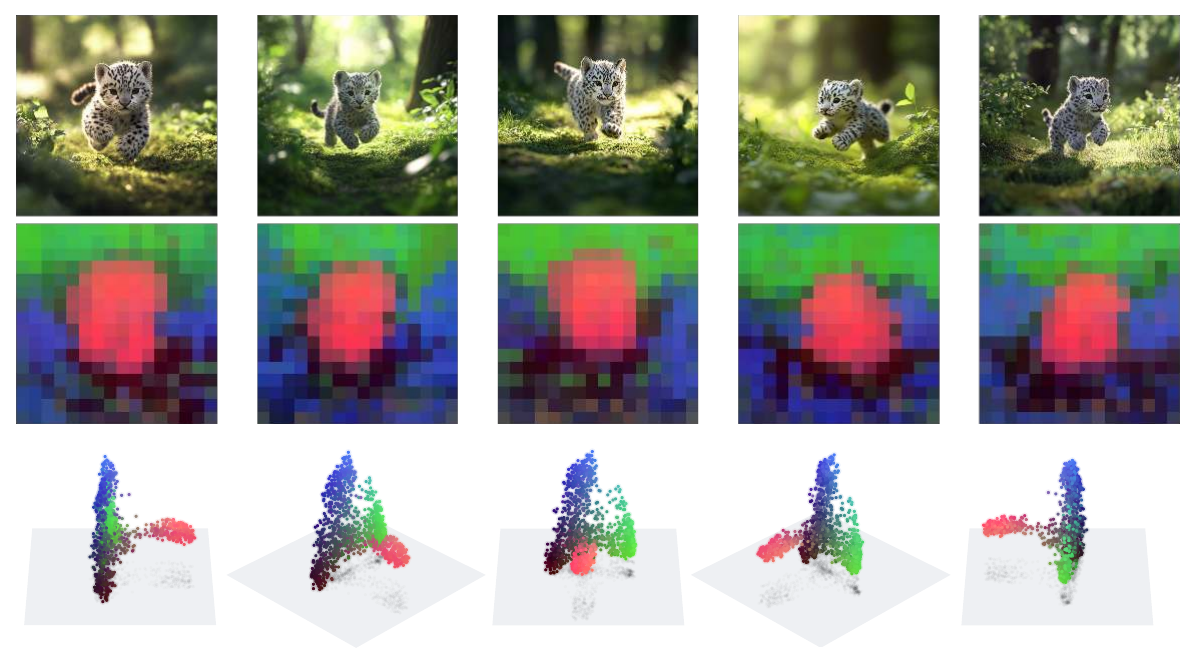}
    \caption{
    \textbf{PCA maps of \dino~patch embeddings reveal locally connected, semantically aligned structure.}
    \textbf{(Top)} Original images. \textbf{(Bottom)} PCA heatmaps of the top three components of patch-token embeddings encoded as $(r,g,b)$ values. Despite no supervision for localization or segmentation, the PCA projections consistently delineate object shapes with smooth transitions across the image grid. Since PCA is a linear operator, it cannot fabricate curvature; the observed connectedness therefore reflects genuine structure in the embeddings. Moreover, the alignment of token geometry across distinct images suggest that the representation is not purely relative (i.e., based only on distances within each image tokens).
    }
    \label{fig:pca_animals}
\end{figure}

All previous sections examined \dino~representations through the lens of dictionary learning. Yet several observations point to additional structure beyond a strictly quasi-orthogonal sparse superposition: anisotropic clustering, low-rank task subspaces, antipodal semantic axes, and locally connected neighborhoods (see \cref{fig:pca_animals}). Motivated by these signals, we now turn to the ``model-native'' geometry of tokens, without imposing a dictionary factorization. Specifically, we ask: within a single image, how are patch tokens organized in the high-dimensional embedding space? This geometry is often visualized using PCA, where token embeddings are projected onto top components and shown as colored 2D maps~\cite{oquab2023dinov2,darcet2025cluster}, which are widely used to illustrate semantic structure and even serve as compact inputs to student models in distillation pipelines~\cite{kouzelis2025boosting}.

Our goals in this section are twofold: (i) to demystify what PCA actually captures in single-image embeddings, particularly the origin of the observed ``smooth'' structure; and (ii) to characterize token-level geometry across layers without presupposing a sparse, near-orthogonal coding model.

\paragraph{Just Position?}
A natural hypothesis is that the local connectedness observed in PCA maps, illustrated in \cref{fig:pca_animals} and \cref{fig:local_geometry_pca1}, is inherited from positional encodings. After all, spatial position is the only structured signal available at initialization, and remains entangled with token identity throughout the network.
To test this, we begin by analyzing the evolution of positional information across layers. We trained linear decoders to decode the token position from each layer's activations and extracted the corresponding decoding vectors at each spatial location $(i,j)$ to form a position encoding matrix $\bm{P} \in \R^{256 \times d}$ at each layer (see \cref{app:position} for more details). In early layers, this positional embedding matrix $\bm{P}$ is effectively high-rank: we can linearly decode the precise $(i,j)$ location of a token with near-perfect accuracy (see Fig.~\ref{fig:pos_basis}, Top). This aligns with the model's need for precise positional resolution during early attention operations. However, as shown in Fig.~\ref{fig:pos_basis} (top, right), the rank of the positional subspace decreases sharply over depth, eventually collapsing to a 2D plane in the final layers. This compression is reminiscent of a transition from place cell-like coding schemes~\cite{OKeefe1971Hippocampus}, where individual units encode specific spatial locations, to axis coding~\cite{chang2017code} where each dimension encodes the continuous value of horizontal or vertical coordinates. This compression is also visible in the PCA projection of positional embeddings themselves (\cref{fig:pos_viz_across_layers}) that eventually collapse to a smooth 2D sheet in the final layers.

\begin{figure}[t]
    \vspace{-5mm}
    \centering
    \includegraphics[width=0.35\linewidth]{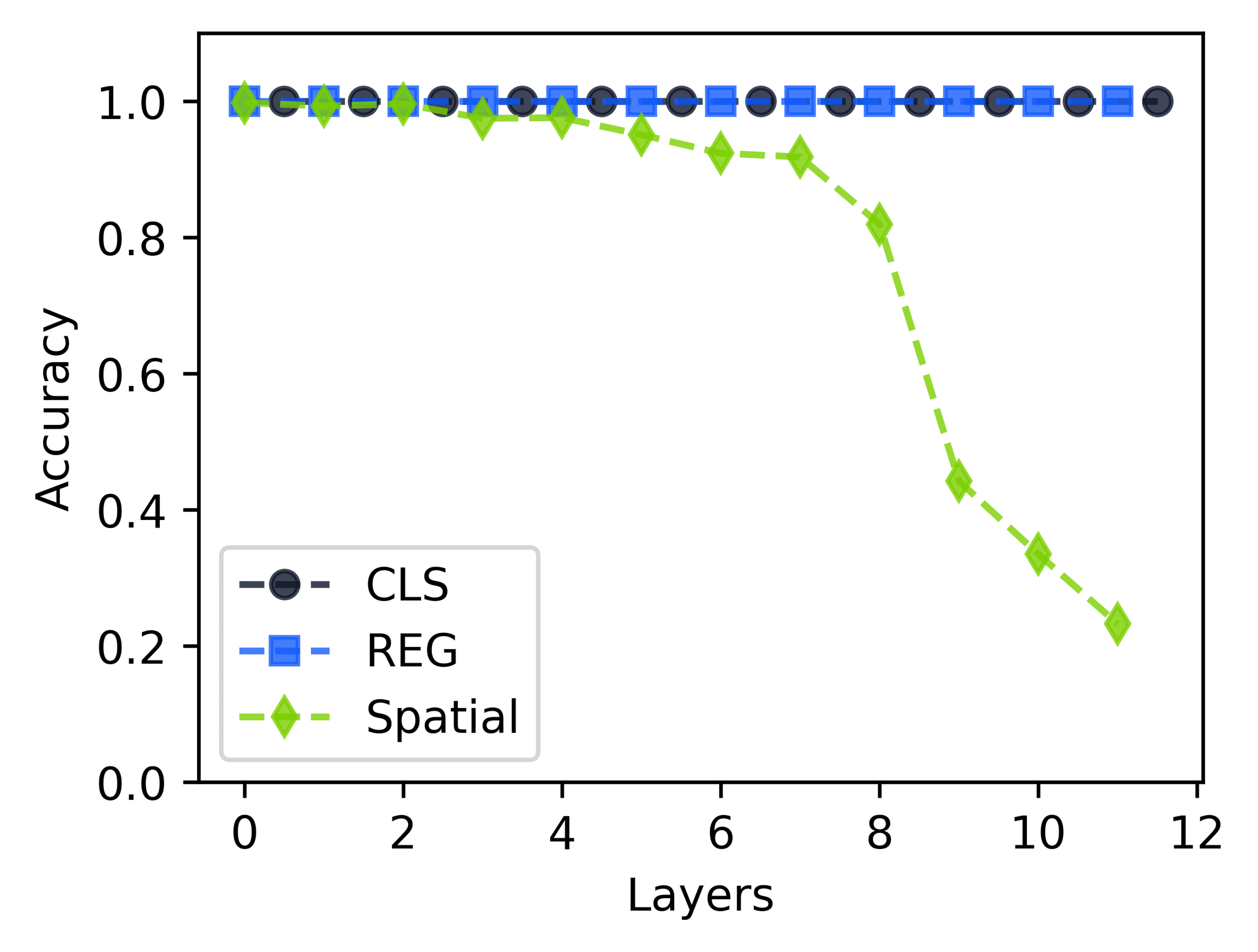}
    \includegraphics[width=0.60\linewidth]{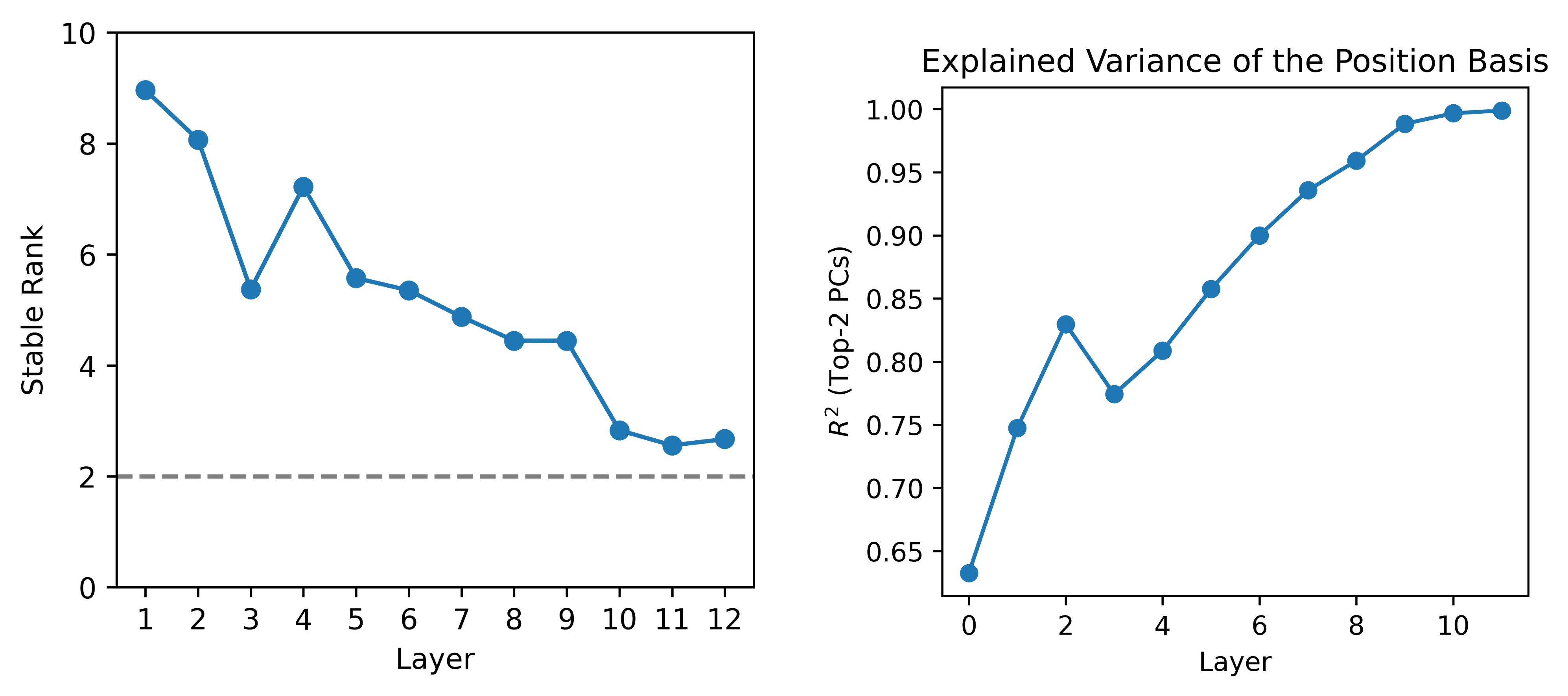}
    \includegraphics[width=0.95\linewidth]{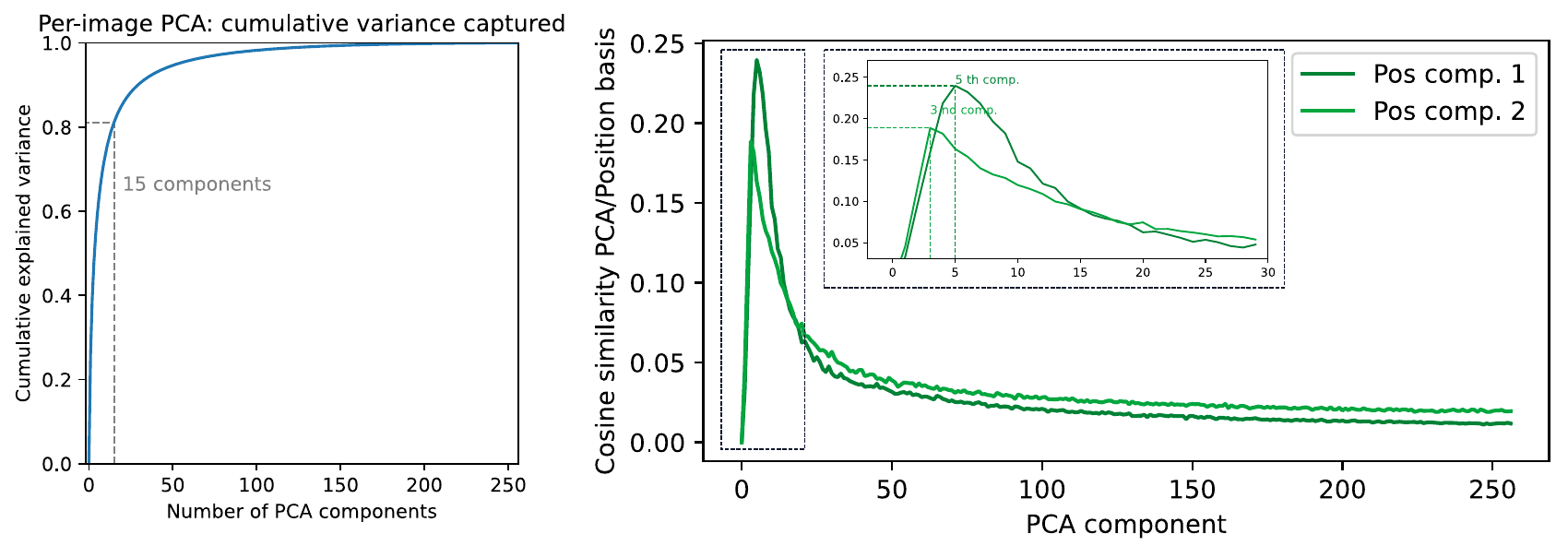}
    \caption{
    \textbf{(Top) Positional encodings compress over layers}
    (Left) Linear decodability of token position across layers: precise spatial coordinates are recoverable up to layer 8, after which accuracy collapses. \texttt{cls} \& \texttt{reg} tokens are linearly decodable until the end. (Right) The rank of the positional embedding subspace steadily drops, converging to a 2d subspace by the final layers—indicating strong compression. \textbf{(Bottom)} Few PCA components of single images tokens explain the variance of tokens across the entire image, suggesting they are lying in a low dimensional subspace. On average, the image-wise PCA shows that position basis correlates with components 3 and 5, but not the dominant directions. The ``smoothness'' of the PCA maps persists even when positional components are removed, suggesting that this reflects deeper geometric organization beyond position alone (for a qualitative example see \cref{app:position}).
    }
    \label{fig:pos_basis}
\end{figure}

This compression is also visible in the PCA projection of positional embeddings themselves (\cref{fig:pos_viz_across_layers}). The positional basis contracts into a two-dimensional structure, possibly reflecting the emergence of coarse spatial priors. While this may contribute to the overall connectedness of PCA maps, it cannot fully explain them.
Indeed, when we project token embeddings from actual image inputs, the principal components that correspond to position appear only in intermediate directions—typically around the 3rd to 5th component (see~\cref{fig:pos_basis}, bottom). Moreover, PCA maps remain locally connected even after removing positional components (see~\cref{app:fig:pca_no_pos}). This indicates that PCA is capturing something beyond explicit position.

\begin{figure}[h]
    \vspace{-1mm}
    \centering
    \includegraphics[width=0.95\linewidth]{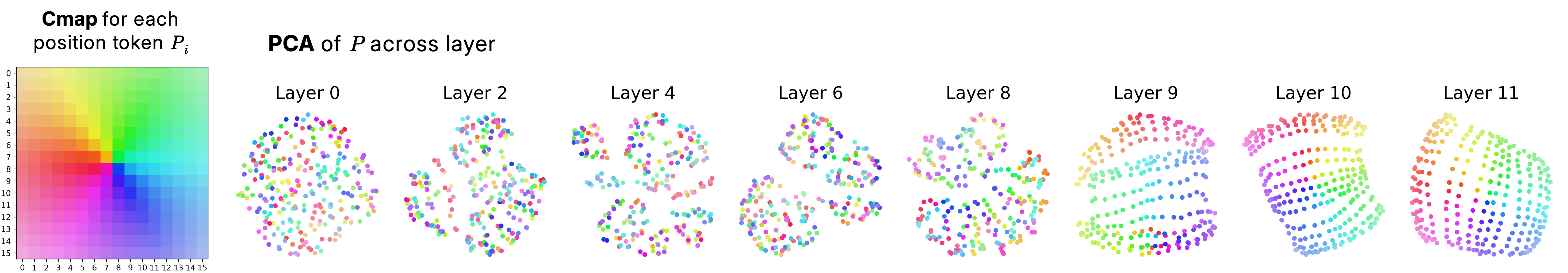}
    \vspace{-3mm}
    \caption{
    \textbf{Visualization of positional encoding across layers reveals compression to a 2D sheet.}
    PCA projections of the positional encoding vectors at different layers show a clear evolution: from high-rank, dispersed geometry in early layers (\cref{fig:pos_basis}) to a low-dimensional sheet in the final layers.}
    \vspace{-5mm}
    \label{fig:pos_viz_across_layers}
\end{figure}

\paragraph{Toward interpolative geometry.}
\cref{fig:pca_animals} shows that, across many images, patch tokens occupy a consistent low dimensional and locally connected set that aligns with object and coherent regions. 
First, this stability across images suggest that embeddings are not purely relative. 
Second, we saw that the connectedness is not explained by position alone, since principal components that encode position have intermediate rank and PCA maps remain locally connected after removing the positional subspace (see~\cref{fig:pos_basis} and \cref{app:fig:pca_no_pos}). 
What remains is the possibility that tokens interpolate smoothly between a discrete set of anchor points or landmark representations.

This interpolative structure finds natural support in DINOv2's training objectives. The global DINO head matches teacher representations using 128k prototypes, while the iBOT~\cite{zhou2021ibot} head performs masked prediction with its own 128k prototype vocabulary\footnote{The DINOv2 architecture employs two distinct prototype-based heads: (1) a DINO head consisting of a \texttt{4096→4096→256} MLP that outputs scores over 128k prototypes for image-level self-distillation, and (2) an iBOT head with a \texttt{4096→4096→256} MLP outputting scores over 128k prototypes for masked token prediction. Both heads apply softmax normalization to produce probability distributions over their respective prototype vocabularies, effectively implementing implicit clustering with soft assignments.}.
The Kozachenko-Leonenko regularizer~\cite{sablayrolles2018spreading} further encourages entropy maximization, promoting uniform distribution across the representation space while avoiding collapse. 
Rather than learning arbitrary smooth manifolds, this setup naturally favors representations that can be expressed as mixtures of multiple prototype sets -- each token becomes a weighted combination of landmarks from different prototype vocabularies.
This suggests that the observed connectedness may arise not from continuous interpolation per se, but from discrete mixing of archetypal points, where the ``interpolation'' reflects probabilistic combinations across multiple archetypal systems.
To formalize this geometric intuition, the next section investigates how such landmark-based representations can be recovered and whether attention mechanisms naturally generate this structure.

\section{\hyp}
\label{sec:hypothesis}

\begin{figure}[t]
\vspace{-8mm}
\centering
\includegraphics[width=0.75\linewidth]{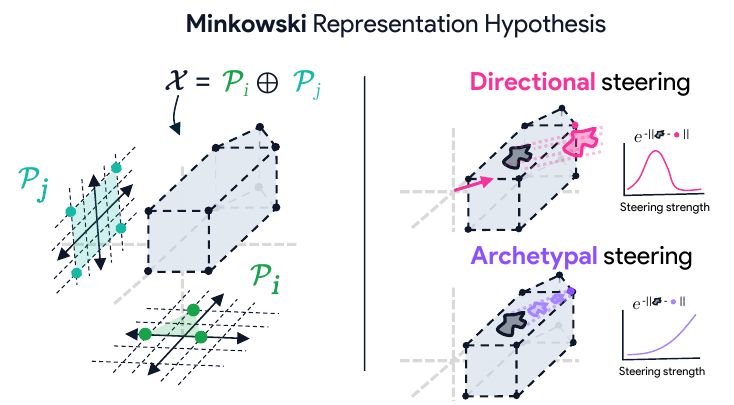}
\caption{
\textbf{Schematic view of the Minkowski Representation Hypothesis.}
\textbf{(Left)} The activation space $\mathcal{X}$ is represented as the Minkowski sum $\mathcal{X} = \mathcal{P}_i \oplus \mathcal{P}_j$ of polytopes formed by archetypal points (colored dots). 
Each polytope $\mathcal{P}_i = \conv(\A_{\mathcal{T}_i})$ is the convex hull of archetypes within tile $\mathcal{T}_i$, shown as the lower geometric structure.
This Minkowski sum structure naturally emerges from multi-head attention, where each head produces convex combinations of its value vectors, and the final output layer additively combines these convex hulls from different heads.
\textbf{(Right) Consequences for steering.} Under directional steering (\textcolor{pink}{\textit{Pink}}), interventions follow unbounded linear directions that quickly leave the manifold, with steering strength first increasing the proximity to the archetype before decreasing it when moving away from valid representations.
Under archetypal steering (\textcolor{purple}{\textit{Purple}}), interventions move toward specific archetypal landmarks within the polytope structure, with proximity increasing with steering strength.
}
\label{fig:minkowski_hypothesis}
\end{figure}
While hypotheses about representational composition are rarely stated explicitly, the geometry assumed by our methods is foundational: it constrains both their validity and the phenomena they can reveal.
If sparse autoencoders implicitly answer to the Linear Representation Hypothesis, then we understand the importance of specifying the correct ambient geometry: it not only conditions how we interpret representations, but also determines the very form that our analytical methods of concept extraction can take.

Armed with our preceding observations, we contend that an alternative account can explain the phenomena we documented, in particular the interpolative geometry within single images. 
Our motivation traces to Gärdenfors’ conceptual spaces, where concepts inhabit convex regions along geometric quality dimensions \cite{gardenfors2004conceptual}. 
Put plainly, we have reasons to believe that the observed interpolation is the surface of a deeper organization: the activation space behaves as a sum of convex hulls~\footnote{This aligns with evidence that concept structure can be convex and compositional in other domains~\cite{park2025geometry}}. 
A single attention head already performs convex interpolation over its values, creating an archetypal geometry; multi-head attention then aggregates these convex pieces additively, yielding a Minkowski sum. 
It is natural to imagine one hull reflecting token position, another depth, another object or part category, so that the final activation is the sum of these convex interpolations, and the ``concepts'' available to probes are the archetypes governing each hull. 
In this section, we will start by making this idea explicit, then we will review some theoretical evidence based on previous works and showing how simple attention blocks generate such geometry, Then, we will follow showing some empirical signals that make the proposal a plausible candidate and the implication of such geometry. 
We now formally define the \hyp.
\begin{definition}{\textbf{Minkowski Representation Hypothesis (MRH).}}
Let $\mathcal{X} \subset \mathbb{R}^{d}$ be a layer’s activation space and $\x \in \mathcal{X}$. Let $\Arch = (\bm{a}_{1},\ldots,\bm{a}_{c}) \in \mathbb{R}^{c\times d}$ be an Overcomplete Dictionary of Archetypes ($c\gg d$).
We partition the archetypes into $m$ disjoint tiles $\{\T_i\}_{i=1}^m$, $\T_i\subset\{1,\ldots,c\}$, and define the tile polytopes:
\[
\P_i=\conv(\Arch_{\T_i}) \quad \text{with} \quad \Arch_{\T_i}=\{\bm{a}_j: j\in\T_i\}.
\]
Then $\mathcal{X}$ satisfies MRH if
\[
\begin{cases}
\text{\textit{(\textbf{i})~~~Minkowski sum:}} &
\mathcal{X}=\oplus_{i=1}^{m}\P_i.\\
\text{\textit{(\textbf{ii})~~Block-convex coding:}} &
\x = \sum_{i \in S} \z_i \Arch_{\T_i},\ \ \z_i \in \Delta^{|\T_i|},\ \ |S| \ll m. \\
\end{cases}
\]
where $\Delta^{k} = \{\bm{z} \geq 0 : \mathbf{1}^{\top}\bm{z} = 1\}$.
\label{def:minkowski}
\end{definition}
Put simply, MRH says that a point $\x$ is a sparse composition of concept regions, with only a few tiles active $|S| \ll m$, and each active tile contributing a convex combination of its archetypes. For example, a token representing a rabbit might combine: (1) a convex mix from an ``animal category'' tile (capturing rabbit-like features), (2) a convex mix from a ``spatial position'' tile (left/center/right), and (3) a convex mix from a ``depth'' tile (foreground/background). The final activation is the sum of these convex contributions.
More formally, the activation set is generated by Minkowski addition of distinct tile polytopes $\P_i=\conv(\Arch_{\T_i})$, with a union over attention patterns when different tiles are selected across inputs, rather than a single global linear space: representations live in localized convex regions whose sums capture the interpolative geometry we observe. Before turning to empirical evidence, we briefly review prior work that hinted at such convex organization and show how standard attention mechanisms naturally generate this structure.

\subsection{Theoretical account}

We provided a formal definition of MRH and connected it to the interpolative geometry that could arise under DINO-style self-supervision, which encourage smooth interpolation between prototypes. We now give theoretical support for MRH along three complementary avenues.

\paragraph{Gärdenfors' conceptual spaces}
In \cite{gardenfors2004conceptual}, Peter Gärdenfors models concepts as convex regions within domain-specific coordinate systems; points represent individual objects and prototypes correspond to focal (often extreme) points of these regions. Dimensions capture perceptual or functional qualities (e.g., color, weight, temperature), and category membership varies with proximity -- if two objects belong to a category, then any point on a straight path between them remains within the category’s region. From this perspective, MRH can be viewed as an operationalization of Gärdenfors’ theory for decomposing visual concepts: tiles (or regions) correspond to convex concept sets, while archetypes play the role of prototypes. A natural question follows: is such an instantiation realistic in modern architectures, and what evidence favors regions built from points rather than purely directional features?

\paragraph{Convex-geometric precedents}
Prior work provides evidence that situates MRH within a convex-analytic tradition, which we organize into three strands. First (i), in piecewise-linear networks, ReLU-based architectures partition input space into convex polyhedral linear regions (often unbounded). Theoretical analyses of this partition structure have established upper and lower bounds (and in some cases exact counts) on the number of linear regions and have clarified depth–width trade-offs; in parallel, the spline perspective makes the underlying affine-template structure explicit~\cite{montufar2014number,telgarsky2015representation,serra2018bounding,raghu2017expressive,balestriero2018spline,balestriero2020mad}. Empirically, prior studies have found that trained networks realize far fewer regions than the maximal theoretical counts and that the realized tessellations depend sensitively on optimization, yielding structured yet parsimonious decompositions~\cite{hanin2019complexity,zhang2020empirical}; related interpretability work have  exploited this polyhedral structure by enumerating regions to extract exact piecewise-linear rules~\cite{black2022interpreting,chu2018exact}. Second (ii), in representation space, recent analyses demonstrate convex organization of activations and architecture-specific convex projections~\cite{tvetkova2025convex}. This observation dovetails with results in population geometry indicating that network activity concentrates on low-dimensional manifolds with structured variability~\cite{chung2021neural,cohen2020separability}. 
Relatedly, recents concepts extraction methods such as \emph{k}-Deep Simplex~\cite{tankala2023kdeep}, MFAs~\cite{shafran2026directions} and SpaDE~\cite{hindupur2025projecting} constrain hidden representations to lie in the convex hull of a learned dictionary (i.e., as convex combinations of simplex vertices), thereby reinforcing a prototype centric view of representation geometry. Third (iii), in language models, recent work has shown that categorical and hierarchical concepts admit polytopal encodings whose geometric relations mirror semantic relations~\cite{park2025geometry,park2026information}. 

Our last theoretical argument will consist in demonstrating that standard attention mechanisms naturally generate the geometry in \cref{def:minkowski}. 

\paragraph{From Multi-Head Attention to Minkowski Geometry}
In brief, a single head produces outputs in the convex hull of its projected value vectors (softmax yields barycentric coordinates); affine transformations preserve this convex structure; and multi-head aggregation sums headwise polytopes, yielding a Minkowski sum. The argument proceeds in three steps, with full proofs in \cref{app:hypothesis_theory}, showing how elementary operations compose to create the geometric structure described above.

\begin{lemma}[Single head yields a convex polytope and matches MRH for $|S|=1$]
Let one attention head have queries $\bm{Q}$, keys $\bm{K}$, values $\bm{V}=\{\bm{v}_1,\ldots,\bm{v}_m\}$, attention $\bm{A}=\bm{\sigma}(\bm{Q}\bm{K}^\top)$, and outputs $\bm{Y}=\bm{A}\bm{V}$ with attainable set $\mathcal{Y}$.
Then $\mathcal{Y}\subseteq \conv(\bm{V})$.
Moreover, every output admits an MRH form with $|S|=1$, codes $\bm{z}=\bm{\alpha}$ (rows of $\bm{A}$) and archetypes equal to $\bm{V}$.
If, in addition, the pre-softmax logit map has image $\mathrm{Im}(\bm{K}^\top)+\mathrm{span}\{\mathbf{1}\}=\mathbb{R}^m$ as the query varies, then $\mathcal{Y}=\conv(\bm{V})$.
\end{lemma}
For full details see \cref{app:sec:single_head_polytope}. This simple observation see the attention as generator of convex geometry: each output is a convex combination of value vectors with attention weights serving as the convex codes. In fact, when attention patterns are block-sparse (they select disjoint subsets of tokens across different input regimes) and the values within each block are affinely independent, one could show we obtain disjoint polytopes.%
Crucially, this convex structure is preserved under affine transformations:
\begin{lemma}[Affine transformations preserve MRH structure]
If $\bm{\gamma}(\x) = \bm{W} \x + \bm{b}$ is an affine transformation and $\x = \sum_j z_j \a_j$, then $\bm{\gamma}(\x) = \sum_j z_j \a_{j}'$ 
where $\a_{j}' = \bm{W}\a_j + \bm{b}$ are transformed archetypes and the bias is absorbed into each archetype while codes $\z$ remain unchanged.
\end{lemma}
This \cref{app:sec:mrh_affine_robust} for full details. This ensures that the archetypal structure survives linear projections and affine normalizations commonly found in transformer architectures. The preservation of convex combination coefficients is crucial: regardless of how the archetypal landmarks are transformed, the relative positions within each polytope remain geometrically consistent.
We now show that multi-head attention naturally aggregates these individual polytopes into the Minkowski sum structure describe by MRH.
\begin{proposition}[Multi-head attention realizes MRH]
Let there be $H$ heads with value sets $\bm{V}_h$ and per-head output projections $\bm{W}_O^{(h)}$.
For any input, head $h$ produces weights $\bm{\alpha}_h \in \Delta^{m_h}$ and output
$\bm{y}_h = \sum_{i} \alpha_{h,i} \bm{v}_h^{(i)} \in \conv(\bm{V}_h)$.
After projection and summation,
\[
\bm{y} 
= \sum_{h=1}^H \bm{W}_O^{(h)}\bm{y}_h
= \sum_{h=1}^H \sum_{i}\alpha_{h,i} \bm{W}_O^{(h)} \bm{v}_h^{(i)}
\in \oplus_{h=1}^H \bm{W}_O^{(h)}\bigl(\conv(\bm{V}_h)\bigr).
\]
Thus every output admits an MRH representation with block-convex codes
$\bm{z}=(\bm{\alpha}_1,\ldots,\bm{\alpha}_H)$ and archetypes
$(\bm{W}_O^{(1)}\bm{V}_1,\ldots,\bm{W}_O^{(H)}\bm{V}_H)$.
If, in addition, each head can realize any point of $\mathrm{relint}(\Delta^{m_h})$ up to the softmax additive constant, then the attainable set is exactly the Minkowski sum.
\end{proposition}
See \cref{app:sec:minkowski_mha}. This completes our theoretical account: multi-head attention produces outputs that lie within Minkowski sums of head polytopes, and under standard reachability assumptions, the attainable set equals this sum. The resulting representations admit the block-structured decomposition that MRH describe, with each head contributing a distinct tile of archetypes.

\subsection{Empirical evidence}
We now provide preliminary support for each criterion. The results below should be read as compatible evidence, not proof: multiple mechanisms can mimic the same surface phenomena. Unless stated otherwise we use ImageNet-1k validation tokens from the last DINOv2-B layer; cosine distance and $k$-NN graphs with standard symmetrization.
Figure~\ref{fig:hypothesis_quantitative} (left) contrasts straight-line interpolation between tokens with piecewise-linear geodesics computed as shortest paths on the token $k$-NN graph. Straight lines depart the data support quickly, whereas graph geodesics remain close throughout. This matches what a Minkowski sum structure predicts: feasible displacements arise as sums of small face-walks within head polytopes, which are piecewise linear in barycentric coordinates yet appear curved in the ambient space. The curved, on-support geodesics thus support criterion (\textbf{\textit{i}}) by reflecting head-wise convex reweighting rather than simple linear combinations.
To test the convex coding assumption (\textbf{\textit{ii}}), we compare Archetypal Analysis~\cite{cutler1994archetypal} (AA) with SAE for token reconstruction in \cref{fig:hypothesis_quantitative} (middle). Note that AA is precisely the single-tile case of MRH ($|S| = 1$), making this a direct test of our geometric assumptions. AA imposes dramatically stronger constraints\footnote{Archetypal Analysis seeks $\bm{Z} = \bm{X}\bm{B}$ where archetypes are convex combinations of data (so $\bm{Z} \subset \conv(\bm{X})$), and $\bm{X} \approx \bm{Z}\bm{A}$ where data are convex combinations of archetypes. Both $\bm{A}$ and $\bm{B}$ have simplex constraints (columns sum to 1, non-negative).}: it forces all reconstructions to lie within the convex hull of observed tokens and requires archetypes to be actual data combinations. Despite these restrictive assumptions, AA matches SAE performance with remarkably few archetypes (10 archetypes per image), providing preliminary evidence that even the simplest case of MRH captures fundamental geometric structure. 
This is particularly striking in high dimensions, where the probability that a random point lies in the convex hull of a small set of samples decays exponentially with dimension \cite{barany1988shape,balestriero2021learning}; the fact that tokens can nevertheless be accurately reconstructed from only $\sim$10 archetypes indicates that embeddings concentrate on low-dimensional polytopes embedded within the ambient space.
Additionally, the block structure emerges naturally: \cref{fig:hypothesis_quantitative} (right) reveals that archetypal decompositions spontaneously organize into the block-sparse pattern predicted by criterion (\textbf{\textit{iii}}), with distinct clusters of co-activating archetypes rather than uniform mixing.

\begin{figure}[t]
    \centering
    \includegraphics[width=0.98\linewidth]{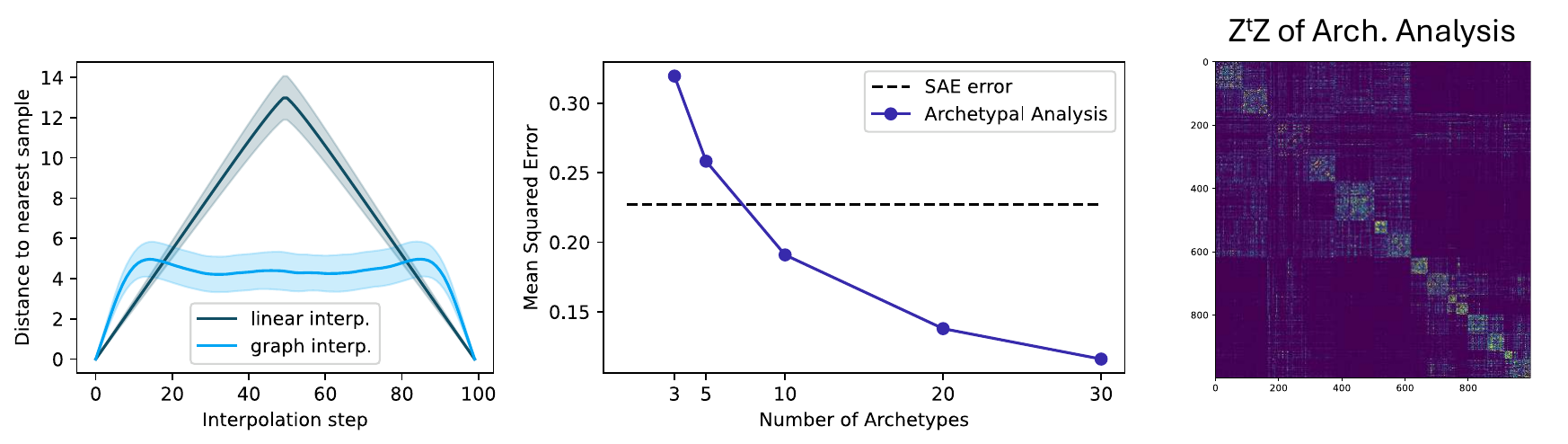}
    \caption{
\textbf{Empirical support for MRH criteria on ImageNet-1k validation set.}
\textbf{(Left)} Distance to data along interpolation paths between tokens. Linear interpolation (dark line) rapidly deviates from valid embeddings, while there exist piecewise-linear paths computed via shortest paths on token $k$-NN graphs (light blue) that remain consistently close to the data manifold. This supports criterion \textbf{(\textit{i})} Minkowski sum structure: feasible connections follow polytope face-walks rather than straight lines through empty space.
\textbf{(Middle)} Archetypal Analysis (single-tile MRH with $|S|=1$) achieves lower error than a sparse autoencoder (dashed line) with as few as 10 archetypes, despite stronger constraints, supporting criterion \textbf{(\textit{ii})} convex coding assumptions.
\textbf{(Right)} Archetypal coefficient matrix $\bm{Z}^T\bm{Z}$ after clustering reveals emergent block structure with bright diagonal clusters. Even without knowledge of tile boundaries, archetypes naturally organize into co-activating groups, supporting the Tiles assumption as it naturally imply block-structured decomposition.
}
    \label{fig:hypothesis_quantitative}
\end{figure}

We emphasize that these observations -- piece-wise linear on-manifold trajectories, efficient archetypal reconstructions, and block-structured co-activation -- represent preliminary evidence rather than definitive proof. Multiple geometric hypotheses could generate similar surface patterns. Nevertheless, the convergent evidence motivates exploring what MRH would mean for interpretability practice.

\subsection{Implications for Interpretability}
If, and this is an assumption, the~\hyp~holds, three immediate implications follow:

\textbf{\textit{(i)} Concepts are no longer Directions, but Points and Regions.} In the context of the Archetypal representation framework, the very idea of a concept diverges from the classical linear factor picture in which semantic evidence is quantified by the magnitude of an inner product with a preferred direction. A concept is better understood as a landmark, an extremal point of the convex polytope that tessellates the latent manifold, or as a small constellation of such landmarks whose hull supports a coherent semantic region. 

\textbf{\textit{(ii)} Steering admits a strict identifiable maximum.}
In a landmark based geometry probing no longer involves sliding indefinitely along a vector but steering an activation toward a specific point on the manifold. Once the activation reaches that landmark (or the convex cell surrounding it) the semantic signal saturates (further movement would drives the embedding off-manifold). This bounded trajectory could explains why current SAE probing gains plateau~\cite{wu2024reft,mueller2025mib,karvonen2025saebench} or even reverse~\cite{hedstrom2024steer,templeton2024scaling} when the scaling coefficient is pushed too far: the useful range ends at the landmark’s basin of attraction. Practical probes should therefore measure how close an activation can get to the landmark set and stop once convergence is achieved, for instance by estimating barycentric weight or geodesic distance inside the convex cell rather than by extrapolating along an unbounded direction.

\textbf{\textit{(iii)} Decomposition is fundamentally non-identifiable.}
Examining the resulting activation space alone don't have a unique solution and thus can't guarantee to recover the original generating factors (the individual polytopes) without strong additional assumptions.
\begin{proposition}[\textbf{Non-identifiability of Minkowski decomposition}]
Let $\mathcal{X} = \oplus_{i=1}^m \mathcal{P}_i$ be a Minkowski sum of convex polytopes. Given only observations from $\mathcal{X}$, the decomposition $\{\mathcal{P}_1,\ldots,\mathcal{P}_m\}$ is generally non-unique: there exist distinct collections $\{\mathcal{Q}_1,\ldots,\mathcal{Q}_k\}$ such that $\mathcal{X} = \oplus_{j=1}^k \mathcal{Q}_j$. In particular, even very simple polytopes admit infinitely many decompositions as sums of line segments (zonotope generators) with varying directions and lengths.
\end{proposition}
See~\cref{app:sec:non_identifiability} and \cite{smilansky1987decomposability} for a more detailed analysis. This limitation follows from the additivity of support functions under Minkowski addition, $h_{\mathcal{X}}(u)=\sum_i h_{\mathcal{P}_i}(u)$: decomposing a single sublinear function $h_{\mathcal{C}}$ into sublinear summands is typically non-unique.
This means that, under the MRH, any method attempting to recover individual concept contributions $\z$ or concept polytopes $\P$ \textit{only from activations from a single layer alone} faces a non-unique solution. 

However, this impossibility, at its core, also hint a possible direction. While the final Minkowski sum hide its constituent factors, the information to find the decomposition is available within previous intermediate representations.
Accessing attention weights and activations from earlier layers may render the factorization tractable, and pointing interpretability efforts toward exploiting architectural structure rather than treating learned representations as opaque geometric objects. 
We believe that developing these structure-aware interpretability techniques could yield promising advances in concept extraction.

\section{Conclusion}
We presented a large-scale concept analysis of \dino, introducing a 32k-unit dictionary that we use as a working baseline.
Our analysis first revealed functional specialization: classification relies on ``Elsewhere'' detectors implementing learned negation; segmentation on boundary concepts forming coherent subspaces; depth estimation on three monocular cue families; and register tokens on global scene factors such as illumination and motion blur.

Examining geometry and statistics of concepts, we found that representations depart from a strictly sparse, quasi-orthogonal regime: atoms are distributed rather than neuron-aligned, anisotropy follows task subspaces, and positional information compresses into 2D while local neighborhoods remain connected even after removing position information. 
These observations collectively point to an organization beyond pure linear sparsity

Synthesizing these observations, we propose the Minkowski Representation Hypothesis (MRH), in which tokens behave as convex mixtures of archetypes and activations form Minkowski sums of convex polytopes, yielding convex concept regions rather than linear directions. We explored theoretical and preliminary empirical evidence for this organisation.

If MRH holds, then this very non-identifiability becomes instructive: it tells us that concept extraction cannot stop at a layer’s activations. To recover meaningful structure, we must explain a layer \textit{through the sequence of transformations that formed it}.

\section{Acknowledgments}
The authors thank Isabel Papadimitriou, Chris Hamblin, Sumedh Hindupur and Alessandra Brondetta for many fruitful discussions. The authors would also like to thanks David Bau for his help in conceiving the name of the Elsewhere concept. This work has been made possible in part by a gift from the Chan Zuckerberg Initiative Foundation to establish the Kempner Institute for the Study of Natural and Artificial Intelligence at Harvard University. MW acknowledges support from a Superalignment Fast Grant from OpenAI, Effective Ventures Foundation, Effektiv Spenden Schweiz, and the Open Philanthropy Project.

\newpage
\bibliography{main}
\bibliographystyle{iclr2026_conference}
\clearpage

\appendix
\onecolumn

{\hypersetup{colorlinks=true, allcolors=black}
\appendix
\onecolumn
\section{Appendix Index}
\startcontents
\printcontents{}{1}{\setcounter{tocdepth}{2}}
\clearpage
}

\section{From LRH to Dictionary Learning}
\label{app:dico}

The LRH induces a natural inverse problem: recover the latent basis on which activations are sparsely expressed. Concretely, given activations \(\A\in\R^{n\times d}\) (we adopt the row–atom convention \(\D\in\R^{c\times d}\) used in the main text so that \(\A\approx \Z\D\) with \(\Z\in\R^{n\times c}\)), concept extraction becomes a dictionary–learning problem with method-specific constraints on \(\Z\) and \(\D\):

\begin{equation}
\begin{aligned}
\nonumber
    & (\Z^\star, \D^\star) = \argmin_{\Z, \D} || \A - \Z \D ||^2_F, \\
    \text{s.t.} ~~ &\begin{cases}
        \forall i, \Z_i \in \{ \e_1, \ldots, \e_k \}, & \text{(\textbf{ACE} - K-Means)}, \\
        \D \D^\tr = \mathbf{I}, & \text{(\textbf{ICE} - PCA)}, \\
        \Z \geq 0, \D \geq 0, & \text{(\textbf{CRAFT} - NMF)}, \\
        \Z_i \in \Delta^{c}, \D_i \in \conv(\A) & \text{(\textbf{AA} - Archetypal Analysis)}, \\
        \Z = \bm{\Psi}_{\theta}(\A), ||\Z||_0 \leq K, & \text{(\textbf{SAE}s)}.
    \end{cases}
\end{aligned}
\label{eq:dico_constraints}
\end{equation}

Here \(\mathbf{I}\) is the identity, \(\mathbf{e}_j\) denotes a canonical basis vector, and \(\Psi_\theta\) is an encoder and a sparsity projection (e.g., TopK, Jump-ReLU, or simply ReLU) producing sparse codes. This formulation unifies the previous clustering-based concept extraction~\cite{ghorbani2019towards}, orthogonal factorization (PCA/ICE)~\cite{zhang2021invertible}, Nonnegative concept extraction (CRAFT/NMF)~\cite{fel2023craft}, Archetypal Analysis~\cite{cutler1994archetypal,thurau2009archetypal}
and modern SAEs~\cite{cunningham2023sparse,bricken2023monosemanticity}. In practice, these approaches trade off \emph{fidelity} (\(\|\A-\Z\D\|_F\)) against \emph{sparsity} (e.g., \(\|\Z\|_0\)), yielding a Pareto frontier; SAEs are attractive at scale because the encoder \(\Psi_\theta\) enables amortized, batched inference while retaining LRH’s sparse, overcomplete structure.

\section{Task Specific Concept}

In this section, we review additional results and observation on the differents task-specific concepts discussed in \cref{sec:tasks}. We will start by giving details on the theoretical root of the importance measure, then we will briefly expand on the ``Elsewhere'' concept before delving into the monocular depth estimation.

\subsection{Importance Measure for Concept-Task Alignment}
\label{app:importance}

In the \cref{sec:tasks}, we ask which concepts in the dictionary are actually recruited by downstream tasks. We describe here precisely the importance measure we used, which has some appealing property as the linear probe allow us to directly interpret the importance for any linear probe as a linear combination of concepts.

Let $\A\in\mathbb{R}^{nt\times d}$ denote token activations (over $n$ images and $t$ tokens per image). We factor $\A \approx \Z\D$ with codes $\Z \in \mathbb{R}^{nt\times c}$ and dictionary $\D \in \mathbb{R}^{c \times d}$. For a linear probe with weights connecting to $o$ classes $\W \in\mathbb{R}^{o\times d}$ and predictions $\bm{Y}=\A\W^\top\in\mathbb{R}^{nt\times o}$, substituting the factorization gives
\[
\bm{Y}=(\Z\D)\W^\top = \Z\,\underbrace{(\D\W^\top)}_{\W'\in\mathbb{R}^{c\times o}}\,.
\]
The matrix $\W'$ encodes the alignment between dictionary concepts (rows) and task outputs (columns). We define the concept-importance vector for the probe as the expected concept activation weighted by this alignment:
\[
\bm{\phi} = \mathbb{E}(\Z)\,\W'  \in  \mathbb{R}^{o}\,,
\]
where the expectation is taken over tokens (and samples, but not over the concepts) in the evaluation set. Class-wise scores correspond to the components of $\bm{\phi}$; concept-wise scores can be read from $\mathbb{E}(\Z)$ together with the corresponding rows of $\W'$. 

In linear regimes, this coincides with canonical attribution functionals when expressed in the concept basis. Specifically gradient$\times$input \cite{shrikumar2017learning,simonyan2013deep}, Integrated Gradients with zero baseline \cite{sundararajan2017axiomatic,ancona2017better}, Occlusion \cite{zeiler2014visualizing}, and RISE \cite{petsiuk2018rise} reduce to linear functionals that are proportional to $\mathbb{E}(\Z)\W'_{:,j}$ when aggregated across tokens. Under standard faithfulness criteria such as C-Deletion, C-Insertion \cite{petsiuk2018rise}, and C-$\mu$Fidelity~\cite{yeh2019infidelity}, this is proven to be the optimal attribution when concept are linearly linked to class score; see Theorem~3 in \cite{fel2023holistic} (and \cite{ancona2017better} for the initial discussion). We thus use this formulation as a principled and canonical measure of concept importance for linear readouts.

\subsection{On the ``Elsewhere'' Concepts}
\label{app:elsewhere}

We have observed and discussed in~\cref{sec:tasks} a consistent and intriguing pattern: across a wide range of ImageNet classes, the top few most important concepts typically include not only interpretable objects or object-parts, but also an intriguing ``\textit{Elsewhere}'' concept. As detailed in Figure~\ref{fig:hollow}, these concepts activate broadly across the tokens, but crucially \emph{not} on the object itself. Their firing is suppressed exactly where the object appears, and prominent in surrounding regions or background areas. Importantly, ``Elsewhere'' concepts are \emph{not} generic background detectors: their firing depends critically on the object's presence, and they vanish entirely if the object is removed from the image.

This phenomenon reveals that certain concepts do not fire where the relevant information is located, but instead are able to extract information from one region to fire in spatially distant locations. The Elsewhere concepts represent an extreme example of this spatial decoupling, where the relationship between what drives the concept's firing and where it actually fires exhibits a complete spatial inversion. Rather than simply detecting local features, these concepts implement a form of distributed spatial reasoning that can be characterized as implementing the logical relation ``\emph{not the object, but the object exists}''. This suggests that \dino~has implicitly learned a sophisticated form of fuzzy spatial logic, systematically distributing class-relevant evidence across both object-centric and contextually-related off-object tokens.

The utility of this distributed representation becomes particularly evident when considering the architectural constraints of vision transformers. Since \dino's final classifier operates on the spatial average of patch embeddings (concatenated with \texttt{cls} token), having class-relevant information distributed across tokens that do not directly contain the object provides several computational advantages. This strategy enhances robustness to partial occlusion, as class evidence remains available in unoccluded regions even when the primary object is hidden. It also provides invariance to viewpoint changes and spatial transformations, since class information is not concentrated solely at object locations. Furthermore, this distributed approach allows the model to integrate multi-scale contextual information that may be crucial for disambiguation in challenging visual scenarios.

To causally verify these interpretations, we employed RISE~\cite{petsiuk2018rise} analysis, applying random perturbations that mask portions of the image while measuring the resulting changes in concept activation values. The causal attribution maps (bottom row of Figure~\ref{fig:hollow}) demonstrate that despite firing in off-object locations, these Elsewhere concepts are causally dependent on the object itself. The RISE attribution is computed as:
$$
\Gamma_{\text{RISE}}^{(i)}(\bm{f}, \bm{x}) = \mathbb{E}_{\bm{m} \sim \mathbb{P}_{\bm{m}}}(\bm{f}(\bm{x} \odot \bm{m}) | \bm{m}_i = 1)
$$
where $\bm{f}$ represents the composition of the model and the SAE, using 8000 forward passes for each explanation. The results consistently show that the object or animal itself is the most causally important region responsible for concept firing, even though the concept manifests its activation in spatially disjoint locations.

This finding challenges the implicit assumption underlying most heatmap-based concept visualization approaches: that \textbf{a concept is primarily \textit{about} the spatial tokens where the concept fires most strongly}. The Elsewhere phenomenon demonstrates a clear dissociation between activation localization and causal attribution, revealing that the most informative regions for understanding a concept's behavior may be spatially distinct from where the concept exhibits its strongest activations. This spatial decoupling would have implications for interpretability research and practice, as it warns against the common tendency to overtrust activation-based visualizations as direct indicators of what information a concept requires or processes.

The prevalence of Elsewhere concepts across diverse object categories indicates that this distributed spatial reasoning is not an artifact of specific classes but represents a fundamental computational strategy employed by vision transformers. This pattern suggests that these models naturally evolve sophisticated spatial logic capabilities that go beyond simple local feature detection, instead developing context-dependent activation patterns that are dynamically modulated by global image content. The discovery of this phenomenon highlights the need for interpretability tools that explicitly account for the potential disconnect between concept activation locations and their causal dependencies, incorporating both spatial activation analysis and causal perturbation methods~\cite{shaham2024multimodal} to provide accurate and complete characterizations of learned representations.

\subsection{Monocular Depth Estimation}
\label{app:depth}

Still in \cref{sec:tasks}, We have seen that despite being trained without explicit 3D supervision, DINO exhibits surprising aptitude for depth-related tasks. We have conducted targeted perturbation analysis of depth-relevant concepts using five controlled image manipulations: local median blurring to suppress shadows, global blur to remove fine details, high-pass filtering to emphasize geometric patterns, soft edge enhancement to retain contours, and geometric warping to distort perspective cues. We measure concept activation profiles across these perturbations and project results onto a UMAP embedding showcased in ~\cref{fig:depth}.

This systematic analysis reveals three coherent clusters with distinct sensitivity profiles. The local frequency transition cluster responds to blur and median filtering, capturing spatial detail and texture gradients. The projective geometry cluster shows sensitivity to warping and high-pass filtering, detecting perspective lines and structural convergence. The shadow-based cluster exhibits primary sensitivity to median filtering, responding to lighting gradients and cast shadows. \textbf{Many concepts exhibit mixed sensitivity profiles}, suggesting DINO learns composite depth representations integrating multiple visual channels. This taxonomy aligns with classical monocular depth cues from visual neuroscience, demonstrating that interpretable 3D perception primitives emerge through self-supervised learning. Full perturbation profiles are available in Figure \ref{app:fig:depth_full}.

\begin{figure}[t]
    \centering
    \includegraphics[width=0.98\linewidth]{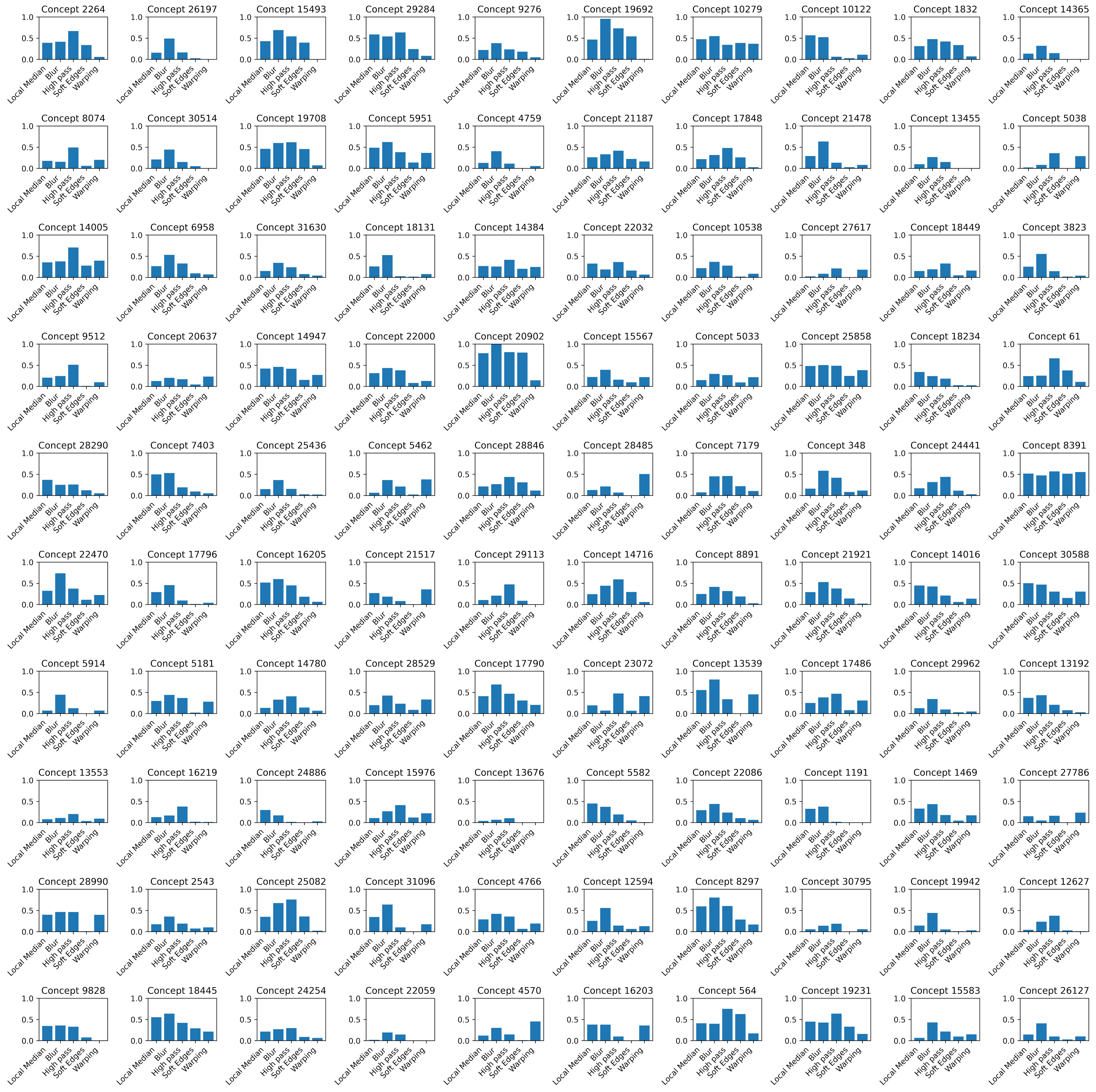}
    \caption{Complete perturbation analysis for depth-relevant concepts showing sensitivity profiles across five image manipulations. The systematic clustering into three main families demonstrates organized depth representation structure within the concept dictionary.}
    \label{app:fig:depth_full}
\end{figure}

\clearpage
\section{Tasks Concept form low-dimensional subspaces}
\label{app:downstream}

To better understand how different tasks recruit concepts from our learned dictionary, we examine the geometric organization of task-specific concept subsets. While ~\cref{sec:tasks} demonstrated that different tasks draw from distinct regions of the concept space, here we investigate whether these functional specializations exhibit coherent geometric structure.

We extracted the top-500 concepts most aligned with each task (classification, segmentation, depth estimation) based on their importance scores, as well as a random subset of 500 concepts as a control. For each subset, we computed 2D PCA projections to visualize their geometric arrangement within the concept space.

The results in~\cref{app:fig:functional_pca} reveal differences in geometric organization across tasks. Classification concepts are broadly scattered across the projection space, consistent with the diverse range of features needed for multi-class recognition. In contrast, segmentation concepts trace a coherent low-dimensional arc. Depth estimation concepts exhibit clear bimodal structure, reminiscent of the bi-modal organization observed in Figure~\ref{fig:downstream_intratask}, potentially reflecting the distinct families of monocular depth cues we identified (projective geometry, shadows, and frequency transitions). The random concept subset shows no discernible structure, confirming that the observed patterns reflect genuine functional organization rather than artifacts of the projection method.

These geometric signatures support the hypothesis that task-specific concept recruitment follows principled patterns: each task draws from geometrically distinct subregions of the concept space, with the local geometry reflecting the underlying computational requirements. This functional-geometric correspondence suggests that the concept dictionary exhibits hierarchical organization, where related computational primitives cluster together in representational space.

\begin{figure}[ht]
    \centering
    \includegraphics[width=0.9\linewidth]{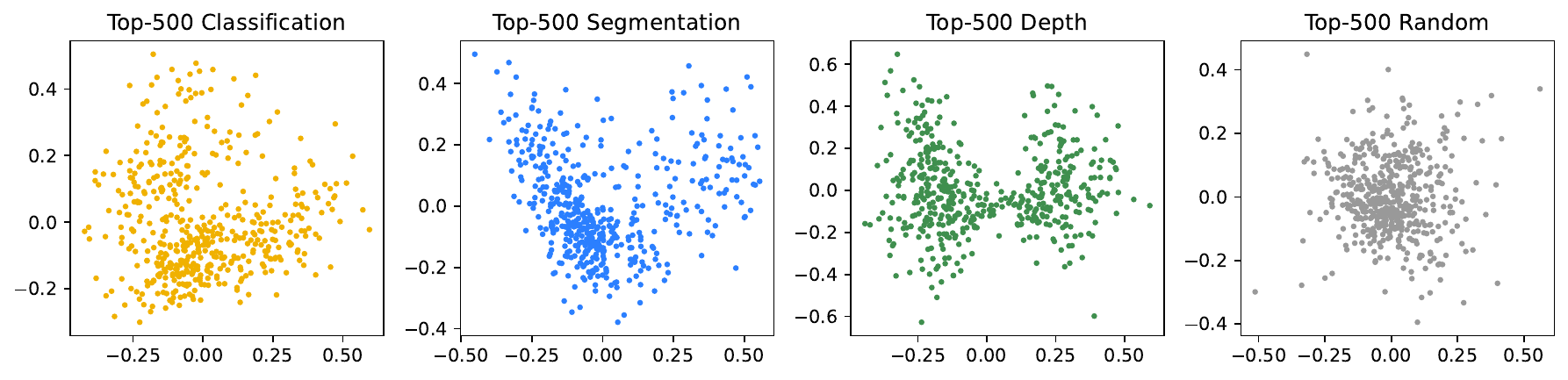}
    \caption{2D projections (PCA) of top-500 concepts most aligned with each task. Classification concepts (left) are broadly scattered; segmentation concepts (middle-left) trace a low-dimensional arc; depth concepts (middle-right) show clear substructure (bimodality), reminiscent of the bimodal structure in Fig.\ref{fig:downstream_intratask} (middle); random concept subset (right) shows no structure. This supports the hypothesis that each task draws from a geometrically distinct subregion of concept space.
}
    \label{app:fig:functional_pca}
\end{figure}

\section{Baseline Models for Co-Activation Spectrum Analysis}
\label{app:baseline_ztz}

To contextualize the spectral properties of $\Z^\tr\Z$ studied in ~\cref{sec:statistics_and_geometry_of_concepts}, we construct two baseline models that preserve structural characteristics while eliminating semantic organization.
\paragraph{Random Baseline.} We generate a random symmetric matrix $\bm{R}$ with identical sparsity $\rho$ and total mass:
\begin{align}
R_{ij} &= U_{ij} \cdot \bm{1}(V_{ij} < \rho) \quad \text{where } U_{ij}, V_{ij} \sim \mathcal{U}(0,1) \\
\tilde{\bm{R}} &= \frac{(\bm{R} + \bm{R}^\tr)/2}{\|\bm{R} + \bm{R}^\top\|_F} \cdot \|\bm{Z}^\tr\bm{Z}\|_F
\end{align}
with $\text{diag}(\tilde{\bm{R}}) = \text{diag}(\bm{Z}^\tr\bm{Z})$. This tests whether the observed spectral structure results from random co-occurrences expected with identical activation sparsity.

\paragraph{Shuffled Baseline.} We preserve the exact empirical distribution of co-activation strengths while destroying semantic organization through random permutation. Let $\bm{G} = \bm{Z}^\top\bm{Z}$ be the original co-activation matrix. We construct the shuffled baseline as follows:

\begin{enumerate}
\item Extract all upper triangular entries: $\mathcal{U} = \{G_{ij} : i < j\}$
\item Apply a random permutation $\pi$ to obtain shuffled values: $\mathcal{U}' = \pi(\mathcal{U})$
\item Construct matrix $\bm{S}$ where:
\begin{align}
S_{ij} = \begin{cases}
G_{ii} & \text{if } i = j \text{ (preserve diagonal)} \\
u'_k & \text{if } i < j \text{ (where } u'_k \text{ is the } k\text{-th element of } \mathcal{U}') \\
S_{ji} & \text{if } i > j \text{ (copy from upper triangle to preserve symmetry)}
\end{cases}
\end{align}
\end{enumerate}

This procedure preserves: (i) all diagonal entries (self-activations), (ii) the empirical distribution of off-diagonal values, and (iii) matrix symmetry. However, it destroys the specific concept pairs and potential block structure while maintaining the same marginal statistics as the original co-activation matrix.
The shuffled baseline is particularly diagnostic: substantial deviation from 
$\text{eig}(\bm{S})$
indicates that the specific pattern of concept co-activation (not merely the distribution of co-activation magnitudes) carries semantic information.

\section{Baseline Models for Concept Geometry Analysis}
\label{app:baseline_conceptD}
To contextualize the geometry of dictionary atoms $\D$, we also build a few baselines. 
\paragraph{Random vectors on sphere} 
We sample a Gaussian random matrix $H\in\R^{c\times d}$, where $H_{ij}\sim\mathcal N(0,1)$ i.i.d., and then we normalized each row to have L2 norm 1. 
This is reminiscent of the concept vectors when randomly initialized before training. 
In high dimension, these vectors are usually relatively isotropic on the unit sphere. 

\paragraph{Grassmannian frame} 
Next, to stress test the LRH, we numerically computed the Grassmannian from of $c$ atoms in $d$ dimension. 
This is a non-trivial computational problem, where analytical solution is only available for few scenarios, otherwise we need to rely on iterative optimization to find approximate Grassmannian frames. 
Brute force gradient optimization of incoherence is slow to converge at our scale.  
Here, we adapted the algorithm from recent work TAAP \cite{massion2025grassmannianTAAP}, to further accelerate the solver, we adapted it to CUDA, made Grassmannian frame solving feasible at our problem scale $c=32000$, $d=768$. 
To note, even with GPU acceleration, solving the frame once still takes 6hr on an A100 GPU. In the end, we reached maximal coherence of $0.065897$. Indeed, more isotropic and less coherence than the Gaussian random vectors. 

\clearpage
\section{Token-Type-Specific Concept}

We complement here to work showcased in~\cref{sec:tokenspecific}, specifically on the Token-specific concepts. Motivated by the fact that token there are different token types in classical ViT, we have studied the different concepts appearing on each concepts. We give details on our entropy measure and how we find token-specific concepts.

We start with the footprint of each concept: the distribution of its activations across token positions. For every concept, we compute the entropy of its token-wise activation over 1.4 million images. Concepts with low footprint entropy are highly localized -- activating consistently on specific token subsets -- whereas high-entropy concepts are spatially diffuse and positionally agnostic.

\begin{figure}[ht]
    \centering
    \includegraphics[width=0.9\linewidth]{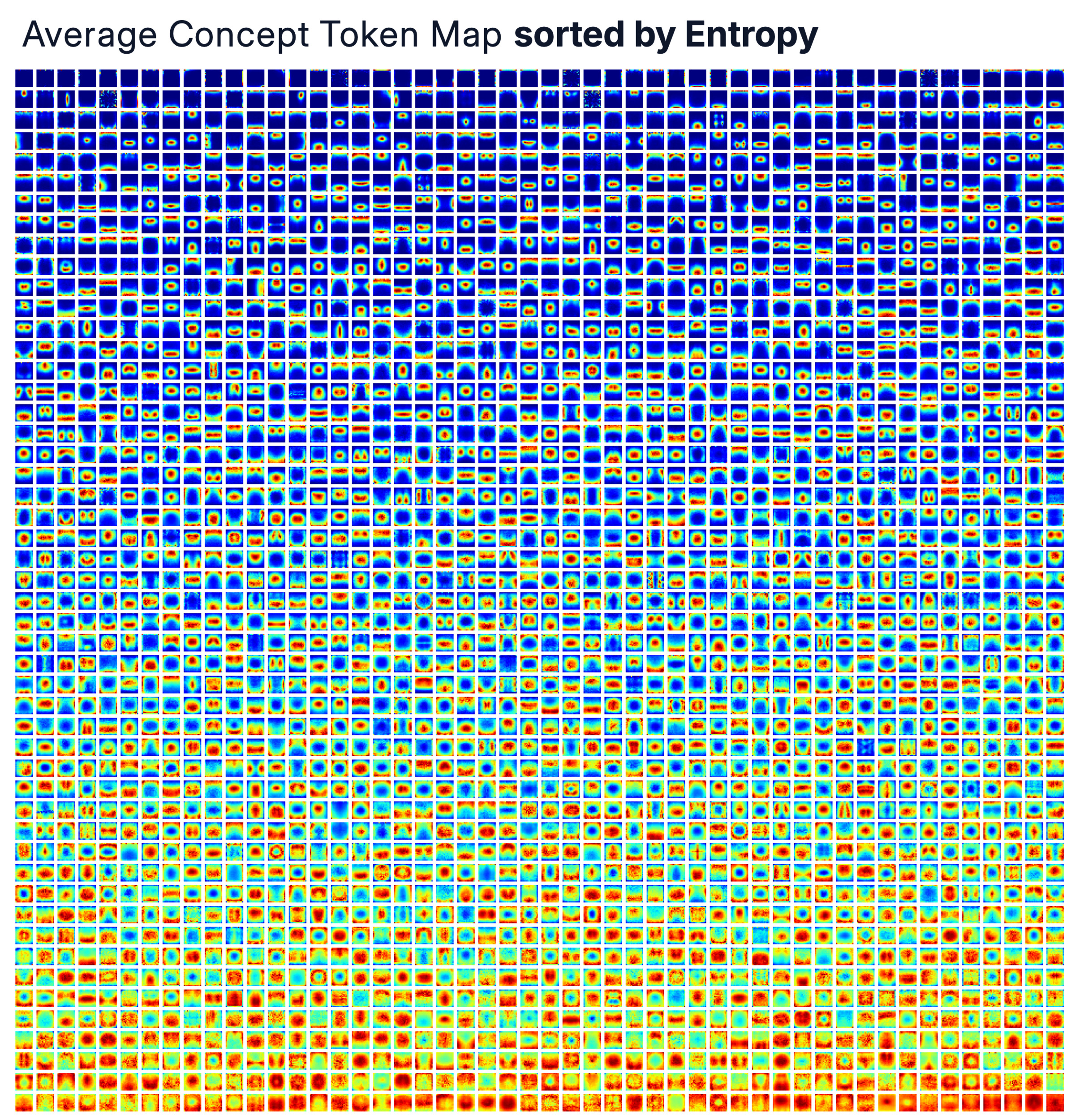}
    \caption{\textbf{Spatial footprints of 2,500 concepts sorted by entropy}. Each row shows the average activation pattern of one concept across token positions, arranged from most localized (low entropy, top) to most uniform (high entropy, bottom). Localized concepts exhibit strong positional preferences, firing predominantly in specific regions or on particular token types (\texttt{cls}, register, or spatial positions), while uniform concepts show broad, distributed activation patterns across the entire image.}
    \label{app:fig:footprints_maps}
\end{figure}

Formally, to understand the spatial organization of concept activations, we propose the notion of ``footprint'' of each concept: its distribution of activations across token positions within images. 
With $\Z \in \R^{nt \times c}$ that we reshape to the tensor $\Z \in \R^{n \times t \times c}$
, then for the concept $i$ we compute the footprint $\bm{\omega} \in \R^t$ as:
\[
\bm{\omega}_{i} = \frac{1}{N} \sum_{n=1}^{N} \Z_{n,:,i}
\]
where $\Z_{n,:,i}$ is the concept map (a vector of $261$ scalars for DINOV2-b) of the concept $i$ on input $n$, and $N$ is the total number of inputs. This captures the average activation strength of each concept at each spatial location, revealing whether concepts exhibit positional preferences or fire uniformly across the image.

We characterize each concept's spatial specificity using the entropy of its empirical footprint distribution. Low entropy indicates highly localized concepts (e.g., firing only at specific positions or token types), while high entropy suggests spatially uniform activation patterns.

The analysis reveals a spectrum of spatial behaviors. Most concepts exhibit relatively uniform firing patterns across spatial tokens, but a significant subset shows strong positional biases. These include concepts that fire exclusively on register tokens (capturing global scene properties like illumination), position-specific concepts that consistently activate in particular spatial regions (potentially encoding geometric or compositional biases), and a single concept that fires exclusively on the \texttt{cls} token (likely encoding its positional embedding). This spatial specialization provides further evidence for the functional organization of the concept dictionary, with different concept families optimized for different computational roles within the architecture.

\begin{figure}[ht]
    \centering
    \includegraphics[width=0.9\linewidth]{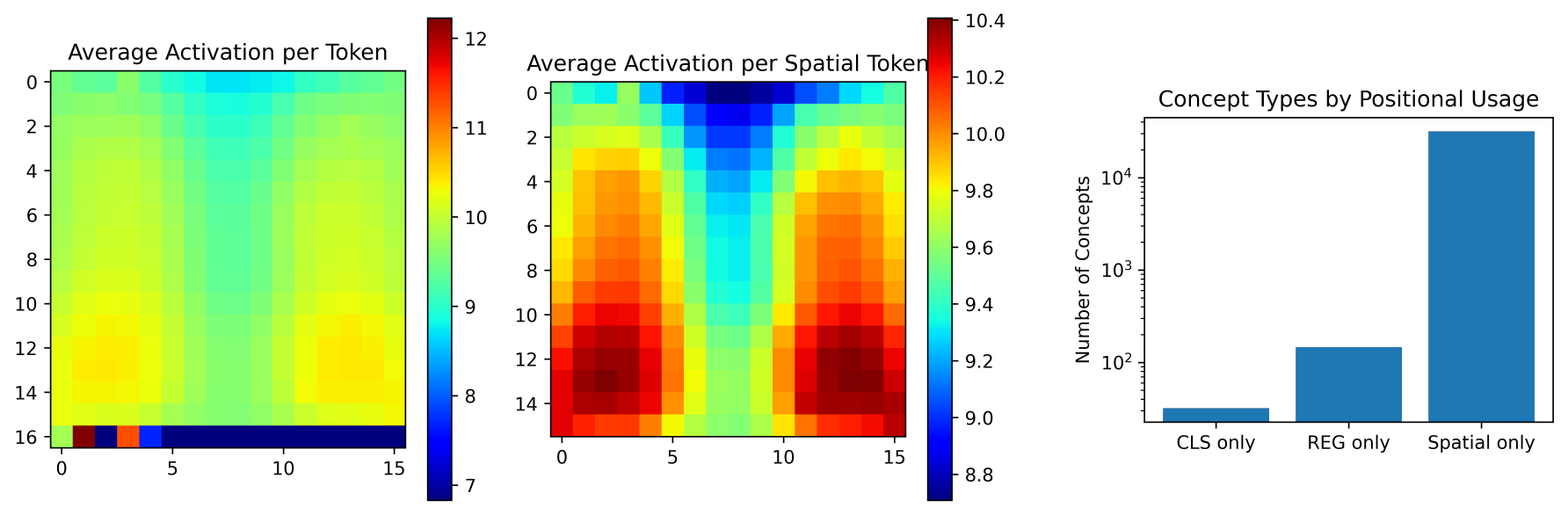}
    \caption{\textbf{Statistical analysis of concept footprints}. (Left) Average activation intensity per token position, showing elevated activity on register and \texttt{cls} tokens compared to spatial tokens. (Middle) Same analysis restricted to spatial tokens only, revealing subtle positional biases within the image grid. (Right) Distribution of concepts by token-type exclusivity, confirming the findings from Figure~\ref{fig:footprint}: one concept fires exclusively on \texttt{cls}, hundreds specialize for register tokens, and many are restricted to spatial positions, indicating substantial functional specialization beyond uniform activation patterns.}
    \label{app:fig:footprints_stats}
\end{figure}

\clearpage
\section{Qualitative Visualization of Local Geometry}
\label{app:local_geometry}

\begin{figure}[ht]
    \centering
    \includegraphics[width=\textwidth]{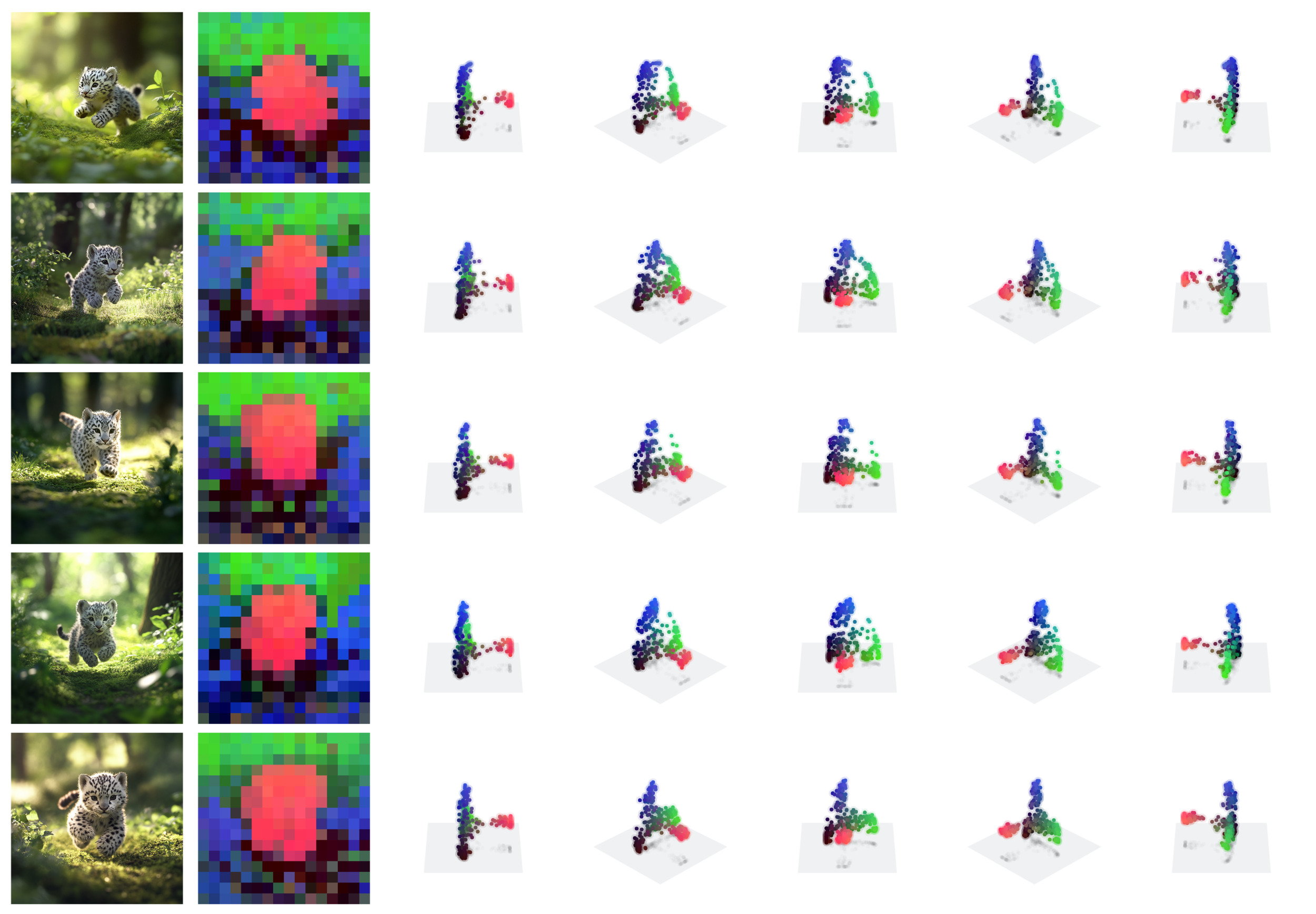}
    \caption{%
        \textbf{Detailed structure of token geometry in snow leopard images.}
        Here we project all patch tokens from five different snow leopard images into the PCA basis trained on all the \dino~token pooled across the five images. Each token is visualized as a point in 3D PCA space, and colored in RGB according to its coordinate along the principal components. This shared projection reveals how similarly-positioned objects (here, the leopards) align across images within the same geometric manifold. Despite slight variations in pose and lighting, the representations remain coherent and consistent across instances. This provides evidence for a global manifold structure, within which local image tokenizations trace smooth trajectories.
    }
    \label{fig:global_pca_snowleopard}
\end{figure}

We provide further visualizations to illustrate the structure of \dino~representations at the per-image level. In Figures~\ref{fig:local_geometry_pca1} to \ref{fig:local_geometry_pca5}, patch tokens are projected into their top PCA components, computed independently for each image. The resulting RGB visualizations highlight smooth embedding transitions that often align with object boundaries.

To complement these local views, Figure~\ref{fig:global_pca_snowleopard} shows a controlled example across five similar inputs (snow leopards), using a shared PCA basis trained on the entire token corpus (of those 5 images). While individual token positions vary slightly due to pose or lighting, the embeddings align within a common geometric frame. This suggests that \dino~not only builds smooth manifolds locally, but does so in a globally consistent latent space.

\begin{figure}[ht]
    \centering
    \includegraphics[width=\textwidth]{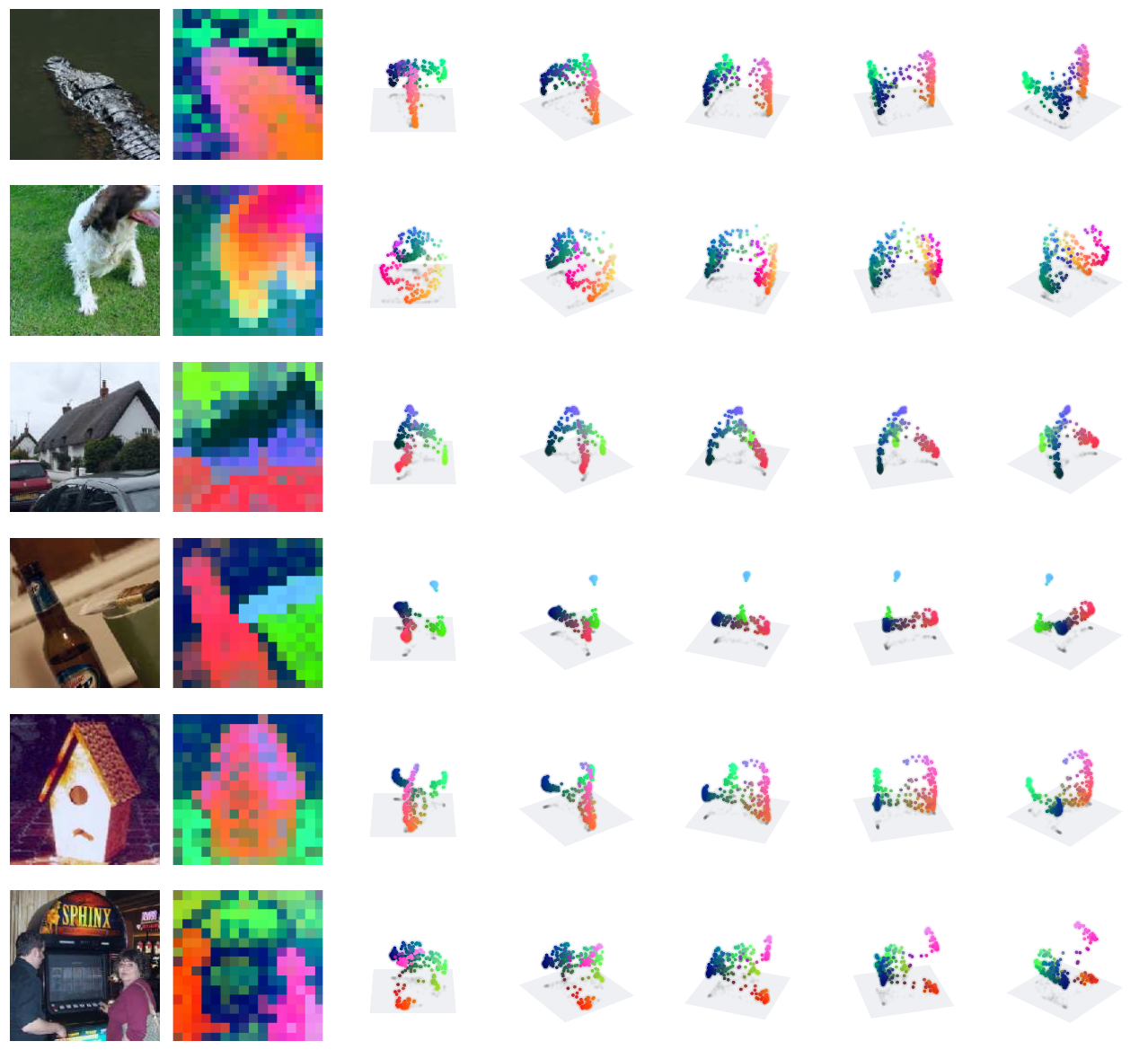}
    \caption[]{\textbf{(cont.)} More examples of PCA-colored patch token embeddings as in Fig.~\ref{fig:local_geometry_pca1}.}
    \label{fig:local_geometry_pca3}
\end{figure}

\begin{figure}[ht]
    \centering
    \includegraphics[width=\textwidth]{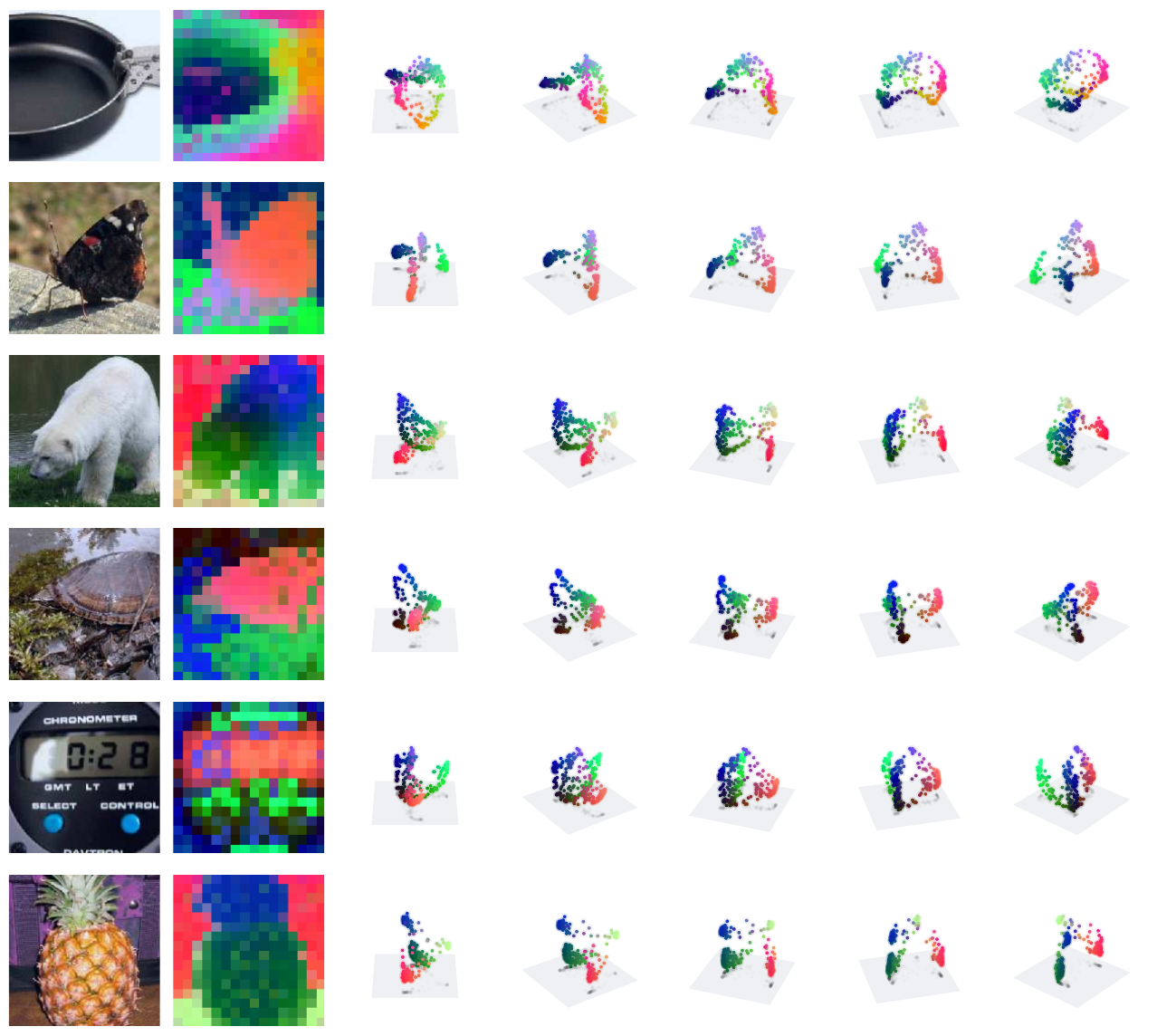}
    \caption[]{\textbf{(cont.)} More examples of PCA-colored patch token embeddings as in Fig.~\ref{fig:local_geometry_pca1}.}
    \label{fig:local_geometry_pca4}
\end{figure}

\begin{figure}[ht]
    \centering
    \includegraphics[width=\textwidth]{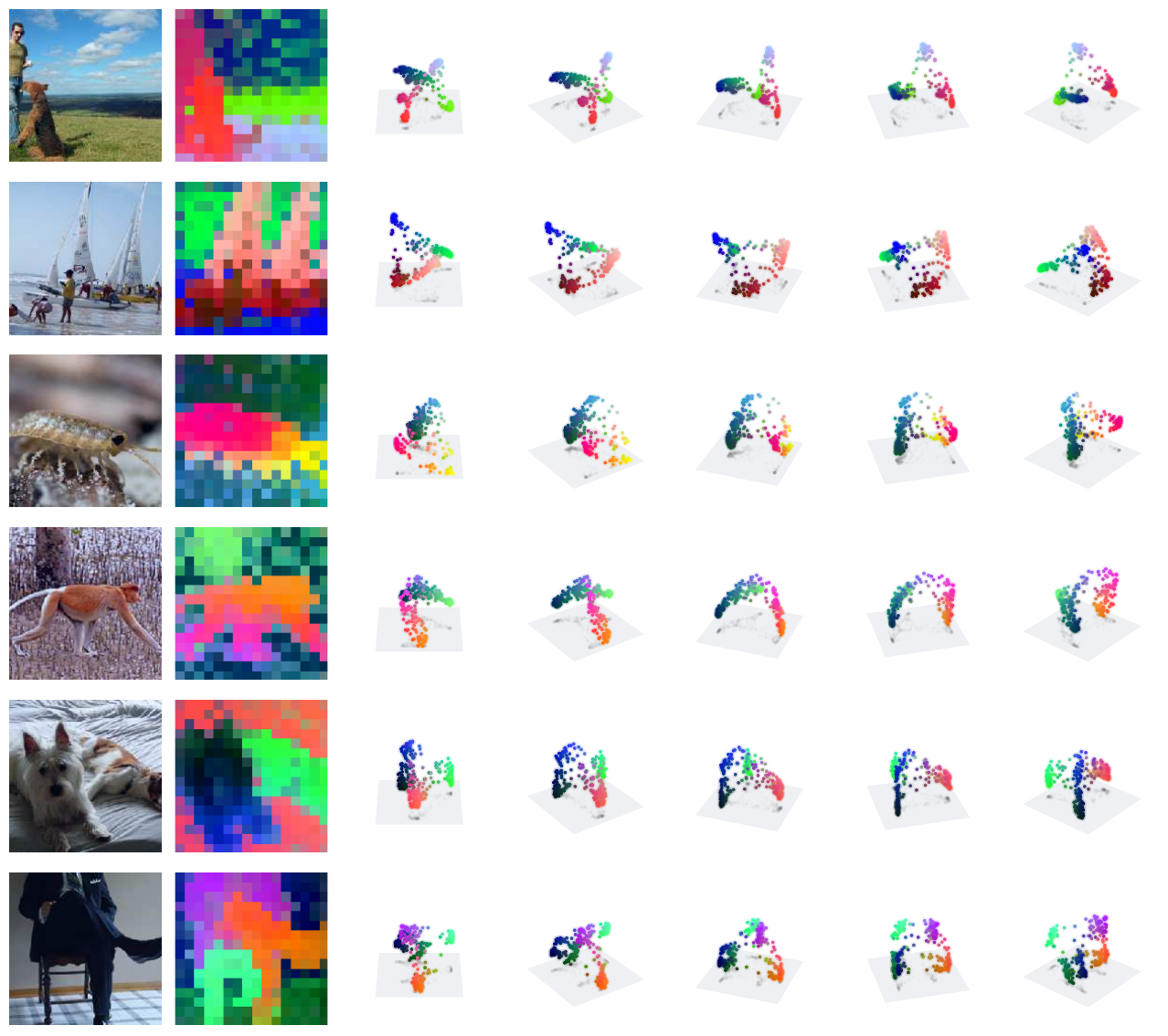}
    \caption[]{\textbf{(cont.)} More examples of PCA-colored patch token embeddings as in Fig.~\ref{fig:local_geometry_pca1}.}
    \label{fig:local_geometry_pca5}
\end{figure}

\begin{figure}[ht]
    \centering
    \includegraphics[width=\textwidth]{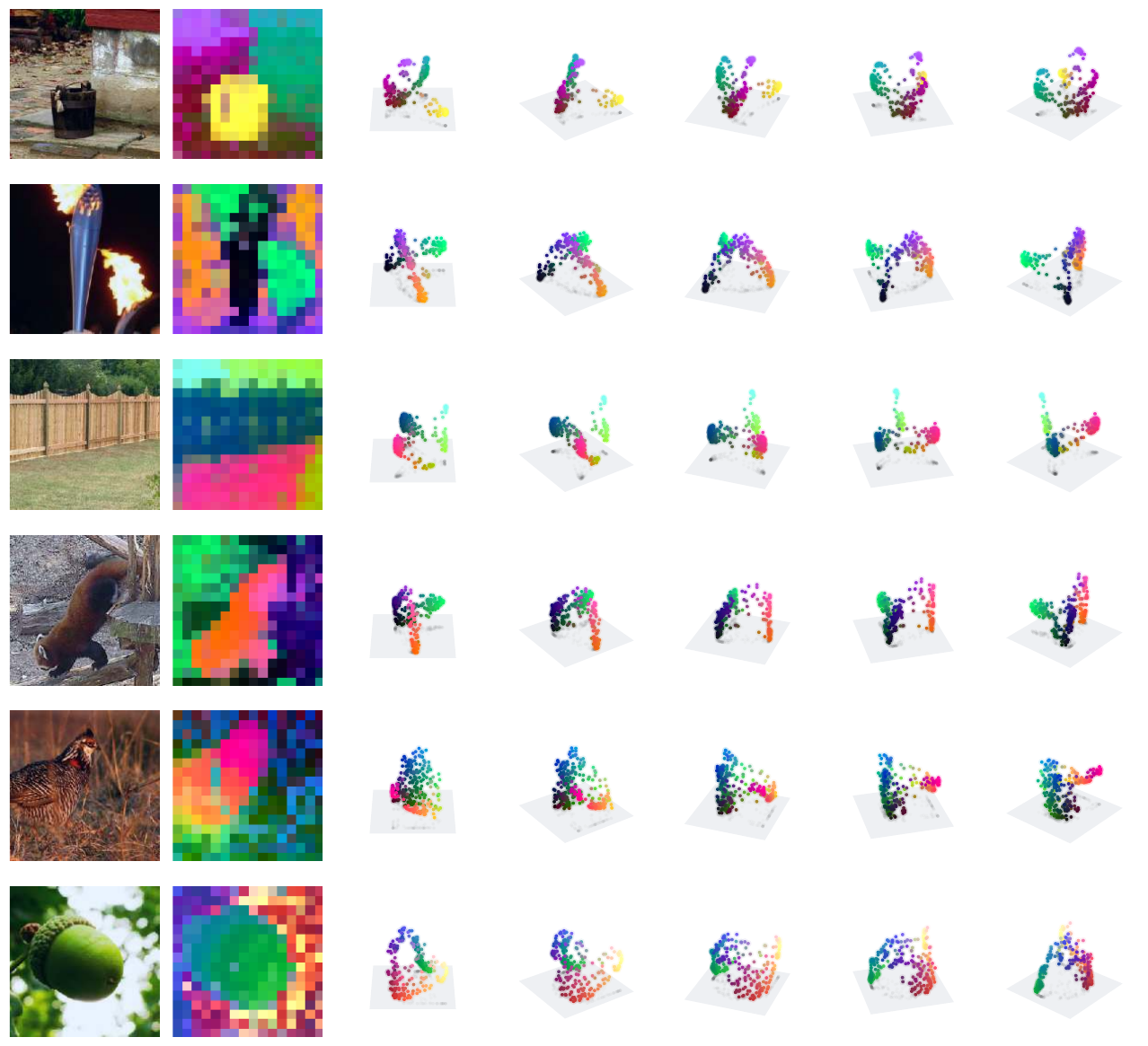}
    \caption[]{\textbf{(cont.)} Final examples of PCA-colored patch token embeddings as in Fig.~\ref{fig:local_geometry_pca1}..}
    \label{fig:local_geometry_pca6}
\end{figure}

\section{Position Embedding Analysis}
\label{app:position}

To investigate the role of positional information in the observed smooth token geometry, we conduct two experiments: (i) extracting and characterizing the positional basis across layers, and (ii) verifying that smoothness persists even after removing positional information.

\subsection{Positional Basis Extraction.} 

For each layer, we extract positional embeddings from 1 million ImageNet images, yielding $\A \in \R^{N \times t \times d}$ token representations where $n$ is the number of images, $t = 261$ tokens per image, and $d = 768$ is the embedding dimension. We employ two approaches to recover the positional basis:
\textbf{Direct averaging.} We compute the average embedding for each spatial position across all images: $\bm{p}_{i} = \frac{1}{N} \sum_{k=1}^N \A_{n,i}$ where $\A_{n,i}$ is the embedding of token at position $i$ in image $n$. We repeat this procedure for each layer.
\textbf{Linear classification.} We train a linear classifier to predict token position from embeddings, yielding weight vectors $\bm{w}_{i}$ for each position at each layer.

Both methods produce highly consistent results: the stable rank profile is similar and the accuracy yield the same results. We therefore choose to use the classifier weights as our primary positional basis for the rest of the experiment.

The analysis in \cref{fig:pos_basis} reveals that positional information undergoes systematic compression across layers. Early layers maintain high-rank positional representations that allow precise spatial localization, but this progressively collapses to a low-dimensional (approximately 2D) subspace in the final layers, consistent with a transition from place-cell-like to coordinate-based encoding.

\subsection{Structure Persists After Position Removal}

Still in \cref{fig:pos_basis}, we showed that the position basis is not responsible for the main part of the structure observed in the PCA visualization. To test qualitatively this effect, we project token embeddings orthogonal to the positional subspace, completely removing positional information by projecting the token on the orthogonal subspace of the classifier.

Remarkably, PCA visualizations (in \cref{app:fig:pca_no_pos}) of the original image of \cref{fig:global_pca_snowleopard} embeddings continue to exhibit the same structure, with smooth patterns that align with object boundaries and semantic regions. This demonstrates that the interpolative geometry we observe reflects genuine semantic organization rather than artifacts of positional encoding. The structure emerges from the model's representation of visual content itself, not from spatial coordinate information.

\begin{figure}
    \centering
    \includegraphics[width=0.9\textwidth]{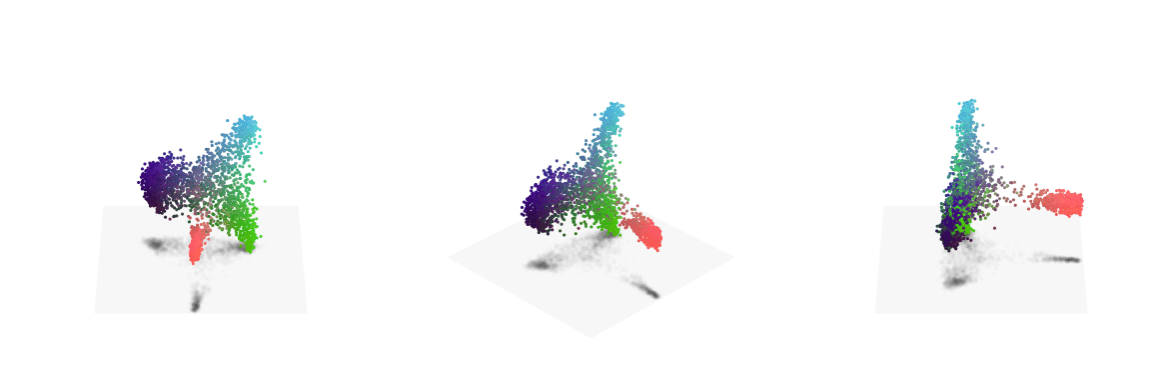}
    \caption{\textbf{Smooth token geometry persists after removing positional information}. PCA visualization of patch token embeddings from a rabbit image after projecting orthogonal to the positional basis to completely remove spatial coordinate information. Despite the absence of any positional cues, the token embeddings (colored by their first three PCA components) continue to exhibit smooth, structured patterns that align with object boundaries and semantic regions. This demonstrates that the interpolative geometry reflects genuine semantic organization rather than artifacts of positional encoding.}
    \label{app:fig:pca_no_pos}
\end{figure}

\clearpage
\section{Minkowski Representation Hypothesis}
\label{app:hypothesis_theory}

The hypothesis stated in~\cref{sec:hypothesis} is an hypothesis about how the representation space is composed. Such assumption is crucial, as it determines what the method, and what the method can validly recover, (what we can see). To put it simply, if sparse autoencoders implicitly answer to the Linear Representation Hypothesis, then we understand the importance of specifying the right ambient geometry as it not only conditions our interpretation but also determines the very methods we use to extract concepts.

Armed with the observation of~\cref{sec:statistics_and_geometry_of_concepts}, we contend that an alternative account can explain the phenomena we documented, in particular the interpolative geometry within single images.

\subsection{Background and Related Work on Convex/Polytopal Views}

As explained in the main text, our motivation for this hypothesis traces to Gärdenfors’ conceptual spaces, where concepts inhabit convex regions along geometric quality dimensions \cite{gardenfors2004conceptual}. 
Put plainly, we have reasons to believe that the observed interpolation is the surface of a deeper organization: the activation space behaves as a sum of convex hulls~\footnote{This aligns with evidence that concept structure can be convex and compositional in other domains~\cite{park2025geometry}}. 
A single attention head already performs convex interpolation over its values, creating an archetypal geometry; multi-head attention then aggregates these convex pieces additively, yielding a Minkowski sum. 
It is natural to imagine one hull reflecting token position, another depth, another object or part category, so that the final activation is the sum of these convex interpolations, and the ``concepts'' available to probes are the archetypes governing each hull. 
We will start by making this idea explicit, then we will review some theoretical evidence based on previous works and showing how simple attention blocks generate such geometry, Then, we will follow showing some empirical signals that make the proposal a plausible candidate and the implication of such geometry. 
We first formally recall the \hyp~stated in \cref{sec:hypothesis}.

\begin{definition}{\textbf{Minkowski Representation Hypothesis (MRH).}}
Let $\mathcal{X} \subset \mathbb{R}^{d}$ be a layer’s activation space and $\x \in \mathcal{X}$. Let $\Arch = (\bm{a}_{1},\ldots,\bm{a}_{c}) \in \mathbb{R}^{c\times d}$ be an Overcomplete Dictionary of Archetypes ($c\gg d$).
We partition the archetypes into $m$ disjoint tiles $\{\T_i\}_{i=1}^m$, $\T_i\subset\{1,\ldots,c\}$, and define the tile polytopes:
\[
\P_i=\conv(\Arch_{\T_i}) \quad \text{with} \quad \Arch_{\T_i}=\{\bm{a}_j: j\in\T_i\}.
\]
Then $\mathcal{X}$ satisfies MRH if
\[
\begin{cases}
\text{\textit{(\textbf{i})~~~Minkowski sum:}} &
\mathcal{X}=\oplus_{i=1}^{m}\P_i.\\
\text{\textit{(\textbf{ii})~~Block-convex coding:}} &
\x = \sum_{i \in S} \z_i \Arch_{\T_i},\ \ \z_i \in \Delta^{|\T_i|},\ \ |S| \ll m. \\
\end{cases}
\]
where $\Delta^{k} = \{\bm{z} \geq 0 : \mathbf{1}^{\top}\bm{z} = 1\}$.
\end{definition}

We have formally defined MRH motivated by an empirical observation of interpolative geometry and now ask a basic question: can a single attention block generate such structure?
The answer is yes, and the mechanism is elementary.
We first review relevant prior work, then show how attention naturally generates such structure.
Previous work provides two lines of supporting evidence. First, neural networks naturally partition input spaces into convex regions through their piecewise-linear activations~\cite{montufar2014number,raghu2017expressive,balestriero2018spline}, suggesting that convex decompositions may be fundamental to deep architectures~\cite{tvetkova2025convex,tankala2023kdeep,hindupur2025projecting}. Analyses have also approached neural networks explicitly through the polytope lens, showing how piecewise-linear partitions structure representation space~\cite{black2022interpreting}.
Second, recent work on neural population geometry shows that deep networks organize representations into low-dimensional manifolds with rich geometric structure~\cite{chung2021neural,cohen2020separability}, which is consistent with a small number of convex factors combining to yield the observed variability.
Third, in the language domain, categorical and hierarchical concepts have been shown to admit convex (polytopal) representations, with semantic relations reflected directly in geometric structure~\cite{park2025geometry}.  

We now demonstrate that standard attention mechanisms naturally generate the geometry described by \cref{def:minkowski}. The argument proceeds in three steps, and full proofs are provided in~\cref{app:hypothesis_theory}. showing how elementary operations compose to create the previously describe geometric structure.

We will show that the elementary operations available in DINO, and in particular a standard multi-head attention block, already generate the geometry described by the Minkowski Representation Hypothesis. We then state and prove a few basic properties of this geometry and discuss its robustness. We begin by recalling some basic fact and with elementary results.

\paragraph{Notations.} For a finite set of vectors $\V = \{\bm{v}_1,\ldots,\bm{v}_m\}\subset\mathbb{R}^d$ the convex hull $\conv(\cdot)$ is
\[
\conv(\bm{V}) = \Bigl\{\sum_{j=1}^m \alpha_j \bm{v}_j \,:\, \bm{\alpha}\in\Delta^m\Bigr\}
\quad\text{where}\quad
\Delta^m = \Bigl\{\bm{\alpha}\in\mathbb{R}^m \,:\, \alpha_j\ge 0,\ \sum_{j=1}^m \alpha_j = 1\Bigr\}.
\]
And we denote by $\bm{\sigma} : \R^{d} \to \R^d$ the standard softmax:
\[
\sigma(\x)_j = \frac{e^{x_j}}{\sum_{k=1}^m e^{x_k}}
\]
and recall two basic properties that will be used repeatedly:
(\textbf{\textit{i}}) softmax is invariant under adding a constant along the all-ones direction, $\bm{\sigma}(\x+\lambda\mathbf{1}) = \bm{\sigma}(\x)$, (\textit{\textbf{ii}}) it maps $\mathbb{R}^m /\mathrm{span}\{\mathbf{1}\}$ diffeomorphically onto the relative interior of the simplex $\text{relint}(\Delta^m)$, with inverse given by the log map up to an additive constant, namely if $\bm{\alpha} \in \text{relint}(\Delta^m)$ then $\bm{x} = \log \bm{\alpha}+ \lambda\mathbf{1}$ satisfies $\bm{\sigma}(\x) = \bm{\alpha}$.

\subsection{From a single head to convex polytopes}
\label{app:sec:single_head_polytope}

We start with the most basic question: \emph{what does a single attention head produce geometrically?} A single head takes a query, forms attention weights over a fixed set of value vectors, and returns a weighted sum of those values. We now make precise the fact that the range of this map is a convex polytope, namely the convex hull of the values, and that under a mild reachability condition every point of that hull can be attained.

\begin{lemma}[Single head creates convex polytopes]
\label{app:prop:single_head}
Consider one attention head with values $\bm{V}=\{\bm{v}_1,\ldots,\bm{v}_m\}\subset\mathbb{R}^d$ and attention weights $\bm{\alpha}\in\Delta^m$. The attainable output set is
\[
\mathcal{Y} = \Bigl\{\sum_{j=1}^m \alpha_j \bm{v}_j \,:\, \bm{\alpha}\in\Delta^m\Bigr\} \subseteq \conv(\bm{V}).
\]
Moreover, suppose the pre-softmax logit map can generate any vector in $\mathrm{Im}(\bm{K}^\top)+\mathrm{span}\{\mathbf{1}\}=\mathbb{R}^m$ as the query varies, where $\bm{K}$ denotes the matrix of keys and the invariance to $\mathbf{1}$ reflects softmax’s additive invariance. Then $\mathcal{Y}=\conv(\bm{V})$. In this case they admit a MRH representation, their codes are $\bm{z} = \bm{\alpha}$ and the archetypes are $\Arch = \bm{V}$.
\end{lemma}

\begin{proof}
By definition of a single head, the output has the form $\bm{y} = \bm{\alpha} \bm{v}$
with $\alpha \in \Delta^{m}$. Therefore $\bm{y} \in \conv(\bm{V})$ and $\mathcal{Y} \subseteq \conv(\bm{V})$. Now, let $\bm{\alpha} \in \text{relint}(\Delta^m)$ be arbitrary. Define $\bm{u}= \log \bm{\alpha}$ entrywise and note that adding any constant $c$ along $\mathbf{1}$ leaves softmax unchanged. Hence $\bm{\sigma}(\x+ \lambda\mathbf{1}) = \bm{\alpha}$. If the pre-softmax logit map can realize any vector in $\mathbb{R}^m$ up to the $\mathbf{1}$ direction, then there exists a query producing logits $\x+\lambda\mathbf{1}$. For that query, the attention weights equal the prescribed $\bm{\alpha}$ and the head output is the convex combination $\sum_j \alpha_j \bm{v}_j$.
The argument above covers all points with strictly positive coefficients, that is $\text{relint}(\Delta^m)$. The extreme points and boundary faces of $\Delta^m$ can be attained by limits of interior points (or by sending some logits to $-\infty$), hence their images under the affine map $\bm{\alpha} \mapsto \sum_j \alpha_j \bm{v}_j$ are limits of attainable outputs. Therefore $\mathcal{Y}$ contains all of $\conv(\bm{V})$.

Finally, for the MRH identification, this corresponds to the special case of $|S| = 1$ with a single polytope, where the codes are simply $\bm{z} = \bm{\alpha}$ and the archetypes are the values $\Arch = \bm{V}$.
\end{proof}

An interesting observation is that one could show that a strictly block-sparse attention (with $\bm{V}$ affinely independent) would induce disjoint polytopes by applying the previous lemma to each block, splitting $\text{conv}(\bm{V})$ as the union of $\text{conv}(\bm{V}_i)$ for each block $b$. This is particularly interesting as ViT attention patterns are often observed to be block-sparse in practice. This establishes the base case: a single head creates a convex polytope. The next question is: \emph{does this convex structure survive the linear and affine mappings applied throughout the transformer, such as projection matrices or RMSNorm?}

\subsection{Robustness to transformations} 
\label{app:sec:mrh_affine_robust}

The first observation is straightforward but crucial: the image of a convex combination under an affine map is the corresponding convex combination of the images. In our setting, this means that if an activation admits an archetypal decomposition, then any affine layer simply moves the archetypes while keeping the codes unchanged. This is exactly the stability we need to propagate archetypal structure through projections and bias terms.

\begin{lemma}[Affine transformations preserve MRH structure]
\label{app:prop:affine}
Let $\bm{\gamma}(\bm{x}) = \bm{W}\bm{x} + \bm{b}$ be an affine transformation and
$\bm{x} = \sum_{j} z_{j}\, \bm{a}_{j}$ with $z_{j} \ge 0$ and $\sum_{j} z_{j} = 1$.
Then
\[
\bm{\gamma}(\bm{x}) = \sum_{j} z_{j}\, \bm{a}'_{j}
\quad \text{with} \quad
\bm{a}'_{j} = \bm{W}\bm{a}_{j} + \bm{b}.
\]
Hence convex structure is preserved: archetypes absorbe the transformation, but the codes $\bm{z}$ remain unchanged.
\end{lemma}

\begin{proof}
Starting from $\bm{x} = \sum_{j} z_{j}\, \bm{a}_{j}$ with $z_{j}\ge 0$ and $\sum_{j} z_{j} = 1$,
apply $\bm{\gamma}$ and use linearity of $\bm{W}$:
\[
\bm{\gamma}(\bm{x})
= \bm{W}\Bigl(\sum_{j} z_{j}\, \bm{a}_{j}\Bigr) + \bm{b}
= \sum_{j} z_{j}\, \bm{W}\bm{a}_{j} + \bm{b}
= \sum_{j} z_{j}\, (\bm{W}\bm{a}_{j} + \bm{b})
= \sum_{j} z_{j}\, \bm{a}'_{j}.
\]
The right-hand side is a convex combination of the transformed archetypes $\bm{a}'_{j}$.
\end{proof}

Any linear projection $\bm{W}$ and any bias addition preserve convex decompositions exactly as in Lemma~\ref{app:prop:affine}. In particular, the per-head output projections and the final output projection of attention blocks do not break archetypal structure. We remark that LayerNorm and RMSNorm are not globally affine because their scale factor depends on the input. However, many norm become affine in evaluation, once mean/variance are held fixed.

\subsection{From multiple heads to Minkowski sums} 
\label{app:sec:minkowski_mha}

We now turn to the multi-head case. Each head yields a convex polytope (the convex hull of its value vectors), and the standard attention output aggregates head outputs additively after a per-head linear projection. This naturally leads to a Minkowski-sum geometry as describe in \cref{def:minkowski}.

\begin{proposition}[Multi-head attention realizes MRH geometry]
\label{app:prop:mha_minkowski}
Let there be $H$ heads. For head $h$, let $\bm{V}_{h} = \{\bm{v}_{h}^{(1)},\ldots,\bm{v}_{h}^{(m_h)}\}$ be
its value vectors, and let the head output be
\[
\bm{y}_{h} = \sum_{j=1}^{m_h} \alpha_{h,j}\, \bm{v}_{h}^{(j)}
\quad \text{with} \quad
\bm{\alpha}_{h} \in \Delta^{m_h}.
\]
Let $\bm{W}_{O}^{(h)}$ be the per-head output projection. The block output is
\[
\bm{y} = \sum_{h=1}^{H} \bm{W}_{O}^{(h)} \bm{y}_{h}.
\]
Then
\[
\bm{y} \in \oplus_{h=1}^{H} \bm{W}_{O}^{(h)}\bigl(\conv(\bm{V}_{h})\bigr),
\]
that is, the attainable outputs lie in the Minkowski sum of the head polytopes after projection.
Moreover, under the reachability condition that each head can realize any point of
$\text{relint}\bigl(\Delta^{m_h}\bigr)$ (up to the softmax additive constant), the set of attainable
outputs equals this Minkowski sum. In that case, the representation admits an MRH form with
block-convex codes
\[
\bm{z} = \bigl(\bm{\alpha}_{1},\ldots,\bm{\alpha}_{H}\bigr)
\quad \text{and archetypes} \quad
\Arch = \bigl(\bm{W}_{O}^{(1)}\bm{V}_{1},\ldots,\bm{W}_{O}^{(H)}\bm{V}_{H}\bigr),
\]
where each block $\bm{\alpha}_{h}$ belongs to a simplex and each block of archetypes is the
projected value set for that head.
\end{proposition}
\begin{proof}
By Lemma~\ref{app:prop:single_head}, for each head $h$ we have
$\bm{y}_{h} \in \conv(\bm{V}_{h})$. Then, for any linear map $\bm{L}$ and finite set $S$, we have
$\bm{L}(\conv(S)) = \conv(\bm{L}(S))$.
Therefore $\bm{W}_{O}^{(h)} \bm{y}_{h} \in \bm{W}_{O}^{(h)}\bigl(\conv(\bm{V}_{h})\bigr)
= \conv\bigl(\bm{W}_{O}^{(h)}\bm{V}_{h}\bigr)$. By definition of the Minkowski sum, if $\bm{p}_{h} \in \bm{W}_{O}^{(h)}\bigl(\conv(\bm{V}_{h})\bigr)$
then $\sum_{h} \bm{p}_{h} \in \oplus_{h} \bm{W}_{O}^{(h)}\bigl(\conv(\bm{V}_{h})\bigr)$.
Taking $\bm{p}_{h} = \bm{W}_{O}^{(h)} \bm{y}_{h}$ and summing over heads gives
\[
\bm{y} = \sum_{h=1}^{H} \bm{W}_{O}^{(h)} \bm{y}_{h}
\in \oplus_{h=1}^{H} \bm{W}_{O}^{(h)}\bigl(\conv(\bm{V}_{h})\bigr).
\]
This proves the inclusion.

Now, assume head $h$ can realize any $\bm{\alpha}_{h} \in \text{relint}\bigl(\Delta^{m_h}\bigr)$. Let an arbitrary element of the Minkowski sum be given:
\[
\bm{y}^{\star} = \sum_{h=1}^{H} \bm{p}_{h}
\quad \text{with} \quad
\bm{p}_{h} \in \bm{W}_{O}^{(h)}\bigl(\conv(\bm{V}_{h})\bigr).
\]
For each $h$ there exists $\bm{\beta}_{h} \in \Delta^{m_h}$ such that
$\bm{p}_{h} = \bm{W}_{O}^{(h)} \sum_{j} \beta_{h,j}\, \bm{v}_{h,j}$.
By reachability, choose a query so that the head-$h$ attention equals $\bm{\alpha}_{h} = \bm{\beta}_{h}$.
Then the block output equals $\bm{y}^{\star}$. Since $\bm{y}^{\star}$ was arbitrary in the Minkowski sum,
the attainable set equals the sum.

For the MRH identification, we collect the per-head attention weights into the block vector $\bm{z} = (\bm{\alpha}_{1},\ldots,\bm{\alpha}_{H})$ and the per-head projected values into archetype blocks $\Arch = (\bm{W}_{O}^{(1)}\bm{V}_{1},\ldots,\bm{W}_{O}^{(H)}\bm{V}_{H})$. The output $\bm{y}$ is thus a sum of $H$ block-convex combinations, realizing the Minkowski sum structure with block-convex coding as required by MRH.
\end{proof}

Having established that attention produces convex polytopes and that affine transformations preserve their structure, we now move toward more realistic operating regimes. What does it mean, geometrically, to have \emph{sparse} or \emph{block-sparse} attention as is commonly observed, and what happens under the elementwise nonlinearities used in practice?

\subsection{Attention Concentration Effects}

We first examine the limit in which softmax sharpens. In fact, as temperature $\tau$ decreases, attention places nearly all mass on the highest-logit index. The attainable set contracts from the full convex hull toward its vertices.

\begin{lemma}[Low-temperature softmax selects vertices]
\label{app:lem:lowT}
Consider attention with values $\bm{V}=\{\bm{v}_1,\ldots,\bm{v}_m\}$ and weights $\bm{\alpha}=\bm{\sigma}(\x/\tau)$ with logits $\x \in \mathbb{R}^m$ and temperature $\tau>0$. Then
\[
\lim_{\tau\to 0} \ \sum_{j=1}^{m} \alpha_j \bm{v}_j = \bm{v}_{j^\star}
\quad \text{where} \quad
j^\star = \arg\max_{j} x_j.
\]
In the zero-temperature limit the output lies at a vertex of $\conv(\bm{V})$.
\end{lemma}

\begin{proof}
For $\tau>0$, $\alpha_j = \exp(x_j/\tau)/\sum_k \exp(x_k/\tau)$.
If $j^\star$ indexes the maximum of $\x$ and $\Delta_j = x_{j^\star}-x_j \ge 0$, then
$\alpha_{j^\star} = 1\big/\bigl(1+\sum_{j\neq j^\star} e^{-\Delta_j/\tau}\bigr)$ and
$\sum_{j\neq j^\star} \alpha_j = \sum_{j\neq j^\star} e^{-\Delta_j/\tau}\, \alpha_{j^\star}$.
As $\tau\to 0$, all terms $e^{-\Delta_j/\tau}$ vanish for $\Delta_j>0$, hence $\alpha_{j^\star}\to 1$ and $\alpha_j\to 0$ for $j\neq j^\star$. The convex combination collapses to $\bm{v}_{j^\star}$.
\end{proof}

In fact, we can derive a quantitative measure of this convergence, let $\mathrm{diam}(\bm{V}) = \max_{p,q}\|\bm{v}_p-\bm{v}_q\|$ denote the diameter of the value set. Then the deviation from the winning vertex $\bm{v}_{j^\star}$ satisfies
\[
\left\| \sum_{j} \alpha_j \bm{v}_j - \bm{v}_{j^\star} \right\|
\le \mathrm{diam}(\bm{V}) \sum_{j\neq j^\star} e^{-\Delta_j/\tau},
\]
showing that a finite logit margin already forces the output to lie in a small neighborhood of a vertex, with exponentially small deviation in $1/\tau$.
Another way to see the effect is to direclty study the geometry created under strict sparsity contraint.
\begin{lemma}[Support restriction selects a subpolytope]
\label{app:lem:support_face}
Fix a subset $S\subseteq\{1,\ldots,m\}$ and consider the feasible set of outputs with attention supported on $S$,
\[
\mathcal{Y}_S = \Bigl\{ \sum_{j\in S} \alpha_j \bm{v}_j \,:\, \bm{\alpha}\in\Delta^m,\ \alpha_j=0 \text{ for } j\notin S \Bigr\}.
\]
Then $\mathcal{Y}_S = \conv(\{\bm{v}_j: j\in S\})$. In particular, if across input regimes the support repeatedly takes values in a family $\mathcal{S}$ of subsets, the attainable set is the union $\bigcup_{S\in\mathcal{S}} \conv(\{\bm{v}_j: j\in S\})$, i.e., a union of lower-dimensional polytopes. When a given $S$ coincides with the maximizers of some linear functional over $\bm{V}$, $\mathcal{Y}_S$ is a face of $\conv(\bm{V})$.
When $S = \arg\max_{j} \langle \bm{w}, \bm{v}_j \rangle$ for some vector $\bm{w}$, then $\mathcal{Y}_S$ is the face of $\conv(\bm{V})$ exposed by the supporting hyperplane with normal $\bm{w}$.
\end{lemma}
\begin{proof}
Immediate from the definition of convex hull and the constraint $\alpha_j=0$ for $j\notin S$. The face condition is the standard supporting-hyperplane characterization of faces~\cite{rockafellar1970convex}.
\end{proof}

Lemma~\ref{app:lem:support_face} formalizes the intuition: as attention sparsifies, geometry collapses from the full polytope to subpolytopes, and in the extreme low-temperature limit to vertices (point). 

\subsection{Non-identifiability of Minkowski decomposition}
\label{app:sec:non_identifiability}

We now address a question that bears directly on recoverability: given only the attainable activation set
\[
\mathcal{X} \subset \mathbb{R}^{d}
\quad \text{with} \quad
\mathcal{X} = \oplus_{i=1}^{m} \mathcal{P}_{i},
\]
can we uniquely recover the summands \(\{\mathcal{P}_{i}\}\)?
If MRH is to be useful for analysis, we must understand when (and why) such decompositions are non-unique.

\begin{proposition}[\textbf{Non-identifiability of Minkowski decomposition}]
\label{prop:nonident}
Let \(\mathcal{X} = \oplus_{i=1}^{m} \mathcal{P}_{i}\) be a Minkowski sum of convex polytopes.
Given only observations from \(\mathcal{X}\), the decomposition \(\{\mathcal{P}_{1},\ldots,\mathcal{P}_{m}\}\) is generally non-unique:
there exist distinct collections \(\{\mathcal{Q}_{1},\ldots,\mathcal{Q}_{k}\}\) such that
\(\mathcal{X} = \oplus_{j=1}^{k} \mathcal{Q}_{j}\).
In particular, even very simple polytopes admit infinitely many decompositions as sums of line segments (zonotope generators) with varying directions and lengths.
\end{proposition}

\begin{proof}
We argue via support functions. For a nonempty closed convex set \(\mathcal{C}\subset\mathbb{R}^{d}\),
its support function is
\[
h_{\mathcal{C}}(\bm{u}) = \sup_{\x \in \mathcal{C}} \langle \bm{u}, \x \rangle
\qquad \bm{u} \in \mathbb{R}^{d}.
\]
Support functions are sublinear (positively homogeneous and subadditive), and they are additive under Minkowski sums~\cite{gardner1995geometric}:
\[
h_{\mathcal{A} \oplus \mathcal{B}}(\bm{u}) = h_{\mathcal{A}}(\bm{u}) + h_{\mathcal{B}}(\bm{u})
\quad \text{for all } \bm{u} \in \mathbb{R}^{d}.
\]
Hence a decomposition \(\mathcal{X} = \oplus_{i} \mathcal{P}_{i}\) is equivalent to a decomposition
\[
h_{\mathcal{X}} = \sum_{i=1}^{m} h_{\mathcal{P}_{i}}
\]
of the single sublinear function \(h_{\mathcal{X}}\) into a sum of sublinear summands.

But additive decompositions of a fixed sublinear function are highly non-unique in general.
Indeed, fix any sublinear \(s\) with \(0 \le s \le h_{\mathcal{X}}\) pointwise. Then both
\[
h_{\mathcal{X}} = (h_{\mathcal{X}} - s) + s
\quad \text{and, more generally,}~~\forall~\{h_1, ..., h_m\}~~\text{s.t.}~~
\sum_{i=1}^{m} h_{i} = h_{\mathcal{X}}
\]
all define a valid support-function decompositions.
Under lower semicontinuity, each sublinear \(h_{i}\) is itself the support function of a unique closed convex set \(\mathcal{Q}_{i}\) containing the origin, i.e., \(h_{\mathcal{Q}_{i}} = h_{i}\).
Therefore
\[
h_{\mathcal{X}} = \sum_{i=1}^{m} h_{i}
\quad \Longleftrightarrow \quad
\mathcal{X} = \oplus_{i=1}^{m} \mathcal{Q}_{i}.
\]
Since there are infinitely many ways to split \(h_{\mathcal{X}}\) into a sum of sublinear functions, there are infinitely many corresponding Minkowski decompositions. This proves the general non-uniqueness claim.
\end{proof}

To build concrete intuition on why non-uniqueness appear, consider the simplest case of decomposing a rectangle as a sum of segments.

\paragraph{Segments generating the same zonotope}
We work in \(\mathbb{R}^{2}\) and consider the axis-aligned rectangle
\[
\mathcal{X} = [-a,a] \times [-b,b]
= ([-a,a]\mathbf{e}_{1}) \oplus ([-b,b]\mathbf{e}_{2}).
\]
For any \(\alpha \in (0,1)\),
\[
[-a,a]\mathbf{e}_{1}
= [-\alpha a, \alpha a]\mathbf{e}_{1} \oplus [-(1-\alpha)a, (1-\alpha)a]\mathbf{e}_{1},
\]
so
\[
\mathcal{X}
= [-\alpha a, \alpha a]\mathbf{e}_{1}
\ \oplus\ [-(1-\alpha)a, (1-\alpha)a]\mathbf{e}_{1}
\ \oplus\ [-b,b]\mathbf{e}_{2}.
\]
Varying \(\alpha\) yields uncountably many distinct 3-segment decompositions of the same rectangle.
Also, iterating this splitting on either axis leads to infinitely many distinct \(k\)-segment decompositions whose projections sum to the same total width and height.
Hence even for very simple polytopes, Minkowski-sum decompositions into segments are non-unique.

\paragraph{Limitations \& Positioning.} We emphasize that the observations done in the \cref{sec:hypothesis} -- piece-wise linear on-manifold trajectories, efficient archetypal reconstructions, and block-structured co-activation -- represent preliminary evidence rather than definitive proof. Multiple geometric hypotheses could generate similar surface patterns. Nevertheless, the convergent evidence motivates exploring what MRH would mean for interpretability practice.

Specifically, we treat the Minkowski Representation Hypothesis (MRH) as a working hypothesis of a refinement of LRH and not a proved theory of representation. We do not claim causal identification of concepts, head ``tiles,'' or their generative mechanisms; our evidence is observational and model-specific (DINOv2-B). Our goal is to describe robust empirical regularities (task-specific subspaces; families of depth cues; ``Elsewhere'' patterns) and propose a geometry that makes testable predictions for future work (e.g., per-head block structure, subspace additivity across heads). We therefore refrain from stating that our results ``contradict'' LRH; instead, we document systematic finding of regularities that go beyond a purely sparse near-orthogonal picture. All conclusions should be read as conditional on our dictionary, probes, and datasets, and we report ablations where feasible.

\clearpage
\section{Visualization tool}
\label{app:demo}

To facilitate exploration of the concept dictionary and make our findings accessible to the broader research community, we release \textbf{Dino}Vision, an interactive web-based visualization tool that enables real-time exploration of the 32,000 extracted concepts. 
The tool presents concepts as points in a 2D UMAP projection where spatial proximity indicates conceptual similarity in the original high-dimensional space, though global clustering patterns should be interpreted with caution due to UMAP's limitations in preserving large-scale structure.
\begin{figure}[h]
\centering
\includegraphics[width=\textwidth]{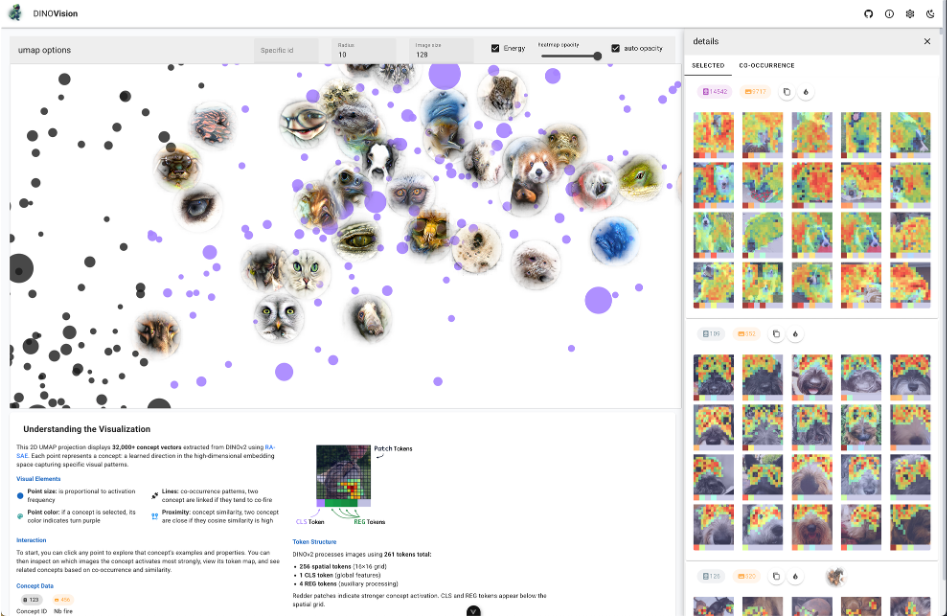}
\caption{DinoVision interface showing the interactive 2D UMAP projection of 32,000 concepts extracted from DINOv2. Users can explore individual concepts by clicking points to reveal activation patterns across the 261 token grid (256 spatial patches plus CLS and register tokens). The visualization includes adjustable parameters for point size, opacity, and co-occurrence links between frequently activated concept pairs.}
\label{fig:website}
\end{figure}
The interface displays each concept as a point whose size reflects activation frequency across the dataset.
Users can click any point to examine detailed activation patterns showing how that concept fires across DINOv2's 261 tokens, which comprise 256 spatial patches arranged in a $16\times16$ grid plus one classification token and four register tokens. The token visualization uses color intensity to indicate activation strength, with redder regions corresponding to stronger concept responses. 
The tool help us discovered that some concepts activate exclusively on register tokens, and they seems to capturing global scene properties like illumination and motion blur.
Interactive features include adjustable visualization parameters such as point size scaling, heatmap opacity controls, and the ability to display co-occurrence links between concepts that frequently activate together. 
Users can navigate directly to specific concepts by entering concept identifiers or explore neighborhoods around selected points. The co-occurrence analysis reveals structured relationships in the concept space, with connecting lines indicating statistical dependencies between concept activations. 
For concept visualization we use Feature Accentuation (FA) from~\cite{hamblin2024feature}, we start from maximally activating images of each concept and perform 1024 steps of gradient ascent optimization parameterized in Fourier space with MACO~\cite{fel2023unlocking} constraints on the magnitude of the spectrum, boosted according to natural image statistics approximately following $1/\omega^2$ where $\omega$ represents cycles per image.

Importantly, the tool implements a composite (two-layer) rendering approach that maintains smooth 60fps interaction for point navigation while progressively loading high-resolution concept visualizations as needed.
We believe the tool serves as both a research instrument for further investigation and a demonstration of the rich structure present in vision transformer representations.

\end{document}